%% file: main.tex
\documentclass{article}

\usepackage{microtype}
\usepackage{graphicx}
\usepackage{subfigure}
\usepackage{booktabs} 

\usepackage[main, final]{neurips_2026}
\usepackage{hyperref}

\usepackage{amsmath}

\usepackage{amssymb}
\usepackage{mathtools}
\usepackage{amsthm}
\usepackage{xcolor}
\usepackage{comment}
\usepackage{amssymb}
\usepackage[capitalize,noabbrev]{cleveref}
\usepackage{algorithm}
\usepackage{algorithmic}
\usepackage{pgfplots}
\pgfplotsset{compat=1.18}
\definecolor{famFT}{HTML}{3B78B8}
\definecolor{famRob}{HTML}{D55E00}
\definecolor{famInf}{HTML}{009E73}

\theoremstyle{plain}
\newtheorem{theorem}{Theorem}[section]

\newtheorem{corollary}[theorem]{Corollary}
\theoremstyle{definition}
\newtheorem{definition}[theorem]{Definition}
\newtheorem{assumption}[theorem]{Assumption}
\theoremstyle{remark}
\newtheorem{remark}[theorem]{Remark}

\newenvironment{squishenumerate}
  {\begin{list}{\arabic{enumi}.}{%
    \usecounter{enumi}%
    \setlength{\itemsep}{0pt}%
    \setlength{\parsep}{0pt}%
    \setlength{\topsep}{0pt}%
    \setlength{\parskip}{0pt}%
    \setlength{\labelwidth}{.5in}%
    \setlength{\labelsep}{0.05in}%
    \setlength{\leftmargin}{.2in}}}
  {\end{list}}

\usepackage[textsize=tiny]{todonotes}

\title{You Can’t Have It Both Ways:
Concept Entanglement Limits Diffusion Model Unlearning}

\newif\ifcomments
\commentstrue        

\newcommand{\ali}[1]{\textcolor{orange}{[ali: #1]}}
\newcommand{\varun}[1]{\textcolor{red}{[Varun: #1]}}

\newcommand{\rs}{r_s}
\author{%
  Yian Wang,~
  Ali Ebrahimpour-Boroojeny,~
  Hari Sundaram,~ 
  Varun Chandrasekaran 
  \\[0.5em]
  University of Illinois Urbana-Champaign \\
  Urbana, IL, USA
}

\begin{document}
\maketitle


\input{sections/00Abstract}
\input{sections/01Introduction}

\input{sections/02Related}

\input{sections/03Motivation}

\input{sections/04Proof}

\input{sections/05Experiment}
\input{sections/06Discussion}
\input{sections/07Limitation}
\newpage
\bibliography{references}   
\bibliographystyle{plainnat}
\newpage
\appendix
{\large \begin{center} {\bf Appendix} \end{center}}
\input{sections/app/01}
\input{sections/app/02}
\input{sections/app/03}
\input{sections/app/04}
\input{sections/app/05}
\input{sections/app/06}

\input{sections/app/07}
\input{sections/app/08}
\input{sections/app/09}
\input{sections/app/10}
\input{sections/app/11}
\input{sections/app/13}
\input{sections/app/14}
\input{sections/app/15}

\end{document}

%% file: sections/00Abstract.tex
\begin{abstract}
Concept unlearning in text-to-image diffusion models aims to suppress a target concept (e.g., \texttt{horse}) while preserving related but distinct content (e.g., \texttt{donkey}), yet existing methods either leak under indirect prompts or visibly degrade other concepts. We show that these failure modes stem from the geometry of concept representations rather than from any particular algorithm. Formalizing concepts as activation-space regions, we prove that the overlap between a target and other concepts lower-bounds the damage any robust erasure must inflict on them, with the trade-off scaling linearly in the degree of overlap. Across thirteen unlearning methods, including methods designed to preserve non-target concepts, no method achieves both strong erasure and strong neighbor preservation: STEREO nearly eliminates indirect leakage but cuts neighbor generation by more than 75\%, while sparse inference-time methods preserve neighbors but leak. Damage increases with our overlap measure, monotonically so for STEREO; the $\kappa$-scaling reproduces on SDXL, and neighbor-selective damage recurs on FLUX. Perfect unlearning is the wrong target for entangled concepts; methods should be evaluated on the Pareto frontier our theorem establishes.
\end{abstract}

%% file: sections/01Introduction.tex
\section{Introduction}
\label{sec:intro}
Text-to-image diffusion models~\citep{ho2020denoising, rombach2022high, saharia2022photorealistic} have rapidly become the dominant generative paradigm for high-fidelity image synthesis, with deployments ranging from creative tools to commercial APIs.
Privacy concerns and copyrights have motivated active research on \emph{concept unlearning}~\citep{gandikota2023erasing}: given a trained model and an undesirable concept (e.g., an object category, an artistic style, or a specific identity), the goal is to suppress the concept while preserving the model's general capabilities. In this work, we focus on object-category concepts, where benchmarks and detectors are most mature.
A growing body of work has proposed methods for this task, ranging from fine-tuning-based approaches that permanently alter model weights~\citep{gandikota2023erasing,fan2023salun,zhang2024defensive,srivatsan2025stereo} to inference-time interventions that block concept-specific activations without modifying the underlying model~\citep{li2024get,lyu2024one,cywinski2025saeuron}. These methods are typically evaluated on two families of metrics: \emph{unlearning accuracy} (UA), measuring whether the target concept has been successfully suppressed, and \emph{retain accuracy} (RA), measuring whether the model's performance on non-target concepts remains intact.

A tacit assumption in this evaluation paradigm is that UA and RA can be optimized independently. However, our empirical evidence suggests this assumption is violated (\cref{sec:motivation}). Methods achieving strong erasure (high UA), such as STEREO~\citep{srivatsan2025stereo}, often degrade generation quality on semantically related concepts (low RA).
Methods with better retention, such as SAeUron~\citep{cywinski2025saeuron}, leave the target concept reachable through prompts that omit its name but supply correlated context (e.g., generating a horse from ``a jockey at the racetrack''), a phenomenon we call \emph{concept leakage}, resulting in low UA.
These observations indicate a fundamental tension: \emph{unlearning and retention are inherently coupled when concepts share representational structure}, a coupling visible directly in the model's internal representations, where semantically related concepts (\texttt{horse}/\texttt{pony}/\texttt{donkey}) occupy overlapping regions of activation space while isolated concepts (\texttt{castle}) sit well-separated (\cref{fig:umap_kappa}). 
This shows that concept activation regions are real geometric structures whose overlap reflects semantic similarity, motivating our formal notion of concept activation regions in~\cref{sec:tradeoff}\footnote{The formal definition of concept activation regions is in Appendix~\ref{app:concept_regions}; full empirical construction details and per-concept clustering statistics are in Appendix~\ref{app:activation_regions}.}.
Two failure modes follow from this geometry: \emph{concept leakage}, where the target concept reappears under indirect prompts, and \emph{collateral forgetting}, degraded generation of semantically related but distinct concepts. We show that these are not shortcomings of individual algorithms but dual symptoms of overlapping concept regions, and that any robust erasure method must trade them off.

\begin{figure*}[t]
\centering
\begin{minipage}[c]{0.55\linewidth}
\centering
\includegraphics[width=\linewidth]{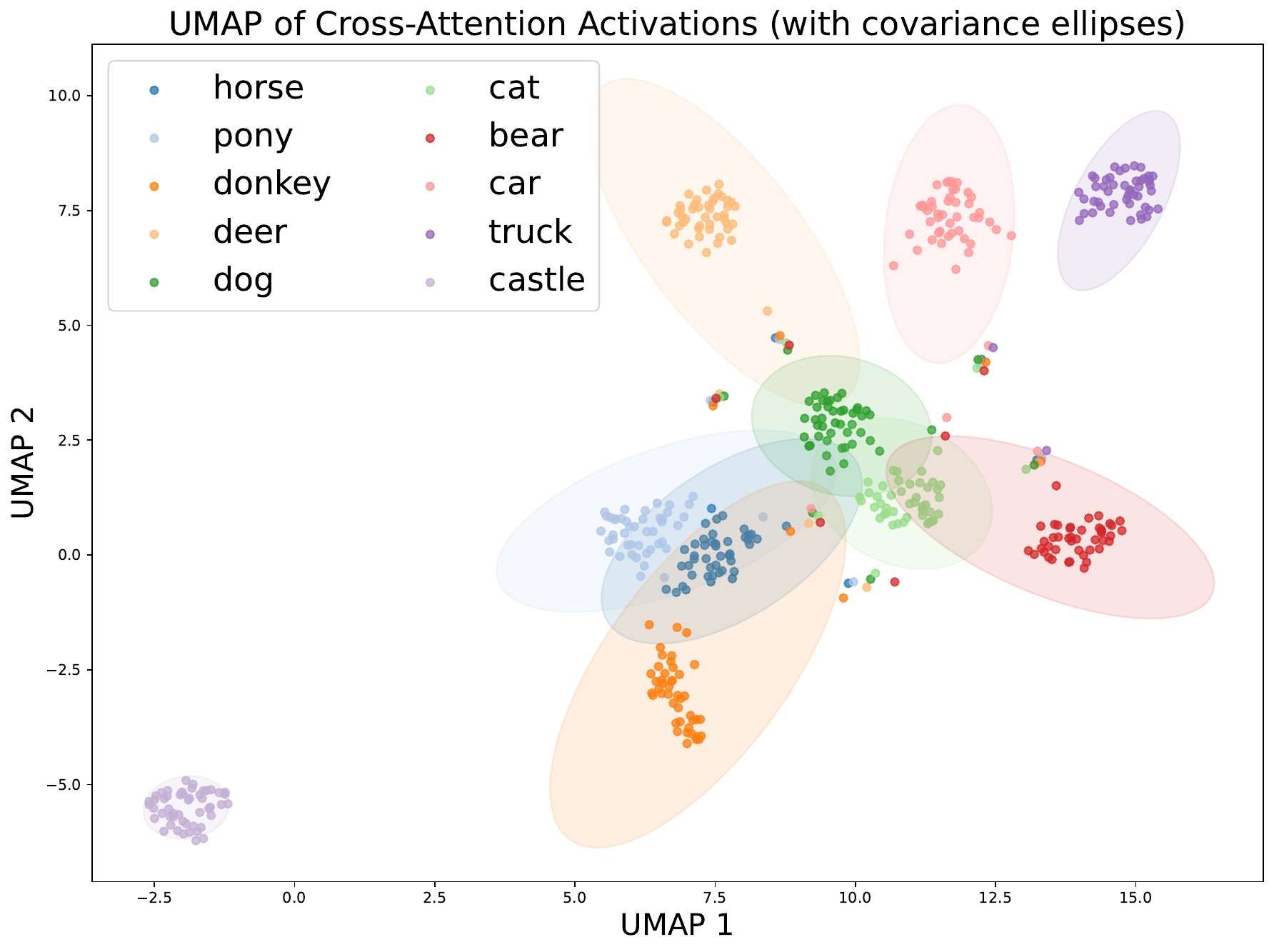}

\vspace{6pt}

\small
\setlength{\tabcolsep}{6pt}
\begin{tabular}{lcl}
\toprule
Concept & $\hat\kappa$ & Neighbors \\
\midrule
dog    & 0.89 & puppy, labrador \\
bear   & 0.82 & grizzly, panda \\
horse  & 0.81 & donkey, foal, mare \\
cat    & 0.77 & kitten, lynx \\
castle & 0.27 & fortress, palace, citadel \\
\bottomrule
\end{tabular}
\end{minipage}\hfill
\begin{minipage}[c]{0.42\linewidth}
\caption{\textbf{Concepts occupy coherent, partially overlapping regions in activation space.}
\emph{Top-left:} UMAP projection of cross-attention activations at \texttt{mid\_block.attentions.0} (step 25/50) for 10 concepts in Stable Diffusion v1.4 (visualization pool; full evaluation pool in~\Cref{app:vocab}); each point is one prompt and ellipses show covariance contours. Semantically related concepts (\texttt{horse}/\texttt{pony}/\texttt{donkey}) cluster together, while isolated concepts (\texttt{castle}) are well-separated, at the bottom left of the figure.
\emph{Bottom-left:} Estimated entanglement coefficient $\hat\kappa$ (fraction of the target's activation region shared with semantic neighbors; see~\cref{def:overlap}). The Neighbors column lists representative discovered neighbors; full per-target neighbor sets are reported in~\cref{app:kappa_estimation}. Animal subtypes have high $\hat\kappa$ and dense neighbor sets; \texttt{castle} has low $\hat\kappa$ despite a comparable number of architectural neighbors.
Ellipses are $(\mathbf z-\boldsymbol\mu_c)^\top\boldsymbol\Sigma_c^{-1}(\mathbf z-\boldsymbol\mu_c)\le 4$ for the mean and covariance of each concept's projected points; they are visualization aids only and are not used to compute $\hat\kappa$, which is estimated in the original activation space. }
\label{fig:umap_kappa}
\label{fig:umap}
\label{tab:kappa}
\end{minipage}
\end{figure*}

\paragraph{Contributions.} 
\textbf{(i)} We give a geometric account of concept entanglement, showing that the literature's two failure modes, 
concept leakage and collateral forgetting, are dual symptoms of overlapping concept regions in cross-attention activation space (quantified by an entanglement coefficient $\hat\kappa$), rather than separate algorithmic shortcomings as treated in prior work~\citep{gandikota2023erasing,zhang2024defensive,srivatsan2025stereo,cywinski2025saeuron}. \textbf{(ii)} Building on this account,
we derive a geometric lower-bound for concept unlearning under empirically-verified assumptions: any method achieving $\varepsilon$-robust erasure must incur utility preservation error $\gamma \geq \kappa(1-\varepsilon)/L$ (\Cref{thm:tradeoff}), with strict impossibility at $\varepsilon = 0$ whenever $\kappa > 0$ (\Cref{cor:impossibility}); this result reframes unlearning as a Pareto problem rather than a dual-objective one. \textbf{(iii)} We validate the predicted $\kappa$-scaled trade-off across thirteen methods on Stable Diffusion v1.4 (all methods on five targets, four methods on five additional targets that fill the intermediate $\hat\kappa$ range), with $\kappa$-scaling reproduced on SDXL, using a new Neighbor Preservation metric tied to the activation-level $\gamma$ in our theorem.

%% file: sections/02Related.tex
\section{Related Work}
\label{sec:background}

Machine unlearning in diffusion models aims to remove a target concept $c$ (e.g., \textit{horse}) from a pretrained model $\mathcal{D}_\theta$
while preserving overall generative capabilities. The goal is an updated model $\mathcal{D}_{\hat{\theta}}$ that suppresses $c$ on target-associated prompts while maintaining quality on unrelated prompts, following prior work on concept erasure in diffusion models. Background on latent diffusion models is in Appendix~\ref{app:background}.

\paragraph{Closed-form editing.} UCE~\citep{gandikota2024unified} and RECE~\citep{gong2024reliable} modify cross-attention projections through analytic updates, without gradient training. RECE additionally searches for embeddings that regenerate the target and removes them in closed form.

\textbf{Fine-tuning methods.}
A large body of work erase concepts by directly modifying the weights of a pretrained diffusion model. Erased Stable Diffusion (ESD)~\cite{gandikota2023erasing} uses negative guidance to align target-token prompts with a neutral anchor. Subsequent work improves the precision by editing cross-attention representations~\citep{wang2025ace} or updating salient concept-related parameters~\citep{fan2023salun,wu2024scissorhands}. CPE~\citep{lee2025concept} trains residual attention gates with an anchoring loss, and TRUST~\citep{kori2026selective} selectively fine-tunes target-associated neurons under Hessian-based regularization. More recent robust methods such as AdvUnlearn~\cite{zhang2024defensive} and STEREO~\cite{srivatsan2025stereo} further defend against prompt-based regeneration attacks.
However, these methods still largely define erasure through target tokens, prompts, or localized parameter subsets. They therefore suppress direct target prompts well, but struggle with \emph{concept leakage} through correlated context; when made more aggressive, their edits can propagate through shared representations and cause \emph{collateral damage} to semantic neighbors.

\textbf{Inference-time methods.}
A complementary line of work avoids modifying $\mathcal{D}_{\theta}$ and instead intervenes during generation by manipulating text embeddings~\citep{li2024get}, projecting prompt tokens away from a target subspace (SAFREE, \citealp{yoon2025safree}), adding safety guidance to the denoising process (SLD, \citealp{schramowski2023safe}), inserting lightweight concept-blocking adapters~\citep{lyu2024one}, or ablating sparse-autoencoder features~\citep{cywinski2025saeuron}. These methods are computationally efficient and often preserve broad model behavior. Intervening at inference time does not by itself make a method conservative: depending on how broadly the intervention displaces activations, these methods span the full range from strong erasure with substantial neighbor damage to strong preservation with residual leakage (Section~\ref{subsec:tradeoff_real}).

\paragraph{Preservation-oriented methods.} Several methods explicitly target the collateral damage we study. CPE anchors non-target concepts during erasure, TRUST restricts updates to target-associated neurons, and SAFREE adapts its filtering to preserve prompt fidelity. Adversarial preservation~\citep{bui2024erasing} identifies and protects the concepts most sensitive to erasure, Receler~\citep{huang2024receler} and MACE~\citep{lu2024mace} localize edits through lightweight erasers and LoRA fusion, and meta-unlearning~\citep{gao2025meta} prevents relearning of erased concepts. We evaluate CPE, TRUST, and SAFREE; they reach more favorable operating points but do not escape the trade-off (Section~\ref{subsec:tradeoff_real}).

\paragraph{Positioning.} That aggressive erasure degrades utility is well
documented, and the preservation-oriented methods above aim to reduce it;
localization work such as \citet{zarei2026localizing} identifies where concepts
are represented. Neither models target-neighbor overlap or derives limits on
unlearning. Our claim is narrower and structural: robust erasure necessarily
degrades related concepts \emph{because} their representations overlap, a limit
we quantify through $\kappa$, prove unavoidable (Theorem~\ref{thm:tradeoff},
Corollary~\ref{cor:impossibility}), and turn into a frontier-based evaluation
protocol. Across all three families, methods implicitly treat concepts as
localized and independently suppressible; this assumption fails whenever
activation regions overlap.

%% file: sections/03Motivation.tex
\section{Motivation}
\label{sec:motivation} 
Current concept unlearning methods are designed to suppress a concept token (e.g., ``\texttt{horse}'') 
so as to remove the corresponding visual concept without substantially affecting other concepts. This design implicitly assumes that visual concepts are represented as localized, token-aligned features in the model's joint text-image space.
In this section, we present empirical evidence that this assumption fails: 
in practice, unlearning a target concept produces 
\emph{collateral damage} to the concepts with which it is entangled. 
\footnote{We define entanglement formally in \cref{subsec:definitions} based on activation regions in the model's cross-attention space.}
The severity of this damage is coupled to both the strength of the erasure applied to the target concept and the extent of its entanglement with other concepts.
This coupling is the phenomenon our theoretical framework analyzes in~\cref{sec:tradeoff}.



The training mechanisms of text-to-image diffusion models make entanglement inevitable.
First, diffusion models condition on prompt embeddings produced by text encoders that often exhibit bag-of-words behavior~\citep{yuksekgonul2022and}. As a result, prompts that omit a concept token (e.g., \texttt{horse}) but retain its characteristic context (e.g., \texttt{jockey}, \texttt{saddle}, \texttt{racetrack})
can still induce activations associated with that concept (see~\cref{fig:bow_context_span}). This implies that the visual concept is not localized in a single token, but is instead distributed across multiple correlated context tokens. Second, semantically related concepts share much of this contextual structure (e.g., \emph{a horse in a field''} vs.\ \emph{a donkey in a field''}), which leads to overlapping activation regions in the model’s internal representation space.
Whether this overlap is already present in text space depends on the encoder. Under our estimator, SDXL's CLIP text embeddings place nine of ten concepts in overlapping regions, whereas the T5 encoder of FLUX shows no measurable overlap for any of our five targets (Appendix~\ref{app:textspace}). Entanglement in the denoiser therefore need not be inherited from the text encoder; its source and severity depend on both the text representation and the denoiser computation. We treat this comparison as descriptive rather than as a causal decomposition.

\Cref{fig:concept_leakage} illustrates an example of the trade-off due to this entanglement. It shows two unlearning methods using \texttt{horse} as the erasure target.
Method that perform localized erasure (ESD~\cite{gandikota2023erasing}) successfully suppress generation when the target token is named, but regenerate the concept through indirect prompts that name only context words, which we call \emph{concept leakage}. 
Methods that erase more aggressively (STEREO~\cite{srivatsan2025stereo}) suppress this leakage but produce visibly degraded outputs for some semantically similar concepts, such as \texttt{donkey} and \texttt{zebra}, that are entangled with the target concept (see~\cref{tab:kappa}). We call this effect \emph{collateral damage}.

Table~\ref{tab:main_results_avg} quantifies this coupling. Averaged over five targets, STEREO achieves near-complete erasure (UA $=99.8$, IRR $=3.5$) but retains only NP $=18.0$ on semantic neighbors, whereas SEOT preserves neighbors (NP $=85.1$) but leaves an indirect recovery rate of $24.5$. The damage also tracks entanglement: STEREO's NP falls monotonically from $31.96$ on \texttt{castle} ($\hat\kappa=0.27$) to $1.81$ on \texttt{dog} ($\hat\kappa=0.89$; Section~\ref{subsec:kappa_scaling}).

\begin{figure*}[htbp]
\centering
\begin{minipage}[c]{0.5\linewidth}
\centering
\setlength{\tabcolsep}{4pt}
\renewcommand{\arraystretch}{0.9}
\begin{tabular}{c c c c}
 & \textbf{Direct} & \textbf{Indirect} & \textbf{Neighbor} \\
\midrule
\textbf{ESD} &
\includegraphics[width=0.28\linewidth]{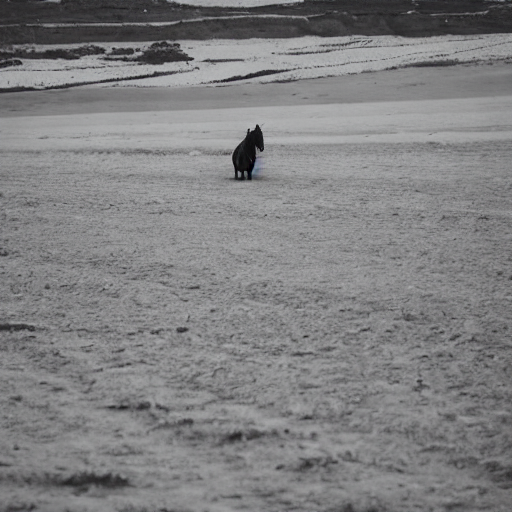} &
\includegraphics[width=0.28\linewidth]{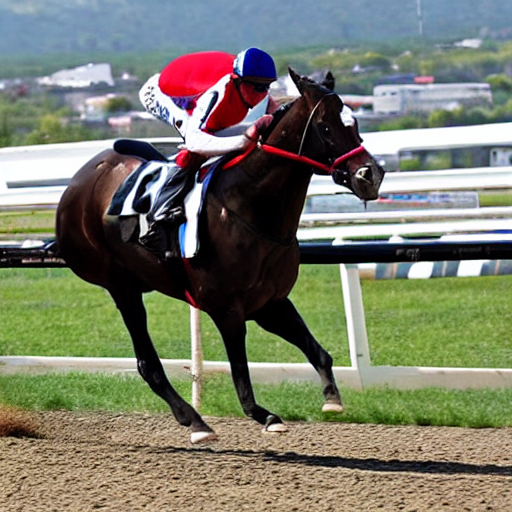} &
\includegraphics[width=0.28\linewidth]{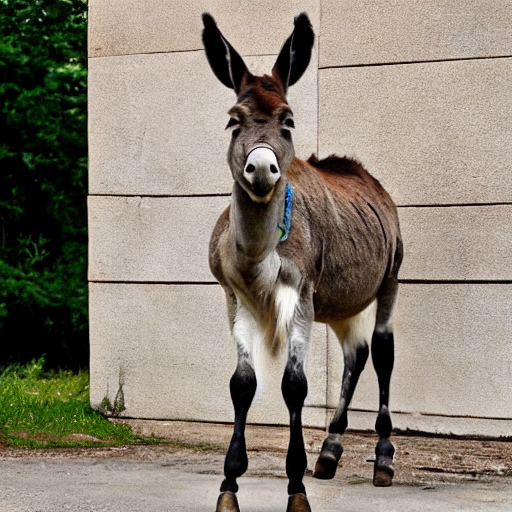} \\
\textbf{STEREO} &
\includegraphics[width=0.28\linewidth]{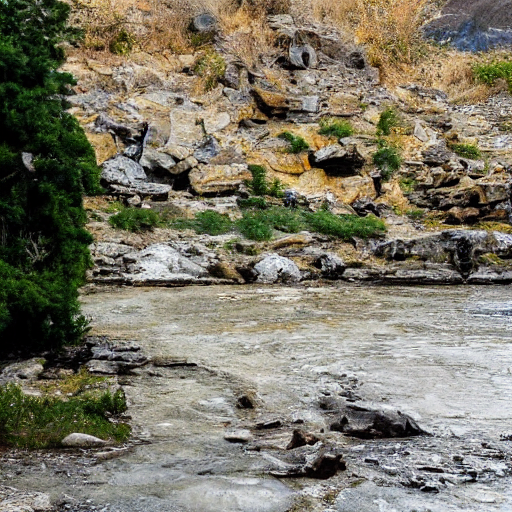} &
\includegraphics[width=0.28\linewidth]{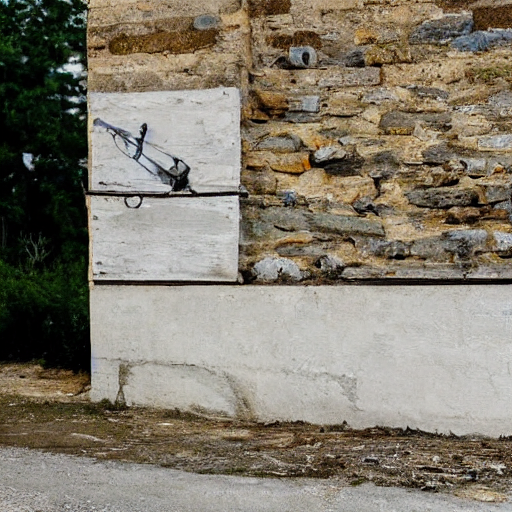} &
\includegraphics[width=0.28\linewidth]{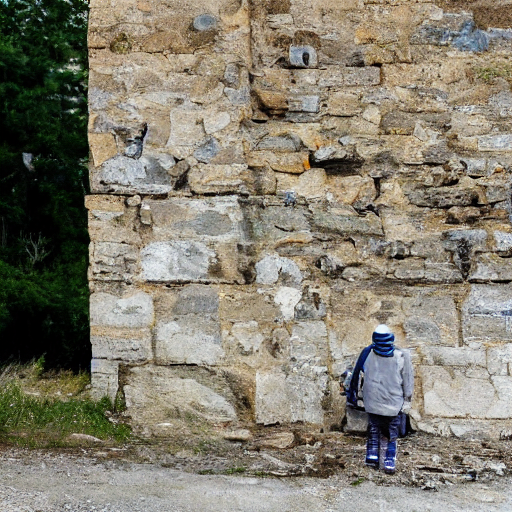} \\
\end{tabular}
\end{minipage}%
\hfill
\begin{minipage}[c]{0.38\linewidth}
\caption{
\textbf{Concept entanglement manifests as two coupled effects.}
Each row erases {horse} and is probed with three prompt types: \emph{Direct} (mentions ``\texttt{horse}''), \emph{Indirect} (uses correlated context like \texttt{jockey} but omits ``\texttt{horse}''), and \texttt{Neighbor} (semantically similar concept \texttt{donkey}). ESD leaks horses through indirect prompts but preserves donkey; STEREO suppresses leakage but destroys donkey. We formalize this duality in~\cref{sec:tradeoff}.
}
\label{fig:concept_leakage}
\end{minipage}
\end{figure*}

%% file: sections/04Proof.tex
\section{The Robustness-Retention Trade-off: A Formal Impossibility Result}
\label{sec:tradeoff}

We now formalize the coupling observed in~\cref{sec:motivation} as a provable lower bound on the trade-off between erasure robustness and concept preservation. The governing quantity is the \emph{concept overlap coefficient} $\kappa$, which measures the fraction of the target concept's activation region shared with the remaining concepts.

\subsection{Setup and Definitions}
\label{subsec:definitions}

Let $\mathcal{D}_\theta$ be a pretrained text-to-image diffusion model with U{-}Net noise predictor $\boldsymbol{\epsilon}_\theta$ and text encoder $f_\psi : \mathcal{P} \to \mathbb{R}^d$, where $\mathcal{P}$ is the space of text prompts. 
For a prompt $p \in \mathcal{P}$, we denote its embedding by $\mathbf{e}_p = f_\psi(p) \in \mathbb{R}^d$, and write $p_\theta(\mathbf{x} \mid p)$ for the conditional distribution over images induced by $\mathcal{D}_\theta$.
Let $\mathcal{D}_{\hat\theta}$ denote a concept-erased model (CEM) obtained by applying an unlearning procedure. 
We correspondingly write $\boldsymbol{\epsilon}_{\hat\theta}$ for the edited U-Net and $p_{\hat\theta}(\mathbf{x} \mid p)$ for the induced conditional distribution. 

\noindent{\bf Concept activation regions.}
For a fixed probe layer $\ell^*$ and denoising timestep $t^*$, let $\boldsymbol{\Phi}^{\ell^*}_\theta(p)$ denote the mean-pooled cross-attention activation produced by prompt $p \in \mathcal{P}$. 
For a concept $c \in \mathcal{C}$, we define its \emph{activation region} at layer $\ell^*$ as
\begin{equation}
    \mathcal{R}_c (\ell^*)
    \;=\;
    \Bigl\{
        \boldsymbol{\Phi}^{\ell^*}_\theta(p)
        \;\Big|\;
        p \in \mathcal{P},\;
        \mathbb{O}_{\mathcal{X}}\!\left(\mathbf{x}_{\theta}(p),\, c\right) = 1
    \Bigr\}
    \;\subseteq\; \mathbb{R}^{d_{\ell^*}},
\end{equation}
where (i) $\mathbb{O}_{\mathcal{X}} : \mathcal{X} \times \mathcal{C} \to \{0,1\}$ is an oracle indicating whether concept $c$ is present in image $x$, (ii) $\mathbf{x}_{\theta}(p) \sim p_{\theta}(\cdot \mid p)$ denotes an image sampled from $\mathcal{D}_\theta$ conditioned on prompt $p$, and (iii) $d_{\ell^*}$ is the per-token channel dimension at layer at layer $\ell^*$.
We also define the closed $\rho$-neighborhood of $\mathcal{R}_c$ as:
\begin{equation}
    \mathcal{R}_c^{(\rho)} (\ell^*)
    \;=\;
    \bigl\{\mathbf{h} \in \mathbb{R}^{d_{\ell^*}} \;\big|\;
    \inf_{\mathbf{r}  \in \mathcal{R}_c}\|\mathbf{h} - \mathbf{r}\|_2 \leq \rho\bigr\}
\end{equation}

When clear from context, we drop $\ell^*$ and simply write $\boldsymbol{\Phi}_\theta(p)$ and $\mathcal{R}_c^{(\rho)}$ for brevity.
The key quantity governing the trade-off is the degree of representational overlap between the target concept and other concepts, defined as follows.

\begin{definition}
\label{def:overlap}
The \emph{entanglement coefficient} of target concept $c_u$ with fattening radius $\rho > 0$ is
\[
    \kappa(c_u,\,\rho)
    \;=\;
    \frac{
        \mathrm{vol}\!\left(
            \mathcal{R}_{c_u}^{(\rho)}
            \cap
            \bigcup_{c' \in \mathcal{C} \setminus \{c_u\}} \mathcal{R}_{c'}^{(\rho)}
        \right)
    }{
        \mathrm{vol}\!\left(\mathcal{R}_{c_u}^{(\rho)}\right)
    },
\]
where $\mathrm{vol}(\cdot)$ denotes the $n$-dimensional volume of subsets of $\mathbb{R}^n$.
\end{definition}

\noindent Intuitively, $\kappa \in [0,1]$ measures what fraction of the target concept's representational footprint overlaps with that of other concepts in the model. When $\kappa = 0$, the concept is fully disentangled and can in principle be erased without collateral damage. When $\kappa > 0$, any edit that displaces activations in $\mathcal{R}_{c_u}$ will inevitably disturb overlapping portions of other concepts' activation regions.
This formulation aggregates overlap across all concepts. In practice, however, the contribution to $\kappa$ is dominated by concepts that occupy nearby regions of the activation space and distant concepts contribute negligibly to the intersection. 

\noindent{\bf Quantifying entanglement.}
\Cref{tab:kappa} reports $\hat\kappa$, an empirical lower-bound estimator of the entanglement coefficient $\kappa$ defined above, for five target concepts.
Each $\hat\kappa(c_u)$ is computed as the fraction of activation samples from $c_u$ that lie within a fattening radius $\rho_{c_u}$ of any activation from a discovered neighbor concept; the full estimation procedure, choice of $\rho_{c_u}$, and per-target results are detailed in~\cref{app:kappa_per_target}.

\paragraph{Neighborhood construction.} Neighbors are discovered automatically in the original activation space in two stages. For each target $c_u$ we extract activations for ten semantically plausible candidates and three unrelated controls, and keep as neighbors $\mathcal{N}(c_u)$ the candidates whose centroid cosine distance to $c_u$ falls below the 25th percentile of pairwise centroid distances in that pool (Definition~\ref{def:semantic_neighborhood}); the controls must fall outside $\mathcal{N}(c_u)$. $\hat\kappa$ is then computed from sample-level coverage, not from centroids: it is the fraction of target samples lying within the target's median intra-concept distance $\rho_{c_u}$ of some neighbor sample. Because $\hat\kappa$ measures union coverage, enlarging $\mathcal{N}(c_u)$ can only weakly increase it, but cardinality does not determine it: \texttt{horse} and \texttt{castle} have seven and six neighbors yet $\hat\kappa=0.81$ and $0.27$. The estimate is also stable in practice: for \texttt{horse}, \texttt{mare} alone gives $0.78$ and adding \texttt{pony} gives $0.81$, after which five further neighbors change nothing; moving the percentile threshold from the 10th to the 50th changes $|\mathcal{N}|$ from 3 to 9 for \texttt{horse} and from 1 to 9 for \texttt{castle} while leaving $\hat\kappa$ at $0.81$ and $0.27$ (Appendix~\ref{app:kappa-ablation}).

We now define the two desiderata that any unlearning method must satisfy:

\begin{definition}[$\varepsilon$-robust unlearning]
\label{def:robust}
The CEM $\mathcal{D}_{\hat\theta}$ achieves \emph{$\varepsilon$-robust unlearning} of concept $c_u$ if, for all prompts $p \in \mathcal{P}$,
\begin{equation}
    \Pr_{\mathbf{x} \sim p_{\hat\theta}(\cdot \mid p)}
\!\left[
\mathbb{O}_{\mathcal{X}}(\mathbf{x} ,\, c_u) = 1
\right]
\;\leq\; \varepsilon.
\end{equation}
\end{definition}
\noindent See~\cref{app:concept_regions} for more details on the oracle.

\begin{definition}[$\gamma$-concept preservation]
\label{def:neighbor}
Let $\mathcal{P}_{c'} \subseteq \mathcal{P}$ denote the set of prompts associated with concept $c'$, and let $\mathcal{P}_{-u}$ denote the \emph{non-target prompt distribution}, defined as the uniform mixture
$\mathcal{P}_{-u} \;=\; \tfrac{1}{|\mathcal{C}\setminus\{c_u\}|}\sum_{c' \in \mathcal{C}\setminus\{c_u\}} \mathcal{P}_{c'}$.
The CEM $\mathcal{D}_{\hat\theta}$ achieves \emph{$\gamma$-concept preservation} with respect to layer ${\ell^*}$ if
\begin{equation}
    \mathbb{E}_{p \sim \mathcal{P}_{-u}}\!\left[\bigl\|\boldsymbol{\Phi}^{\ell^*}_{\hat\theta}(p) - \boldsymbol{\Phi}^{\ell^*}_\theta(p)\bigr\|_2\right]
    \;<\; \gamma.
\end{equation}
\end{definition}

\subsection{Main Result: The Trade-off Lower Bound}
\label{subsec:main_result}

In this section, we show that $\varepsilon$-robust unlearning and $\gamma$-concept preservation are fundamentally in tension.
We begin by stating two mild regularity assumptions that formalize how perturbations in activation space affect generation.

\begin{assumption}[Lipschitz continuity of generation]\label{asm:lip}
For a model $\theta'$ and prompt $p$, let
$g_c(\theta',p)=\Pr_{x\sim p_{\theta'}(\cdot\mid p)}[\mathbb{O}_{\mathcal{X}}(x,c)=1]$.
For each concept $c\in\mathcal{C}$ there is a constant $L$ such that, for all
prompts $p,p'\in\mathcal{P}$ and all $\theta',\theta''\in\{\theta,\hat\theta\}$,
\begin{equation}
  \bigl|g_c(\theta',p)-g_c(\theta'',p')\bigr|
  \le L\,\bigl\|\boldsymbol{\Phi}_{\theta'}(p)-\boldsymbol{\Phi}_{\theta''}(p')\bigr\|_2 .
\end{equation}
\end{assumption}
\noindent\emph{Interpretation.} A single probe layer is not injective, and in a
U-Net the decoder also receives encoder activations through skip connections, so
we do not treat $\boldsymbol{\Phi}_\theta(p)$ as a sufficient statistic for generation.
Assumption~\ref{asm:lip} is a smoothness condition along prompt- and
model-dependent quantities: $L$ is an effective constant of the full network
that absorbs changes in the probe activation, the skip activations, and other
hidden states, rather than a property of the probe layer in isolation.
Appendix~\ref{app:lipschitz} estimates $L$ for $\theta'=\theta''=\theta$ over
natural prompt variations; the proof uses the case $\theta'=\theta$,
$\theta''=\hat\theta$, $p'=p$.

\begin{assumption}[Minimum displacement for erasure]
\label{asm:edit}
Any unlearning edit $\theta \mapsto \hat{\theta}$ that achieves $\varepsilon$-robust unlearning of $c_u$ must displace activations within $\mathcal{R}_{c_u}$ by at least $\Delta>0$:
\begin{equation}
\label{eq:delta_lower}
    \forall p \in \mathcal{P}
    \text{ such that }
    \boldsymbol{\Phi}_{\theta}(p)\in \mathcal{R}_{c_u}:
    \quad
    \|\boldsymbol{\Phi}_{\hat{\theta}}(p)-\boldsymbol{\Phi}_{\theta}(p)\|_2
    \geq
    \Delta .
\end{equation}
\end{assumption}
We empirically verify Assumption~\ref{asm:lip} across natural prompt variations and Assumption~\ref{asm:edit} across six unlearning methods erasing \texttt{horse}, including the gradient-based SalUn and TRUST; all measured pairs satisfy Lipschitz with a small empirical constant, 
and displacement propagates to other concepts through shared activation structure.
Full details are in~\cref{app:lipschitz,app:displacement}.

\begin{theorem}[Robustness-Retention trade-off]
\label{thm:tradeoff}
Let Assumptions~\ref{asm:lip} and~\ref{asm:edit} hold with constants $L$ and $\Delta$.
Let $\rho > 0$ be the fattening radius from Definition~\ref{def:overlap}.
Suppose $\mathcal{D}_{\hat\theta}$ achieves $\varepsilon$-robust unlearning of $c_u$
(Definition~\ref{def:robust}).
Then the aggregate collateral displacement over non-target concepts satisfies
\begin{equation}
    \mathbb{E}_{p\sim \mathcal{P}_{-u}}
    \left[
\|\boldsymbol{\Phi}_{\hat{\theta}}(p)-\boldsymbol{\Phi}_{\theta}(p)\|_2
    \right]
    \geq
    \kappa(c_u,\rho)\cdot \Delta .
\end{equation}
Consequently, achieving $\gamma$-concept preservation(Definition~\ref{def:neighbor}) requires
\begin{equation}
\label{eq:tradeoff}
    \boxed{
    \gamma \;\geq\; \kappa(c_u,\rho)\cdot\frac{1-\varepsilon}{L}.
    }
\end{equation}
Equivalently, when $\kappa(c_u,\rho) > 0$, achieving both $\varepsilon$-robustness and $\gamma$-concept preservation simultaneously requires
\begin{equation}
\label{eq:tradeoff_eps}
    \varepsilon \;\geq\; 1 - \frac{L\gamma}{\kappa(c_u,\rho)}.
\end{equation}
\end{theorem}

\begin{proof}[Proof sketch]
Assumptions~\ref{asm:lip} and~\ref{asm:edit} together force every activation in $\mathcal{R}_{c_u}$ to be displaced by at least $\Delta \geq (1-\varepsilon)/L$ under any $\varepsilon$-robust edit. By Definition~\ref{def:overlap}, a $\kappa(c_u,\rho)$-fraction of $\mathcal{R}_{c_u}^{(\rho)}$ lies in the support of $\mathcal{P}_{-u}$, so this displacement propagates to at least a $\kappa(c_u,\rho)$-fraction of $\mathcal{P}_{-u}$'s mass, yielding $\mathbb{E}_{p\sim\mathcal{P}_{-u}}[\|\Phi_{\hat\theta}(p)-\Phi_\theta(p)\|_2] \geq \kappa(c_u,\rho)(1-\varepsilon)/L$. Definition~\ref{def:neighbor} then gives the bound on $\gamma$. The full proof including the regularity assumption is in Appendix~\ref{app:proof}.
\end{proof}

\begin{remark}[Scope: any unlearning mechanism]\label{rem:scope}
Theorem~\ref{thm:tradeoff} makes no assumption about how $\hat\theta$ is
obtained. $\mathcal{D}_{\hat\theta}$ is any model achieving $\varepsilon$-robust
unlearning, so the bound applies equally to activation interventions,
closed-form weight edits, gradient-based fine-tuning such as SalUn and TRUST,
and inference-time methods, provided the resulting model displaces target
activations as in Assumption~\ref{asm:edit}. We verify this for
gradient-based methods in Appendix~\ref{app:displacement}: TRUST, for example,
displaces target activations by $15.82\pm3.82$, neighbors by $8.82\pm4.88$, and
controls by only $3.10\pm1.36$.
\end{remark}

\subsection{Consequences and Interpretation}
\label{subsec:consequences}

\noindent{\bf Limitations of simultaneous robust erasure and utility preservation.}
Based on~\Cref{thm:tradeoff}, simultaneous perfect erasure and perfect retention is infeasible whenever concepts are entangled:
\begin{corollary}[Impossibility result]
\label{cor:impossibility}
If $\kappa(c_u,\rho) > 0$,
then no unlearning method can simultaneously achieve $\varepsilon = 0$ (perfect erasure) and $\gamma = 0$ (perfect neighbor preservation).
\end{corollary}

\begin{proof}
Setting $\varepsilon = 0$ in Eq.~\eqref{eq:tradeoff} gives 
$\gamma \geq \kappa(c_u,\rho)/L > 0$
under the stated conditions.
\end{proof}

\noindent{\bf The Pareto frontier.}
Eq.~\eqref{eq:tradeoff_eps} defines a \emph{feasibility boundary} in the $(\varepsilon, \gamma)$ plane: no method can simultaneously achieve an erasure robustness below $\varepsilon$ and a neighbor preservation error below $\gamma$ unless the pair $(\varepsilon, \gamma)$ lies above the curve
$\gamma = \kappa \cdot \frac{1-\varepsilon}{L}.$
Different unlearning methods correspond to different operating points on or above this curve: aggressive methods achieve small $\varepsilon$ at the cost of large $\gamma$, occupying the upper-left of the frontier; conservative methods accept large $\varepsilon$ to maintain small $\gamma$, occupying the lower-right.
In Section~\ref{sec:experiment}, we verify that the empirically observed operating points across thirteen methods are consistent with this predicted boundary.

%% file: sections/05Experiment.tex
\section{Experiments}
\label{sec:experiment}
We design our experiments to answer three questions, each tied to a component of our contribution:


\begin{squishenumerate}
    \item[\textbf{Q1.}] Does the robustness-retention trade-off predicted by~\cref{thm:tradeoff} manifest empirically across a broad set of unlearning methods? (\cref{subsec:tradeoff_real})
    \item[\textbf{Q2.}] Does the severity of the trade-off vary with the estimated entanglement coefficient $\hat\kappa$, as predicted by the theorem's dependence on $\kappa$? (\cref{subsec:kappa_scaling})
    \item[\textbf{Q3.}] Do the Lipschitz continuity (\cref{asm:lip}) and minimum displacement (\cref{asm:edit}) assumptions underlying~\cref{thm:tradeoff} hold empirically? (\cref{app:lipschitz,app:displacement})
\end{squishenumerate}

\subsection{Experimental Setup}
\label{subsec:setup}

\noindent{\bf Base model and target concepts.}
We use CompVis Stable Diffusion v1.4~\citep{rombach2022high} as the base model $\mathcal{D}_\theta$. Following~\citep{gandikota2023erasing}, all generation uses 50 DDIM steps with guidance scale 7.5; 
full details in~\cref{app:implementation}.
We select five target concepts spanning the range $\hat\kappa \in [0.27, 0.89]$ identified in~\cref{tab:kappa}: 
four high-$\hat\kappa$ animal-subtype concepts (\texttt{horse}, \texttt{dog}, \texttt{cat}, \texttt{bear}, $\hat\kappa \in [0.77, 0.89]$) and one low-$\hat\kappa$ architectural concept (\texttt{castle}, $\hat\kappa = 0.27$). Each target is evaluated against its discovered concept neighborhood and three control concepts (\cref{tab:kappa_per_target}); per-target prompt sets are detailed in~\cref{app:eval_pipeline}.
To densify the intermediate range we add five targets, \texttt{snake}, \texttt{fish}, \texttt{tree}, \texttt{butterfly}, and \texttt{car} ($\hat\kappa=0.37$ to $0.83$), evaluated with four representative methods (Appendix~\ref{app:additional_concepts}). Cross-architecture experiments on SDXL and FLUX are described in Appendix~\ref{app:arch}.


\noindent{\bf Unlearning methods.}
We benchmark 13 methods across three intervention families: closed-form editing
(UCE~\citep{gandikota2024unified}, RECE~\citep{gong2024reliable});
fine-tuning
(ESD~\citep{gandikota2023erasing}, SalUn~\citep{fan2023salun},
EDiff~\citep{wu2024erasediff}, CPE~\citep{lee2025concept},
TRUST~\citep{kori2026selective}, AdvUnlearn~\citep{zhang2024defensive},
STEREO~\citep{srivatsan2025stereo}); and inference-time methods
(SEOT~\citep{li2024get}, SAeUron~\citep{cywinski2025saeuron},
SLD~\citep{schramowski2023safe}, SAFREE~\citep{yoon2025safree}).
AdvUnlearn and STEREO incorporate adversarial training to target the aggressive-erasure end of the trade-off. Method-specific configurations are detailed in~\cref{app:implementation}.

\paragraph{Experimental data and sources.} We use no existing labeled image test set. All images are generated by Stable Diffusion v1.4 or its unlearned variants from four controlled prompt families per target: direct prompts naming the target, indirect recovery prompts built from correlated context that excludes the target name, neighbor prompts naming concepts in $\mathcal{N}(c)$, and control prompts naming unrelated concepts. Correlated context terms are mined from target-matched LAION caption text~\citep{schuhmann2022laion}; LAION images are not used. For each target we generate around 200 images per prompt family (exact counts in~\cref{app:reproducibility}), all with 50 DDIM steps, guidance scale $7.5$, and fixed seeds shared across the original and unlearned models. Prompt files, seeds, the concept vocabulary, and discovered neighborhoods are released with the code (Appendices~\ref{app:implementation} and~\ref{app:evaluation_setup}).


\paragraph{Evaluation.} We report six metrics spanning the two axes of
Theorem~\ref{thm:tradeoff} (full definitions in Appendix~\ref{app:metric_defs}).
Every image is classified by \texttt{openai/clip-vit-base-patch32} over a fixed
65-concept vocabulary (Appendix~\ref{app:vocab}); each concept's text
representation averages normalized CLIP embeddings of four templates, and the
predicted label $\hat y(x)$ is the concept with the highest cosine similarity.
With $\mathcal{X}_{\text{dir}}(c)$ and $\mathcal{X}_{\text{ind}}(c)$ the
images from direct and indirect recovery prompts,
\begin{equation}
\mathrm{UA}(c)=1-\frac{1}{|\mathcal{X}_{\text{dir}}(c)|}\sum_{x\in\mathcal{X}_{\text{dir}}(c)}\mathbf{1}[\hat y(x)=c],
\qquad
\mathrm{IRR}(c)=\frac{1}{|\mathcal{X}_{\text{ind}}(c)|}\sum_{x\in\mathcal{X}_{\text{ind}}(c)}\mathbf{1}[\hat y(x)=c].
\end{equation}
IRR is our worst-case proxy for $\varepsilon$. On labeled images of the five
targets, the classifier attains a micro-averaged TPR of $97.0\%$ and FPR of
$0.55\%$ (Appendix~\ref{app:classifier}). We additionally report In-domain Retain Accuracy (IRA), the fraction of neighbor and control images classified as the concept they were meant to show.
\emph{Retention metrics} quantify collateral damage. 
\emph{Neighbor Preservation} (NP), is the classifier accuracy on generations from prompts naming semantically related concepts (e.g., \texttt{donkey}, \texttt{pony} when erasing \texttt{horse}). We additionally report \emph{Control Preservation} (CP) on unrelated concepts and $\mathrm{DamageGap} = \mathrm{CP} - \mathrm{NP}$, which isolates neighbor-selective damage from uniform degradation. All metrics are reported both absolutely and as relative changes from $\mathcal{D}_\theta$ to enable cross-concept comparison.

\subsection{The Trade-off is Real and Consistent}
\label{subsec:tradeoff_real}
We first test whether the robustness-retention trade-off predicted by~\Cref{thm:tradeoff} manifests empirically across methods. \Cref{tab:main_results_avg} reports each method's performance averaged over the five target concepts.
Per-concept breakdowns are provided in~\cref{app:per_concept_tables} 
\begin{table*}[t]
\centering
\small
\setlength{\tabcolsep}{2.5pt}
\begin{tabular}{llcccccc}
\toprule
& & \multicolumn{2}{c}{\textbf{Erasure}} & \multicolumn{3}{c}{\textbf{Retention}} & \\
\cmidrule(lr){3-4} \cmidrule(lr){5-7}
\textbf{Method} & \textbf{Type} & UA$\uparrow$ & IRR$\downarrow$ & IRA$\uparrow$ & NP$\uparrow$ & CP$\uparrow$ & DamageGap$\downarrow$ \\
\midrule
Original    & --  & $0.0 \pm 0.0$ & $58.7 \pm 38.1$  & $70.45 \pm 12.6$  & $100.0 \pm 0.0$   & $100.0 \pm 0.0$   & -- \\
\midrule
UCE    & Closed-form & $80.6 \pm 18.4$ & $10.5 \pm 23.4$ & $67.1 \pm 12.4$ & $88.6 \pm 25.1$ & $99.3 \pm 3.1$  & $10.7 \pm 25.4$ \\
RECE   & Closed-form & $99.4 \pm 0.9$  & $3.0 \pm 4.0$   & $43.3 \pm 12.1$ & $56.1 \pm 10.6$  & $93.5 \pm 2.0$  & $37.4 \pm 10.2$ \\
\midrule
ESD     & FT  & $80.7 \pm 11.1$  & $20.9 \pm 14.2$  & $59.2 \pm 12.7$  & $80.6 \pm 13.5$  & $97.0 \pm 1.6$  & $16.4 \pm 14.5$  \\
SalUn   & FT  & $93.7 \pm 5.1$  & $12.8 \pm 11.5$  & $46.5 \pm 8.5$  & $58.4 \pm 13.0$  & $90.2 \pm 8.9$  & $31.8 \pm 13.1$  \\
EDiff   & FT  & $79.4 \pm 12.2$  & $17.4 \pm 12.4$  & $55.6 \pm 12.3$  & $73.4 \pm 12.1$  & $95.8 \pm 3.2$  & $22.4 \pm 10.9$  \\
CPE    & FT          & $99.6 \pm 0.9$  & $4.7 \pm 9.4$   & $53.4 \pm 12.4$ & $71.9 \pm 18.1$  & $99.6 \pm 2.6$ & $27.7 \pm 19.3$ \\
TRUST    & FT          & $94.2 \pm 2.2$  & $6.5 \pm 3.5$   & $53.2 \pm 13.7$ & $76.4 \pm 9.1$  & $98.0 \pm 5.0$ & $21.6 \pm 12.2$ \\
\midrule
AdvUnlearn & Robust FT   & $98.5 \pm 1.7$  & $9.0 \pm 11.5$   & $55.4 \pm 13.9$  & $71.5 \pm 7.5$ & $99.4 \pm 0.9$  & $27.9 \pm 7.3$  \\
STEREO  & Robust FT & $99.8 \pm 3.1$  & $3.5 \pm 3.5$   & $19.7 \pm 10.0$  & $18.0 \pm 12.0$  & $70.8 \pm 27.4$  & $52.8 \pm 22.7$  \\
\midrule
SEOT    & Inf. & $82.6 \pm 13.2$  & $24.5 \pm 13.8$  & $59.3 \pm 15.0$  & $85.1 \pm 13.8$  & $98.3 \pm 1.3$ & $13.6 \pm 13.0$  \\
SAeUron$^\dagger$ & Inf. & $78.4 \pm 5.5$ & $27.6 \pm 12.7$ & $42.1 \pm 31.8$ & $59.0 \pm 37.8$ & $89.2 \pm 10.8$ & $30.2 \pm 27.2$  \\
SLD    & Inf. & $99.6 \pm 0.9$  & $0.5 \pm 0.9$   & $35.0 \pm 8.1$  & $58.8 \pm 16.8$  & $89.1 \pm 4.1$  & $30.3 \pm 16.3$ \\
SAFREE & Inf. & $91.4 \pm 11.4$ & $8.4 \pm 10.2$  & $39.9 \pm 8.6$  & $64.7 \pm 16.0$  & $85.7 \pm 12.2$ & $21.0 \pm 19.4$ \\
\bottomrule
\end{tabular}
\caption{\textbf{Averaged evaluation of unlearning methods on the robustness-retention trade-off.} Results are averaged over the five target concepts (\texttt{dog}, \texttt{bear}, \texttt{horse}, \texttt{cat}, \texttt{castle}) with $\hat\kappa \in [{0.27}, {0.89}]$. Cells report mean $\pm$ sample standard deviation across concepts. UA: Unlearning Accuracy; IRR: Indirect Recovery Rate; IRA/CP: In-domain/Cross-domain Retain Accuracy; NP: Neighbor Preservation (Eq.~\eqref{eq:neighbor_pres}); DamageGap: composite retention-damage summary. The \emph{Original} row reports baseline rates for the pretrained $\mathcal{D}_\theta$. A few runs use non-standard image counts; see the per-concept table captions in~\cref{app:per_concept_tables}. $^\dagger$SAeUron is averaged over four concepts (\texttt{dog}, \texttt{bear}, \texttt{horse}, \texttt{cat}); its released sparse autoencoder features do not cover \texttt{castle} (see~\cref{app:per_concept_castle}).}
\label{tab:main_results_avg}
\end{table*}

\noindent{\bf No method escapes the trade-off.}
\Cref{tab:main_results_avg} reveals a consistent pattern across thirteen methods: no method simultaneously achieves high erasure strength (high UA, low IRR) and high neighbor retention. Averaged over concepts, no method reaches both UA $>90$ and NP $>80$. The most balanced operating points come from methods designed to preserve non-target concepts, yet they remain on the frontier: TRUST reaches UA $94.2$ at NP $76.4$, and UCE attains the highest NP ($88.6$) but with weaker and highly variable erasure (UA $80.6\pm18.4$). Methods that achieve near-complete erasure (RECE, CPE, SLD, STEREO; UA $\geq 99.4$) retain NP between $18.0$ and $71.9$. Mechanism family does not determine the operating point: the inference-time method SLD attains the lowest IRR ($0.5$) at NP $58.8$, whereas SEOT, also an inference-time method, preserves neighbors (NP $85.1$) but leaks (IRR $24.5$). Algorithm design thus determines where a method lies on the frontier, while concept entanglement constrains the frontier it faces, consistent with Corollary~\ref{cor:impossibility}. Appendix~\ref{app:method_analysis} interprets these methods as operating points induced by their effective displacement $\Delta$, and~\cref{fig:pareto_per_concept} plots every per-concept operating point of the seven original methods.

\noindent{\bf The trade-off is structural, not a measurement artifact.}
The DamageGap column spans more than 40 points across methods (from UCE at $10.7$ to STEREO at $52.8$), and the NP column spans more than 70 points (from STEREO at $18.0$ to UCE at $88.6$). If the trade-off were driven by noise or evaluation instability, we would expect method-level differences to be dominated by concept-level variation; instead, method identity is the primary determinant of operating region, with concept identity modulating \emph{severity} of the trade-off at each operating point. We characterize this concept-level scaling next (\cref{subsec:kappa_scaling}) and provide a detailed per-concept view for \texttt{horse} in~\cref{app:per_concept_tables}.

\subsection{Semantic Density Governs Trade-off Severity}
\label{subsec:kappa_scaling}
\Cref{thm:tradeoff} predicts that trade-off severity scales with the entanglement coefficient $\kappa$: target concepts whose activation regions overlap substantially with other concepts should suffer greater collateral damage under robust erasure. We test this across five target concepts with estimated $\hat\kappa$ span from $0.27$ (\texttt{castle}) to $0.89$ (\texttt{dog}); see~\cref{tab:kappa} and~\cref{app:kappa_estimation} for the estimation procedure.

\noindent{\bf $\hat\kappa$ predicts collateral damage on neighbors.}
\Cref{fig:kappa_np_scaling} plots NP against $\hat\kappa$ for the available method-concept pairs among the seven original methods, and the table on the right reports STEREO's NP and in-domain retention ratio ordered by ascending $\hat\kappa$.
Both decrease monotonically as $\hat\kappa$ increases, consistent with the linear $\kappa$-dependence in Eq.~\eqref{eq:tradeoff}. \Cref{fig:kappa_np_scaling} reveals three patterns: robust fine-tuning methods (red squares) exhibit the steepest $\hat\kappa$-dependence, with STEREO tracking $\hat\kappa$ monotonically and AdvUnlearn similarly decreasing; inference-time methods (green triangles) show only weak $\hat\kappa$-dependence because their effective displacement $\Delta$ is too small for the $\kappa$-scaling to dominate; standard fine-tuning methods (blue circles) occupy the middle band. This stratification reflects a structural principle: $\hat\kappa$ governs the lower bound on collateral damage, while the method's effective $\Delta$ determines how close it operates to that bound.
\begin{figure}[t]
\centering
\begin{minipage}[c]{0.45\linewidth}
\centering
\includegraphics[width=\linewidth]
{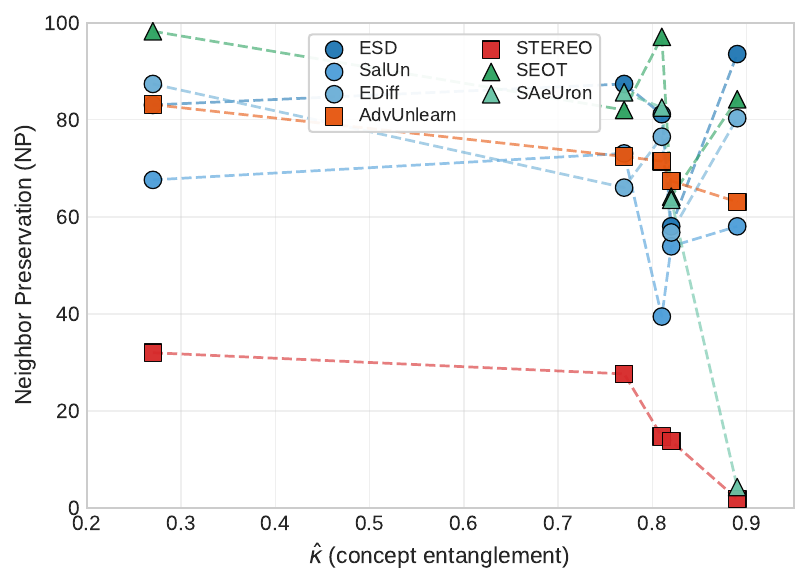}
\end{minipage}\hfill
\begin{minipage}[c]{0.45\linewidth}
\centering
\scriptsize
\setlength{\tabcolsep}{5pt}
\begin{tabular}{lccc}
\toprule
\textbf{Concept} & $\hat\kappa$ & NP$\uparrow$ & IRA-ratio$\uparrow$ \\
\midrule
\texttt{castle} & 0.27 & 31.96 & 0.55 \\
\texttt{cat}    & 0.77 & 27.60 & 0.32 \\
\texttt{horse}  & 0.81 & 14.71 & 0.29 \\
\texttt{bear}   & 0.82 & 13.75 & 0.25 \\
\texttt{dog}    & 0.89 & \phantom{0}1.81 & 0.06 \\
\midrule
\multicolumn{2}{l}{$\rs$} & $-1.0$ & $-1.0$ \\
\bottomrule
\end{tabular}
\end{minipage}
\caption{\textbf{Neighbor preservation scales inversely with $\hat\kappa$.} \emph{Left:} NP for each (method, concept) pair among the seven original methods; dashed lines connect each method's points in order of $\hat\kappa$ and serve as visual guides, not fitted trends. \emph{Right:} STEREO's NP and in-domain retention ratio across the five concepts ($\rs=-1.0$ for both).}
\label{fig:kappa_np_scaling}
\label{tab:kappa_scaling}
\vspace{-2mm}
\end{figure}

\paragraph{Direct test of Eq.~\eqref{eq:tradeoff}.} Eq.~\eqref{eq:tradeoff} predicts that, at matched erasure strength, a larger $\kappa$ forces a larger collateral displacement $\gamma$. STEREO provides this controlled comparison: it reaches UA $=100$ on four of five targets and $93.0$ on \texttt{cat}. Measured directly in activation space (Appendix~\ref{app:displacement_kappa}, Table~\ref{tab:displacement}), its mean neighbor displacement rises from $75.9$ on \texttt{castle} ($\hat\kappa=0.27$) to $234.2$, $233.4$, $240.9$, and $257.0$ on \texttt{cat}, \texttt{horse}, \texttt{bear}, and \texttt{dog} ($\rs=0.90$). All five fine-tuning methods show a positive correlation, four of them at $\rs\geq0.70$, and ESD is strictly monotone ($\rs=1.00$). Inverting Eq.~\eqref{eq:tradeoff} at each STEREO operating point yields $L_{\text{tight}}\in[0.0032,0.0035]$ across concepts, a $7.8\%$ spread, consistent with a single local Lipschitz constant for a method operating near the bound. Table~\ref{tab:main_results_avg} therefore establishes the trade-off, while this analysis tests the $\kappa$-dependence of the bound in the units in which it is stated; NP serves as a probability-level proxy for $\gamma$.

\paragraph{Denser coverage of $\hat\kappa$.} To test whether the trend depends on \texttt{castle} alone or on an animal versus architecture split, we add \texttt{snake}, \texttt{fish}, \texttt{tree}, \texttt{butterfly}, and \texttt{car} ($\hat\kappa=0.37$, $0.52$, $0.63$, $0.66$, $0.83$) and evaluate EDiff, ESD, SalUn, and STEREO (Appendix~\ref{app:additional_concepts}). Method-averaged NP decreases monotonically with $\hat\kappa$, from $65.0$ to $58.8$ ($\rs=-1.0$), and every within-method correlation is negative (EDiff $-0.4$, ESD $-0.7$, SalUn $-0.2$, STEREO $-0.7$). With five additional targets these correlations are descriptive, not individually significant; they show that the direction of the effect persists across animal, plant, and vehicle concepts, including a low-$\hat\kappa$ animal and a high-$\hat\kappa$ vehicle. The $\hat\kappa$ ordering is stable across probe layers ($\rs\geq0.80$) and denoising timesteps ($\rs=0.93$ and $0.96$ at steps 10 and 40 relative to step 25), although absolute values vary (Appendix~\ref{app:layer_sensitivity}). A discussion of why $\hat\kappa$ reflects activation overlap rather than semantic category is in Appendix~\ref{app:kappa_estimation}.

\paragraph{Beyond Stable Diffusion v1.4.} On SDXL, whose cross-attention layers admit $\hat\kappa$, ESD's neighbor preservation decreases monotonically with entanglement: NP $=90.9$, $85.0$, and $77.1$ for \texttt{cat}, \texttt{castle}, and \texttt{dog} ($\hat\kappa_{\text{SDXL}}=0.65$, $0.72$, $0.92$; $\rs=-1.0$ over three concepts). The most entangled target is also the least erased (UA $=77$ for \texttt{dog} versus $99$ for \texttt{cat}), so the ordering is not driven by erasure strength. Absolute $\hat\kappa$ shifts across architectures (\texttt{castle} is $0.27$ on SD v1.4 and $0.72$ on SDXL), but $\hat\kappa$ and observed damage agree within each architecture. On the MMDiT model FLUX, joint text-image attention leaves no cross-attention block analogous to our probe, and a proxy estimator saturates near $1$ even for unrelated control pairs, so $\kappa$-scaling on MMDiT remains untested. The phenomenon itself recurs: ESD on FLUX reduces NP to $35.6$ to $40.5$ on \texttt{dog}, \texttt{cat}, and \texttt{castle} while preserving unrelated controls (CP $=95.3$ to $98.9$), i.e., neighbor-selective rather than global degradation (Appendix~\ref{app:arch}).

%% file: sections/06Discussion.tex
\section{Conclusion}
\label{sec:conclusion}
We show that the tension between robust concept erasure and utility preservation in diffusion models is a structural consequence of representation geometry, not merely a limitation of existing methods.
\Cref{thm:tradeoff} establishes a $\kappa$-scaled lower bound on collateral damage, and our experiments show that existing methods, including those designed to preserve non-target concepts, occupy different points on the resulting robustness-retention frontier, and the frontier becomes steeper as $\hat\kappa$ increases.

\paragraph{Implications.} Three consequences follow. First, $\hat\kappa$ is a pre-deployment diagnostic: estimated from base-model activations alone, it indicates which targets admit clean erasure and which require accepting collateral damage, before any unlearning is run (\cref{app:clean-erasure}). Second, benchmarks should report erasure and retention jointly, on neighbor sets selected by activation overlap, because reporting either axis alone hides where a method sits on the frontier. Third, method design should be $\kappa$-aware: rather than suppressing broad adversarial subspaces, methods should restrict displacement to the directions needed for erasure.

%% file: sections/07Limitation.tex
\section{Limitations}
\label{sec:limitations}

\paragraph{Scope of empirical evaluation.} All thirteen methods are evaluated on five targets on Stable Diffusion v1.4; five further targets are evaluated with four methods. Cross-architecture evidence is limited: SDXL covers three concepts with one method, and on FLUX we can show neighbor-selective damage but cannot resolve $\kappa$, so $\kappa$-scaling on MMDiT remains open. Our targets are object concepts; artistic styles, identities, and harmful-content concepts remain untested.

\paragraph{Theory and $\hat\kappa$ estimation.} Theorem~\ref{thm:tradeoff} applies to any unlearning mechanism, but its constants are effective ones: $L$ captures the full network, including skip connections, and is estimated only for the prompt variations of Appendix~\ref{app:lipschitz}. $\hat\kappa$ is an empirical lower-bound proxy for $\kappa$ whose absolute value depends on the probe layer, the denoising timestep (e.g., $\hat\kappa(\texttt{castle})=0.77$ at step 10), and the candidate pool; our claims rely on its ranking, which is stable across the settings we test. The text-space comparison is not a causal decomposition of where entanglement arises.

%% file: sections/app/01.tex
\section{Formal Definition of Concept Activation Regions}
\label{app:concept_regions}
\subsection{Background: latent diffusion models}\label{app:background}
We briefly introduce the latent diffusion notation used throughout the paper, since our analysis of concept activation regions and unlearning-induced perturbations is defined over the model's latent, text-conditioned denoising process.

\textbf{Text-to-image latent diffusion models.}
We consider a text-to-image latent diffusion model (LDM) of the form introduced in Stable Diffusion~\citep{rombach2022high}.
An input image $\mathbf{x} \in \mathcal{X}$ is first mapped by a pretrained variational autoencoder encoder $E$ into a latent representation 
$\mathbf{z} = E(x) \in \mathcal{Z}$, and reconstructed by a decoder $V : \mathcal{Z} \rightarrow \mathcal{X}$. 

Given a latent $\mathbf{z}$, the forward diffusion process corrupts it with Gaussian noise to obtain $\mathbf{z}_t$ at timestep $t \in \{1,\dots,T\}$. 
A time-conditioned U-Net $\boldsymbol{\epsilon}_\theta$ is then trained to predict the noise $\boldsymbol{\epsilon}$ 
added to $\mathbf{z}_t$. 
In text-conditioned generation, a prompt $p \in \mathcal{P}$ is encoded by a text encoder $f_{\psi}$ into text features $\mathbf{e}_p = f_{\psi}(p)$, and these features condition the U-Net through cross-attention layers. 
The standard latent diffusion objective is
\begin{equation}
\label{eq:ldm_loss}
\mathcal{L}_{\mathrm{LDM}}
=
\mathbb{E}_{x,\epsilon,t}
\left[
\left\|
\boldsymbol{\epsilon} - \boldsymbol{\epsilon}_\theta(\mathbf{z}_t, t, \mathbf{e}_p)
\right\|_2^2
\right],
\end{equation}
where $\epsilon \sim \mathcal{N}(0,I)$ 
At inference time, generation starts from Gaussian noise and iteratively denoises toward a clean latent, which is finally decoded into an image.

\subsection{Concept activation regions}
This appendix provides the formal construction of the concept activation function $\boldsymbol{\Phi}_\theta$ and activation region $\mathcal{R}_c$ used throughout~\cref{sec:tradeoff}. The compressed version in the main text specifies these informally; the definition here is the formal object used in the proofs of~\cref{thm:tradeoff}.

\begin{definition}[Concept activation function and region]
\label{def:activation}
Fix a UNet layer $\ell^*$ and denoising timestep $t^* \in \{1,\ldots,T\}$. Let $p \in \mathcal{P}$ denote a text prompt and $\mathbf{e}_p = f_\psi(p)$ be its text embedding. The \emph{concept activation function} of the pretrained model is
\begin{equation}
    \boldsymbol{\Phi}_{\theta, \ell^*,t^*}(p) \;=\; \mathrm{A}_{\theta}(\mathbf{e}_p,\ell^*,t^*) \in \mathcal{H},
\end{equation}
where $\mathrm{A}_{\theta}(\mathbf{e}_p, \ell^*, t^*)$ denotes the cross-attention output 
of $\epsilon_\theta$ at layer $\ell^*$ and timestep $t^*$, and $\mathcal{H} := \mathbb{R}^{d_{\ell^*}}$ is the corresponding activation space. 
Throughout, we suppress the dependence on $\ell^*$ 
and $t^*$ and write $\boldsymbol{\Phi}_\theta(p)$ for brevity.
Analogously, $\boldsymbol{\Phi}_{\hat{\theta}}(p)$ denotes the activation function of the CEM $\mathcal{D}_{\hat\theta}$.

For a concept $c \in \mathcal{C}$, its \emph{activation region} is
\begin{equation}
    \mathcal{R}_c
    \;=\;
    \Bigl\{
        \boldsymbol{\Phi}_{\theta}(p)
        \;\Big|\;
        p \in \mathcal{P},\;
        \mathbb{O}_{\mathcal{X}}\!\left(\mathbf{x}_{\theta}(p),\, c\right) = 1
    \Bigr\}
    \;\subseteq\; \mathcal{H},
\end{equation}
where $\mathbb{O}_{\mathcal{X}} : \mathcal{X} \times \mathcal{C} \to \{0,1\}$ is an oracle indicating whether concept $c$ is present in image $\mathbf{x}$, and $\mathbf{x}_{\theta}(p) \sim p_{\theta}(\cdot \mid p)$ denotes an image sampled from $\mathcal{D}_\theta$ conditioned on prompt $p$.
\end{definition}

\paragraph{On the oracle classifier.}
The predicate $c \in \text{image}(\mathbf{x})$ is specified as an oracle classifier rather than a specific model because our results are independent of the particular choice. Any reasonable concept detector induces a concept region $\mathcal{R}_c$, and the tradeoff in~\cref{thm:tradeoff} holds with respect to whichever classifier is fixed. In~\cref{sec:experiment}, we instantiate this oracle with a CLIP-based classifier, detailed in~\cref{app:evaluation_setup}. 

%% file: sections/app/02.tex
\section{Empirical Validation of Concept Activation Regions}
\label{app:activation_regions}
This section provides empirical support for~\cref{def:activation} by showing that the cross-attention activations used in our formalization form coherent geometric clusters in practice.
\paragraph{Setup.}
To verify that concept activation regions are well-defined empirical objects, we extract cross-attention activations from Stable Diffusion v1.4 at the mid-block layer (\texttt{mid\_block.attentions.0}) at step $t^*{=}25$ of 50 DDIM steps for 10 concepts, each prompted with 50 shared templates differing only in the concept word.
\texttt{horse}, \texttt{pony}, \texttt{donkey}, \texttt{deer}, \texttt{dog}, \texttt{cat}, \texttt{bear}, \texttt{car}, \texttt{truck}, and \texttt{castle}.
This visualization pool is chosen to be illustrative rather than comprehensive: it deliberately spans three semantically distinct domains (animals, vehicles, architecture) so that the qualitative geometry of concept regions (semantically related concepts clustering while distinct domains separate) is visually apparent. Several concepts in this pool (\texttt{deer}, \texttt{car}, \texttt{truck}) are not part of the evaluation vocabulary in~\cref{tab:vocab}; they appear only to anchor the visualization's domain diversity. Conversely, $\hat\kappa$-estimation and all CLIP-based evaluation metrics use the full $65$-concept pool of~\cref{tab:vocab}, of which the visualization pool is not a strict subset.
For each concept, we generate 50 prompts using a shared set of templates that differ only in the concept word, so that differences in activation primarily reflect the concept token rather than unrelated prompt variation.

For each prompt $p$, we compute the activation
\begin{equation}
    \boldsymbol{\Phi}_{\theta}(p) \;=\; \mathrm{A}_{\theta}(\mathbf{e}_p,\ell^*,t^*) \in \mathcal{H},
\end{equation}
flatten the resulting cross-attention tensor into a vector in $\mathcal{H}$, and project all activations to two dimensions using UMAP for visualization.

\paragraph{Observations.}
\Cref{fig:umap} shows that activations induced by prompts containing the same concept form concentrated local clusters, rather than being arbitrarily scattered in activation space. This supports the view that each concept corresponds to a coherent region in the model's internal representation space.

The figure also reveals non-uniform distances between concepts. Some concepts occupy nearby regions (for example, \texttt{horse}, \texttt{pony}, and \texttt{donkey} form a tight cluster), while others, such as \texttt{castle}, are clearly separated from the animal concepts. Importantly, we do not use these visualizations to infer semantic relatedness from scratch. Rather, the figure shows that the activation geometry is structured and that concepts commonly regarded as related in the data domain tend to induce nearby activations under shared prompt templates.

These results do not by themselves prove overlap in the strict set-theoretic sense used later in Section~4. Instead, they provide empirical evidence for the weaker claim needed to motivate our definitions: concept-conditioned activations are localized, stable enough to visualize, and often organized into neighboring regions rather than isolated points. This makes it meaningful to model a concept through an activation region $\mathcal{R}_c$ and to reason about neighborhood structure in the induced activation space.

\paragraph{Choice of probe layer and timestep.}
We probe the cross-attention output at \texttt{mid\_block.attentions.0} at denoising step $t^*=25$ of $50$ throughout~\cref{app:activation_regions,app:kappa_estimation,app:displacement,app:lipschitz}, and at \texttt{up\_blocks.1.attentions.1} (same timestep) for the direct displacement measurements in~\cref{app:displacement_kappa}. The probe layer $\ell^*$ is a parameter of our framework (\cref{def:activation}), not a property of the model, and~\cref{thm:tradeoff} holds for any chosen $(\ell^*, t^*)$ at which Assumptions~\ref{asm:lip} and~\ref{asm:edit} are satisfied. We selected mid-block as the default probe based on a sensitivity analysis (\cref{app:layer_sensitivity}) over six candidate cross-attention layers spanning the U-Net depth: the $\hat\kappa$ ordering across our five target concepts is stable across all probed layers ($\rs \geq 0.80$ vs.\ mid-block), but mid-block produces the cleanest bimodal separation between high-$\hat\kappa$ and low-$\hat\kappa$ concepts and the tightest within-cluster spread among the animal subtypes. Step $t^*=25$ (the midpoint of the DDIM trajectory) corresponds to the regime in which large-scale layout and object identity are committed in latent diffusion models. The up-block layer in~\cref{app:displacement_kappa} was selected because it lies downstream of the cross-attention edits applied by fine-tuning methods and therefore reflects the cumulative effect of those edits; \cref{app:layer_sensitivity} confirms that the $\hat\kappa$ ranking at this layer is consistent with mid-block.

%% file: sections/app/03.tex
\section{Estimation of the Entanglement Coefficient $\hat\kappa$}
\label{app:kappa_estimation}

The entanglement coefficient $\kappa(c_u, \rho)$ defined in~\cref{def:overlap} involves the volume of high-dimensional activation regions, which is intractable to compute directly. We therefore estimate $\kappa$ via a sample-based procedure that restricts the overlap measurement to a discovered \emph{semantic neighborhood} $\mathcal{N}(c_u) \subseteq \mathcal{C} \setminus \{c_u\}$ rather than the full union over remaining concepts. Because $\mathcal{N}(c_u)$ is a subset of the union appearing in~\cref{def:overlap}, this neighborhood-restricted estimate is a \emph{lower bound} on $\kappa(c_u, \rho)$. 
We argue, and verify in~\cref{app:kappa_per_target}, that this lower bound is tight in practice: across the five targets, the activation overlap with non-neighbors is negligible, so the union over $\mathcal{C} \setminus \{c_u\}$ is dominated by the union over $\mathcal{N}(c_u)$. 
Throughout this section we work with cosine distance $d_{\cos}(\mathbf{a}, \mathbf{b}) = 1 - \langle \mathbf{a}, \mathbf{b} \rangle / (\|\mathbf{a}\|\|\mathbf{b}\|)$, which is the standard choice in representation analysis: it is invariant to activation magnitude (which is dominated by token-position effects under mean-pooling) and captures directional similarity, the property most closely tied to concept identity. 
The centroid-distance matrix used to motivate neighborhood structure in~\cref{sec:tradeoff} (\cref{fig:distance_matrix}) uses the same metric, ensuring internal consistency between our motivation and our formalization.

\paragraph{On the relationship between cosine and $\ell_2$ metrics.}
\Cref{thm:tradeoff} bounds the $\ell_2$ displacement $\|\boldsymbol{\Phi}_{\hat\theta}(p) - \boldsymbol{\Phi}_{\theta}(p)\|_2$, since this quantity measures a physical perturbation to the activation vector. Our $\kappa$ estimator, in contrast, uses cosine distance to characterize region membership. These two choices are compatible: for activations normalized to the unit sphere, $\|\mathbf{a} - \mathbf{b}\|_2^2 = 2 \cdot d_{\cos}(\mathbf{a}, \mathbf{b})$, so an $L$-Lipschitz function under the $\ell_2$ metric satisfies $|f(a)-f(b)|\le L\sqrt{2\,d_{\cos}(a,b)}$ on the unit sphere, i.e., it is H\"older continuous with exponent $1/2$ in cosine distance. The theorem is stated and proved entirely in $\ell_2$; cosine distance is used only to decide region membership in the estimator, and since both metrics induce the same fattened regions up to a monotone re-parameterization of $\rho$, this choice does not affect the bound. 

\paragraph{Activation extraction.}
For each concept $c$, we generate images from 50 shared prompt templates (differing only in the concept word) using the base model (Stable Diffusion v1.4) with DDIM sampling (50 steps, guidance scale 7.5). At step $t^*{=}25$, we extract the cross-attention output at \texttt{mid\_block.attentions.0}, take the text-conditioned component, and mean-pool over spatial dimensions to obtain a single activation vector $\boldsymbol{\Phi}_\theta(p) \in \mathbb{R}^d$ per prompt.

\paragraph{Neighborhood discovery.}
We restrict the overlap measurement in~\cref{def:overlap} to a finite set of candidate concepts that are close to $c_u$ in activation space. We formalize this set as follows.

\begin{definition}[Semantic neighborhood]
\label{def:semantic_neighborhood}
For each target $c_u$, let $\mathcal{C}_{c_u}=\{c_u\}\cup\mathcal{C}^{\text{cand}}_{c_u}\cup\mathcal{C}^{\text{ctrl}}_{c_u}$
be a target-specific pool, where $\mathcal{C}^{\text{cand}}_{c_u}$ contains ten
semantically plausible candidate neighbors of $c_u$ and $\mathcal{C}^{\text{ctrl}}_{c_u}$
contains three unrelated control concepts. For each
$c\in\mathcal{C}_{c_u}$, let $\bar{\mathbf{a}}_c=\frac{1}{|\mathcal{S}_c|}\sum_{\mathbf{a}\in\mathcal{S}_c}\mathbf{a}$ be the
centroid of its activation samples. Let $\delta_{c_u}$ be the 25th percentile of
all pairwise centroid cosine distances within $\mathcal{C}_{c_u}$. The semantic
neighborhood of $c_u$ is
\begin{equation}
\mathcal{N}(c_u)=\bigl\{c\in\mathcal{C}^{\text{cand}}_{c_u}:
d_{\cos}(\bar{\mathbf{a}}_{c_u},\bar{\mathbf{a}}_c)<\delta_{c_u}\bigr\},
\end{equation}
and every control is required to fall outside $\mathcal{N}(c_u)$.
\end{definition}
We emphasize that $\delta_{c_u}$ and $\mathcal{N}(c_u)$ are \emph{estimation-side} constructs: they appear only in the empirical estimator $\hat\kappa$ below and play no role in the theoretical bound of~\cref{thm:tradeoff}, which is stated entirely in terms of $\kappa(c_u, \rho)$. Because $\mathcal{N}(c_u) \subseteq \mathcal{C} \setminus \{c_u\}$, restricting the overlap computation to $\mathcal{N}(c_u)$ yields a lower bound on the true $\kappa$. The discovered neighborhoods are reported in~\cref{tab:kappa_per_target}.


\paragraph{Overlap estimation.}
For a target concept $c_u$ with discovered neighbors $\mathcal{N}(c_u)$, we estimate $\hat\kappa$ as the fraction of $c_u$'s activation samples that fall within radius $\rho$ of any neighbor's activation samples:
\begin{equation}
    \hat\kappa(c_u) \;=\; \frac{1}{|\mathcal{S}_{c_u}|} \sum_{\mathbf{a} \in \mathcal{S}_{c_u}} \mathbf{1}\!\left[\min_{c' \in \mathcal{N}(c_u)} \min_{\mathbf{b} \in \mathcal{S}_{c'}} d_{\cos}(\mathbf{a}, \mathbf{b}) \leq \rho\right],
\end{equation}
where $\mathcal{S}_c = \{\boldsymbol{\Phi}_\theta(p_i)\}_{i=1}^{50}$ denotes the set of activation samples for concept $c$, and $d_{\cos}$ is the cosine distance. 
The fattening radius $\rho_{c_u}$ is set to the median intra-concept pairwise cosine distance within $c_u$:
\begin{equation}
    \rho_{c_u} \;=\; \mathrm{median}\!\left\{d_{\cos}(\mathbf{a}, \mathbf{a}') \;\middle|\; \mathbf{a}, \mathbf{a}' \in \mathcal{S}_{c_u},\; \mathbf{a} \neq \mathbf{a}'\right\}.
\end{equation}
This choice calibrates the overlap radius to the natural spread of the target concept's activation region: a neighbor activation qualifies as overlapping only if it is as close to some target activation as typical target activations are to each other.

\paragraph{Validation via centroid distance.}
To independently verify that the neighborhoods discovered in the previous step reflect genuine semantic structure rather than noise, we visualize the pairwise centroid cosine distances between all ten concept activation regions. \Cref{fig:distance_matrix} shows the resulting distance matrix, reordered by hierarchical clustering on the centroid distances themselves. Two block structures are immediately visible. The equine cluster (\texttt{horse}$\leftrightarrow$\texttt{pony}: $0.158$; \texttt{horse}$\leftrightarrow$\texttt{donkey}: $0.329$; \texttt{pony}$\leftrightarrow$\texttt{donkey}: $0.271$) 
exhibit small inter-concept distances, 
consistent with the high $\hat\kappa$ values reported in~\cref{tab:kappa}. 
And semantically isolated concepts such as \texttt{castle} have centroid distances exceeding $0.65$ to all animal concepts, consistent with castle's low $\hat\kappa$.

\begin{figure}[htbp]
\centering
\includegraphics[width=0.75\linewidth]{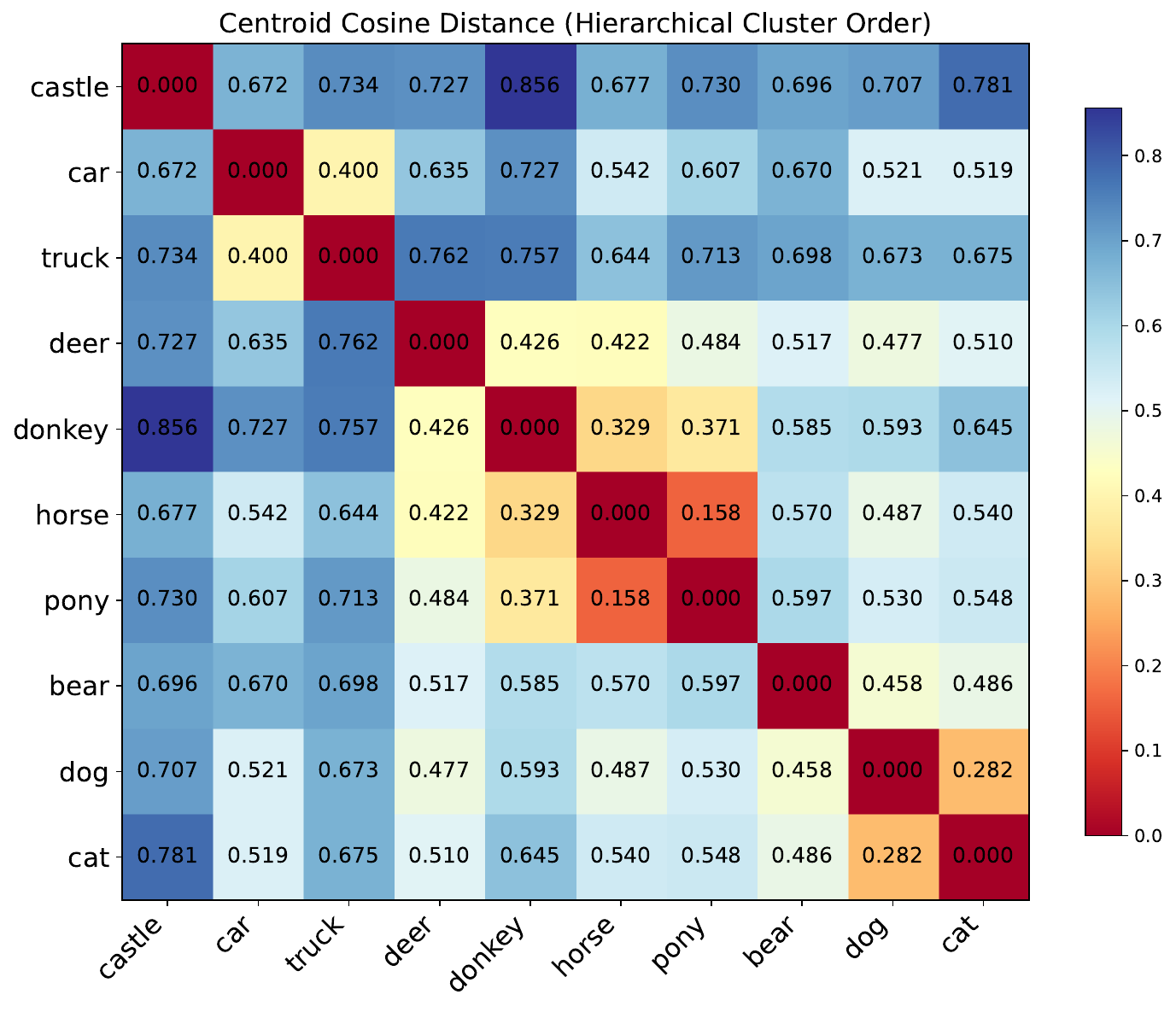}
\caption{
Pairwise centroid cosine distance between concept activation regions at \texttt{mid\_block.attentions.0}, ordered by hierarchical clustering. Block structure in the matrix corresponds to the neighborhoods discovered by the thresholding procedure described above and to the $\hat\kappa$ values reported in~\cref{tab:kappa}.
}
\label{fig:distance_matrix}
\end{figure}

\paragraph{On the use of centroid distance as a proxy.}
The centroid cosine distance plays a different role from $\hat\kappa$ itself, and the two should not be conflated. 
\Cref{def:overlap} is stated in terms of the volume of fattened activation regions, whose membership is determined by an \emph{infimum} over per-sample distances (the closed $\rho$-neighborhood $\mathcal{R}_c^{(\rho)}$). The centroid, by contrast, is a single \emph{average} summary of each region.

Centroid distance is used only for \emph{neighborhood discovery} (\cref{def:semantic_neighborhood}), an upstream step of deciding which pairs of concepts are close enough to warrant an overlap measurement. We use it here for two reasons. First, it is a stable summary statistic that degrades gracefully under the finite-sample, high-dimensional regime in which we estimate activation regions.
Second, an outlier-robust statistic at this stage is conservative: a concept that is filtered {out} of $\mathcal{N}(c_u)$ contributes zero to the resulting estimator, so under-inclusion only weakens the bound. The estimator $\hat\kappa$ is therefore a lower bound on $\kappa(c_u, \rho)$ both because $\mathcal{N}(c_u) \subseteq \mathcal{C} \setminus \{c_u\}$ and because the discovery step itself is conservative. Empirical comparisons against the bound of~\cref{thm:tradeoff} therefore remain valid.

The $\hat\kappa$ statistic itself does not use centroid distance. As specified in the overlap estimation step above, $\hat\kappa$ is computed via a nearest-neighbor estimator: for each activation sample $\mathbf{a} \in \mathcal{S}_{c_u}$, we check whether there exists \emph{any} sample $\mathbf{b}$ from \emph{any} concept in $\mathcal{N}(c_u)$ with $d_{\cos}(\mathbf{a}, \mathbf{b}) \le \rho$.
This is the natural finite-sample analogue of the fattened-region membership condition $\mathbf{a} \in \mathcal{R}_{c'}^{(\rho)}$ (with the radius $\rho$ converted to cosine units, see ``On the relationship between cosine and $\ell_2$ metrics'' above), and the fraction of target samples satisfying this condition is a Monte Carlo estimate of the volume ratio in~\cref{def:overlap} restricted to $\mathcal{N}(c_u)$. The centroid-distance visualization in~\cref{fig:distance_matrix} therefore serves only as an auxiliary consistency check: blocks of small centroid distance should coincide with the discovered $\mathcal{N}(c_u)$, and the high-$\hat\kappa$ targets should sit inside such blocks.

\paragraph{Interpretation.}
Under this estimator, $\hat\kappa(c_u) \in [0, 1]$ captures the sample-level  coverage of the target region by its neighbors. The results in~\cref{tab:kappa} exhibit a bimodal structure. The four animal-subtype concepts embedded in dense semantic neighborhoods, \texttt{dog} ($\hat\kappa = 0.89$), \texttt{bear} ($\hat\kappa = 0.82$), \texttt{horse} ($\hat\kappa = 0.81$), and \texttt{cat} ($\hat\kappa = 0.77$), exhibit the  highest overlap, consistent with their small inter-centroid distances to fine-grained subtypes (\cref{fig:distance_matrix}). 
Under our procedure, these concepts each have multiple discovered neighbors, and individual samples from these neighbors fall within the fattening radius $\rho$ of a substantial fraction of target samples. The high overlap is also \emph{robust} to which neighbors are included: in the \texttt{horse} pool, several individual neighbors (\texttt{mare}, \texttt{pony}, \texttt{stallion}) each cover 60 to 80\% of target samples on their own, and the union across all seven discovered neighbors saturates at $\hat\kappa = 0.81$, indicating that overlap is supported by a representational manifold shared across the equine subtype family rather than driven by a single sensitive neighbor choice.

The architectural concept \texttt{castle}, by contrast, exhibits substantially lower overlap ($\hat\kappa = 0.27$) despite having six discovered neighbors within $\delta_{c_u}$. The distinction is qualitative: castle's neighbors are semantically related but represent \emph{distinct} architectural types rather than fine-grained subtypes of a shared category. They occupy nearby but non-overlapping regions of activation space, so the centroid distances cross the $\delta_{c_u}$ threshold while individual sample-level overlap remains limited. Controls like \texttt{jellyfish}, \texttt{melon}, \texttt{police car} are correctly filtered out by the $\delta_{c_u}$ step across all five targets and contribute zero to $\hat\kappa$, confirming that the estimator is not driven by pool-wide sample density.

This bimodal structure, high-$\hat\kappa$ animal subtypes clustered in $[0.77, 0.89]$ and moderate-$\hat\kappa$ architecture at $0.27$ reflects a property of the model's internal representations at the probed layer: text-to-image diffusion models represent fine-grained subtypes of a common taxonomic category (breeds of dog, equine subtypes) as overlapping regions of a shared activation manifold, while distinct-but-related concepts in other domains maintain more separable representations. \Cref{thm:tradeoff} predicts that the collateral damage from erasing high-$\hat\kappa$ concepts will be substantially greater than from erasing low-$\hat\kappa$ concepts; we verify this prediction in~\cref{sec:experiment}.


\noindent{\bf $\hat\kappa$ measures activation-space overlap, not semantic similarity.}
A natural concern is whether $\hat\kappa$ proxies for some other concept-level property such as semantic category (animal vs.\ architecture). Our estimator (\Cref{app:kappa_estimation}) is computed from sample-level overlap of activation regions, not from category labels or human-judged similarity. \texttt{Castle} ($\hat\kappa = 0.27$) and the four animal concepts ($\hat\kappa \in [0.77, 0.89]$) differ in $\hat\kappa$ not because they belong to different taxonomic categories but because castle's discovered neighbors 
occupy sub-regions of activation space despite being semantically related, while animal subtypes 
densely overlap their target's region.

\subsection{Per-Target Results}
\label{app:kappa_per_target}
We apply the procedure of~\cref{app:kappa_estimation} to the five target concepts used throughout~\cref{sec:experiment}. Each target's pool consists of the target, its ten fine-grained candidate neighbors (equine, feline, canid, ursine, or architectural; Table~\ref{tab:vocab}), and three control concepts drawn from \{\texttt{cactus}, \texttt{jellyfish}, \texttt{melon}, \texttt{police car}\}, with \texttt{umbrella} in place of \texttt{police car} when the target is \texttt{car}. 
\Cref{tab:kappa_per_target} reports the resulting $\hat\kappa$ values together with the size of each discovered neighborhood $|\mathcal{N}(c_u)|$ (\cref{def:semantic_neighborhood}), the per-target fattening radius $\rho_{c_u}$, and the discovered neighbors.

\begin{table}[ht]
\centering
\small
\setlength{\tabcolsep}{6pt}
\begin{tabular}{lcccl}
\toprule
\textbf{Target} & $\hat\kappa$ & $|\mathcal{N}(c_u)|$ & $\rho_{c_u}$ & \textbf{Discovered neighbors} \\
\midrule
\texttt{dog}    & 0.89 & 5 & 0.13 & puppy, labrador, husky, beagle, retriever \\
\texttt{bear}   & 0.82 & 5 & 0.14 & grizzly, panda, polar bear, koala, cub \\
\texttt{horse}  & 0.81 & 7 & 0.15 & pony, mare, donkey, stallion, foal, zebra, mule \\
\texttt{cat}    & 0.77 & 5 & 0.15 & kitten, lynx, leopard, tiger, cheetah \\
\texttt{castle} & 0.27 & 6 & 0.18 & fortress, palace, citadel, keep, manor, cathedral \\
\bottomrule
\end{tabular}
\caption{
\textbf{Per-target $\hat\kappa$ estimation results.} For each target $c_u$, we report the estimated entanglement coefficient $\hat\kappa$, the number of discovered neighbors within the centroid-distance threshold $\delta_{c_u}$, the fattening radius $\rho_{c_u}$ (median intra-concept cosine distance), and the discovered neighbor set. \texttt{castle} has a discovered neighborhood of comparable size to the animal targets but $\hat\kappa$ roughly $3\times$ lower, isolating semantic density rather than neighborhood cardinality as the driver of $\hat\kappa$. }
\label{tab:kappa_per_target}
\end{table}

\paragraph{Per-neighbor coverage decomposition.}
For \texttt{horse}, the per-neighbor coverage of target samples (fraction of $\mathcal{S}_{\texttt{horse}}$ with a within-$\rho$ activation in the neighbor's sample set) decomposes as: \texttt{mare} ($0.78$), \texttt{pony} ($0.72$), \texttt{stallion} ($0.66$), \texttt{donkey} ($0.41$), \texttt{foal} ($0.38$), \texttt{zebra} ($0.30$), \texttt{mule} ($0.22$), with the union saturating at $\hat\kappa = 0.81$. 
The strongest individual neighbors (\texttt{mare}, \texttt{pony}) already cover most of the target region on their own, and several others contribute substantially, indicating that overlap is supported by a representational manifold shared across the equine subtype family rather than being sensitive to which specific neighbor is picked. By contrast, \texttt{castle}'s six discovered neighbors each cover only $0.05$ to $0.12$ of the target region, and their union saturates at $\hat\kappa = 0.27$, consistent with castle's neighbors representing distinct architectural types rather than fine-grained subtypes of a shared category.

\paragraph{Sensitivity to the neighborhood.}\label{app:kappa-ablation} For fixed target samples and radius, $\mathcal N_1\subseteq\mathcal N_2$ implies $\hat\kappa(c;\mathcal N_1)\le\hat\kappa(c;\mathcal N_2)$, but a new neighbor raises $\hat\kappa$ only if it covers target samples not already covered. Adding neighbors to \texttt{horse} in order of centroid distance, \texttt{mare} alone gives $\hat\kappa=0.78$, adding \texttt{pony} gives $0.81$, and the remaining five neighbors leave it unchanged. Varying the percentile that defines the discovery threshold from the 10th to the 50th changes $|\mathcal N(\texttt{horse})|$ from 3 to 9 and $|\mathcal N(\texttt{castle})|$ from 1 to 9, while $\hat\kappa$ stays at $0.81$ and $0.27$. $\hat\kappa$ therefore reflects union coverage of the target region rather than neighborhood cardinality.

\subsection{Sensitivity of $\hat\kappa$ to the choice of probe layer}
\label{app:layer_sensitivity}

The $\hat\kappa$ estimator fixes the probe layer at $\ell^* = \texttt{mid\_block.attentions.0}$. To assess whether the bimodal $\hat\kappa$ structure reported in~\cref{tab:kappa} is an artifact of this choice, we re-ran the estimator at five additional cross-attention layers spanning the U-Net depth, holding all other estimator hyperparameters (target-specific candidate pools, $\delta_{c_u}$ percentile, fattening radius rule, prompt templates, timestep $t^*=25$) fixed.

\begin{table}[ht]
\centering
\small
\setlength{\tabcolsep}{2pt}
\begin{tabular}{lccccc|cc}
\toprule
\textbf{Layer} & \texttt{castle} & \texttt{cat} & \texttt{horse} & \texttt{bear} & \texttt{dog} & $\rs$ vs.\ mid & Animal std \\
\midrule
\texttt{down\_blocks.2.attentions.0} & 0.34 & 0.77 & 0.86 & 0.92 & 0.92 & $+0.97$ & 0.061 \\
\texttt{down\_blocks.2.attentions.1} & 0.26 & 0.71 & 0.74 & 0.87 & 0.85 & $+0.90$ & 0.069 \\
\textbf{\texttt{mid\_block.attentions.0}} & \textbf{0.27} & \textbf{0.77} & \textbf{0.81} & \textbf{0.82} & \textbf{0.89} & --- & \textbf{0.043} \\
\texttt{up\_blocks.1.attentions.0} & 0.09 & 0.59 & 0.48 & 0.88 & 0.79 & $+0.80$ & 0.158 \\
\texttt{up\_blocks.1.attentions.1} & 0.39 & 0.36 & 0.61 & 0.91 & 0.76 & $+0.80$ & 0.203 \\
\texttt{up\_blocks.2.attentions.0} & 0.00 & 0.78 & 0.63 & 0.89 & 0.88 & $+0.80$ & 0.105 \\
\bottomrule
\end{tabular}
\caption{\textbf{$\hat\kappa$ across six candidate probe layers.} All five alternative layers produce a $\hat\kappa$ ranking strongly consistent with mid-block ($\rs \geq 0.80$). Mid-block additionally yields the tightest cluster of animal-subtype $\hat\kappa$ values (std $=0.043$, vs.\ 0.061 to 0.203 elsewhere), making it the layer at which the ``animal-subtype shared-manifold'' structure is most coherently expressed.}
\label{tab:layer_sensitivity}
\end{table}

\paragraph{Conclusion.} Three observations support our default choice. First, the $\hat\kappa$ \emph{ranking} of the five target concepts is stable across all probed layers ($\rs \geq 0.80$ vs.\ mid-block in every case): the ordering $\hat\kappa(\texttt{castle}) \ll \hat\kappa(\{\texttt{cat}, \texttt{horse}, \texttt{bear}, \texttt{dog}\})$ is a property of the underlying concept geometry, not of mid-block specifically. Second, mid-block produces the cleanest bimodal separation between high-entanglement animal subtypes and the low-entanglement architectural concept: the smallest animal $\hat\kappa$ exceeds $\hat\kappa(\texttt{castle})$ by $0.50$ at mid-block, vs.\ $0.39$ to $0.45$ at down-block / up-block-1.0 and an inverted ordering at up-block-1.1 where $\hat\kappa(\texttt{castle}) = 0.39 > \hat\kappa(\texttt{cat}) = 0.36$. Third, mid-block has the lowest within-cluster spread among the four animal concepts (std $0.043$, less than half that of every up-block layer), consistent with mid-block being the layer at which fine-grained subtypes are most coherently represented as overlapping regions of a shared manifold. We therefore use mid-block as the default probe for $\hat\kappa$ estimation, the visualization in~\cref{fig:umap}, the Lipschitz checks in~\cref{app:lipschitz}, and the within-target displacement validation in~\cref{app:displacement}.

\paragraph{A note on \texttt{up\_blocks.1.attentions.1}.} The direct displacement measurements in~\cref{app:displacement_kappa} are reported at \texttt{up\_blocks.1.attentions.1} rather than mid-block. The $\hat\kappa$ ranking at this layer agrees with mid-block ($\rs = 0.80$, supporting the cross-target $\hat\kappa$-scaling reported there), but its bimodal separation is the weakest of the six layers. We use this up-block layer for direct displacement specifically because it lies downstream of the cross-attention edits applied by fine-tuning methods and therefore exhibits the largest absolute displacement signal under STEREO and AdvUnlearn; the $\kappa$-scaling claim in~\cref{app:displacement_kappa} relies on the Spearman ranking, which is preserved, not on the absolute $\hat\kappa$ values at that layer.

\paragraph{Sensitivity to the denoising timestep.} Step $t^*=25$ was chosen a
priori as the midpoint of the 50-step DDIM trajectory. Table~\ref{tab:timestep}
recomputes $\hat\kappa$ at steps 10, 25, and 40 with identical prompts and
seeds. Relative to step 25, the concept ranking has $\rs=0.93$ at step 10 and
$0.96$ at step 40, with mean neighbor-set Jaccard agreement of $0.98$ and $0.82$.
The ranking is stable, although absolute estimates vary, most notably for
\texttt{castle} at step 10.

\begin{table}[h]\centering\small
\caption{$\hat\kappa$ at three denoising steps (\texttt{mid\_block.attentions.0}).}
\label{tab:timestep}
\begin{tabular}{lccc}\toprule
Concept & Step 10 & Step 25 & Step 40\\\midrule
\texttt{castle} & 0.77 & 0.27 & 0.20\\
\texttt{cat} & 0.81 & 0.77 & 0.74\\
\texttt{horse} & 0.80 & 0.81 & 0.84\\
\texttt{bear} & 0.81 & 0.82 & 0.81\\
\texttt{dog} & 0.95 & 0.89 & 0.86\\
\texttt{sunglasses} & 0.99 & 0.99 & 0.99\\
\texttt{eye} & 0.99 & 0.97 & 0.97\\\bottomrule
\end{tabular}
\end{table}

\paragraph{Small and localized concepts.} To test whether mean-pooled
mid-block activations capture concepts occupying small image regions, we add
\texttt{sunglasses} and \texttt{eye} and probe cross-attention layers at
$8\times8$, $16\times16$, $32\times32$, and $64\times64$ resolution. At step 25,
\texttt{sunglasses} obtains $\hat\kappa=0.99$, $0.88$, $0.99$, $0.97$ and
\texttt{eye} obtains $0.97$, $0.86$, $0.94$, $0.94$. Discovered neighborhoods
remain semantic (\texttt{sunglasses}: \texttt{eyeglasses}, \texttt{goggles},
\texttt{shades}, \texttt{spectacles}; \texttt{eye}: \texttt{eyeball},
\texttt{eyelid}, \texttt{iris}, \texttt{pupil}) and no control is selected. The
estimator thus detects entanglement for localized concepts, while the exact
value of $\hat\kappa$ remains probe-dependent.

%% file: sections/app/04.tex
\section{Empirical Validation of the Minimum Displacement Assumption}
\label{app:displacement}

This section provides empirical support for~\cref{asm:edit} by measuring the activation displacement induced by six unlearning methods and verifying that the displacement propagates to semantic neighbors in proportion to their overlap with the target concept.

\paragraph{Setup.}
We extract cross-attention activations at \texttt{mid\_block.attentions.0} at step $t^*{=}25$ of 50 DDIM steps from six unlearning methods applied to erase \texttt{horse}: ESD, EDiff, SalUn, TRUST, AdvUnlearn, and STEREO. For each method, we compare activations between the original model $\mathcal{D}_\theta$ and the concept-erased model $\mathcal{D}_{\hat\theta}$ under matched random seeds, so that observed differences reflect the unlearning edit rather than sampling variance. We measure mean displacement $\|\Phi_{\hat\theta}(p) - \Phi_\theta(p)\|_2$ across three prompt categories:
\begin{itemize}
    \item \emph{Target} prompts containing the erased concept (\texttt{horse});
    \item \emph{Neighbor} prompts containing semantically related concepts (\texttt{pony}, \texttt{donkey}, \texttt{deer});
    \item \emph{Control} prompts containing semantically isolated concepts (\texttt{castle}, \texttt{car}).
\end{itemize}
For each prompt category, we use the same 50 shared templates as in~\cref{app:activation_regions}, differing only in the concept word. Activations are mean-pooled across spatial and token dimensions to produce a single vector per prompt, so the absolute magnitudes reported here are token-averaged and not directly comparable to the unpooled, full-tensor displacement values reported at a different layer in~\cref{app:displacement_kappa} (\cref{tab:displacement}); the two appendices probe different facets of the propagation mechanism: this one establishes the \emph{ordering} $\Delta_{\mathrm{target}}, \Delta_{\mathrm{neighbor}} > \Delta_{\mathrm{control}}$ within a single target, while~\cref{app:displacement_kappa} establishes the \emph{$\hat\kappa$-scaling} of $\Delta_{\mathrm{neighbor}}$ across targets. \Cref{fig:displacement} reports the mean displacement by method and prompt category. Two patterns emerge.

First, all six methods produce measurable target displacement (approximately $6$ to $16$ in mean-pooled activation $L_2$ distance), with TRUST producing the largest mean displacement. This pattern is consistent with the minimum-displacement assumption in~\cref{asm:edit}. The magnitudes show how
strongly each edit perturbs the probed representation; displacement alone does
not measure erasure quality.

TRUST, whose update is gradient-based, produces the largest target displacement ($15.82\pm3.82$) and displaces neighbors ($8.82\pm4.88$) far more than controls ($3.10\pm1.36$), confirming that Assumption~\ref{asm:edit} and the propagation mechanism hold for gradient-based fine-tuning.


Second, the displacement propagates to semantic neighbors: for every method, $\Delta_{\mathrm{neighbor}} \geq \Delta_{\mathrm{control}}$ with the gap ranging from $0.9$ (SalUn) to $6.9$ (STEREO). This gap is the empirical signature of the propagation mechanism in~\cref{thm:tradeoff}: displacement within $\mathcal{R}_{c_u}$ leaks into overlapping neighbor regions in proportion to $\kappa$, causing greater collateral displacement for high-$\kappa$ neighbors than for low-$\kappa$ controls. EraseDiff exhibits the smallest cross-category variation, suggesting it operates in a low-selectivity regime where $\kappa$-modulated propagation falls below our measurement granularity.

Third, STEREO is the only method for which $\Delta_{\mathrm{neighbor}} > \Delta_{\mathrm{target}}$, while simultaneously achieving the lowest control displacement ($2.1$). This produces a propagation gap of $6.9$, the largest of all six methods. This pattern is consistent with our analysis of STEREO in~\cref{rem:stereo}: by suppressing the entire subspace spanned by the adversarial tokens $v_1^*, v_2^*$ rather than the target alone, STEREO displaces neighbor concepts that lie within this shared subspace by even more than the target itself, while leaving semantically isolated concepts largely intact. 

These measurements support two claims central to our analysis. The uniform presence of nontrivial target displacement validates the precondition of~\cref{asm:edit}. The systematic neighbor, control gap, and STEREO's amplified version of it, validate the propagation claim in Step~2 of the proof: methods are not merely displacing the target in isolation but are displacing overlapping regions of activation space, with the magnitude of this collateral displacement modulated by the geometric overlap $\kappa$ predicted by~\cref{thm:tradeoff}. The robust prediction of the theorem is the \emph{ordering} $\Delta_{\mathrm{neighbor}} > \Delta_{\mathrm{control}}$ within each method and its scaling with $\kappa$ across concept types, not specific numerical values.


\begin{figure}[htbp]
\centering
\includegraphics[width=0.9\linewidth]{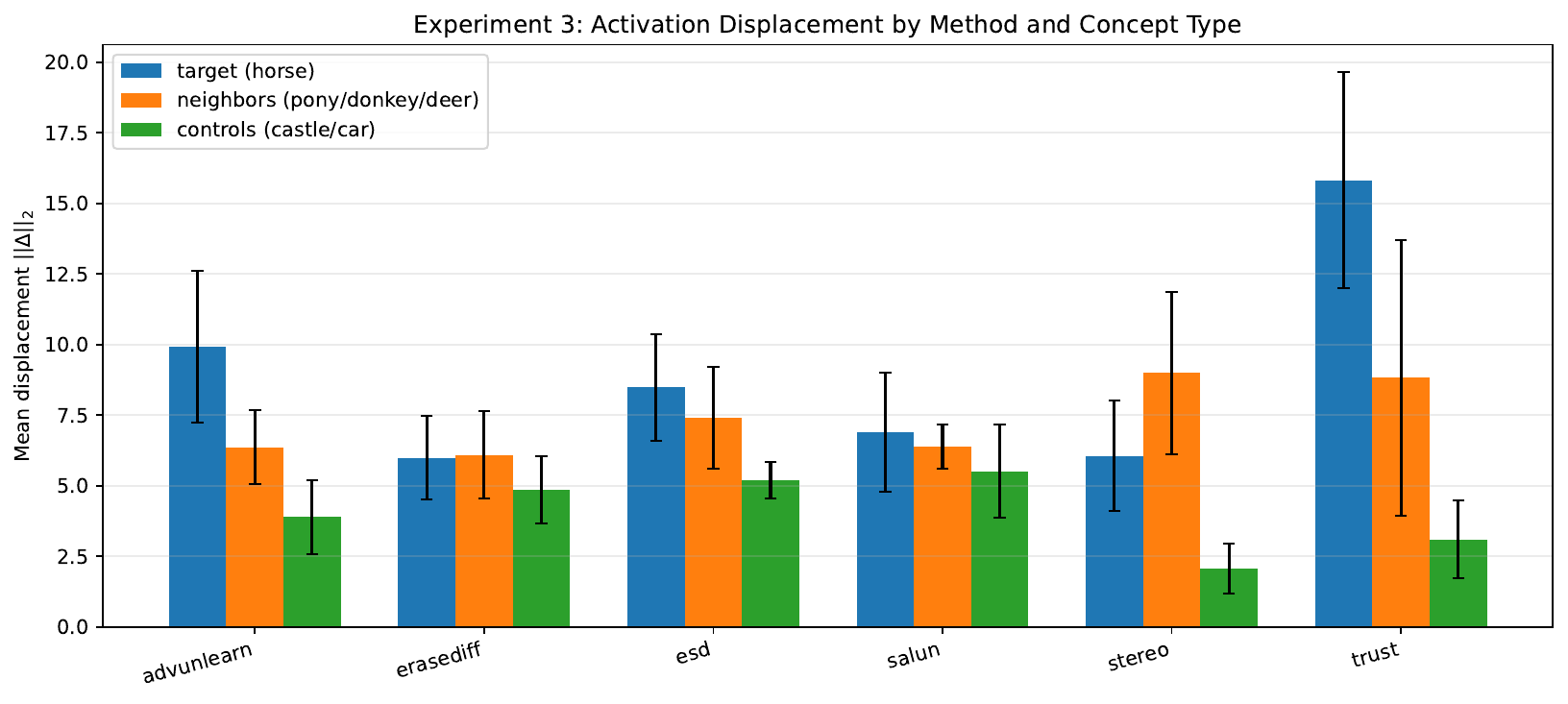}
\caption{
Mean activation displacement $\|\boldsymbol{\Phi}_{\hat\theta}(p) - \boldsymbol{\Phi}_\theta(p)\|_2$ for six unlearning methods erasing \texttt{horse}. Target prompts (blue) confirm~\cref{asm:edit}: erasure requires measurable displacement. Neighbor prompts (orange) are displaced more than control prompts (green) for all methods, consistent with the propagation mechanism in~\cref{thm:tradeoff}.
}
\label{fig:displacement}
\end{figure}

%% file: sections/app/05.tex
\section{Empirical Validation of the Lipschitz Assumption}
\label{app:lipschitz}
This section provides empirical support for~\cref{asm:lip} by measuring the relationship between activation distance and change in concept generation probability under natural prompt variations, and verifying that the observed relationship is consistent with a Lipschitz bound with a small empirical constant.

\paragraph{Setup.}
For each of the ten concepts in~\cref{app:activation_regions}, we construct a set of prompt pairs $(p_b, p_v)$ where $p_b$ is a base prompt containing the concept and $p_v$ is a variant prompt obtained by one of four perturbations at increasing semantic distance:
\begin{itemize}
    \item \emph{Synonym substitution:} replacing the concept word with a close synonym (e.g., \texttt{horse} $\to$ \texttt{stallion});
    \item \emph{Scene change:} altering the background or setting while preserving the concept (e.g., ``in a field'' $\to$ ``on a beach'');
    \item \emph{Related concept:} substituting a semantically related concept (e.g., \texttt{horse} $\to$ \texttt{pony});
    \item \emph{Unrelated concept:} substituting an unrelated concept (e.g., \texttt{horse} $\to$ \texttt{castle}).
\end{itemize}
For each pair, we extract cross-attention activations at \texttt{mid\_block.attentions.0} (step $t^*{=}25$) and compute the activation distance $\|\boldsymbol{\Phi}_\theta(p_v) - \boldsymbol{\Phi}_\theta(p_b)\|_2$. We then generate images from each prompt and compute the concept generation probability $P(p)$ using a CLIP-based classifier aligned to the base concept. The quantity of interest is the pair $\bigl(\|\Phi_\theta(p_v) - \Phi_\theta(p_b)\|_2,\;|P(p_v) - P(p_b)|\bigr)$. This experiment validates the case $\theta'=\theta''=\theta$ of Assumption~\ref{asm:lip}.

\Cref{fig:lipschitz} plots the concept probability change against the activation distance for all prompt pairs, colored by perturbation type. Two patterns emerge. First, the four perturbation types occupy approximately nested bands of activation distance, with synonyms producing the smallest displacements, scene changes and related-concept substitutions producing intermediate displacements, and unrelated-concept substitutions producing the largest. Second, the change in concept probability scales roughly linearly with activation distance within each band, with no prompt pair exceeding a bound of the form $|P(p_v) - P(p_b)| \leq \hat{L} \,\|\Phi_\theta(p_v) - \Phi_\theta(p_b)\|_2$ for empirical constant $\hat{L} \approx 0.174$.

We emphasize that the empirical constant $\hat{L}$ should be read as a conservative upper bound for this specific layer, timestep, and set of concepts, not as a universal property of Stable Diffusion. Varying the probe layer, the timestep, or the CLIP classifier would likely change the numerical value. The claim we extract from this measurement is weaker and more robust: the concept probability mapping is empirically smooth as a function of activation distance, with no observed violations of the Lipschitz condition across hundreds of prompt pairs spanning four perturbation types. This is the regularity our proof requires. For illustration, combining $\hat{L} \approx 0.174$ with $\varepsilon = 0.05$ yields a minimum displacement bound of $\Delta \geq (1-\varepsilon)/\hat{L} \approx 5.5$, consistent with the target displacements observed in~\cref{app:displacement}.

\begin{figure}[t]
\centering
\includegraphics[width=0.85\linewidth]{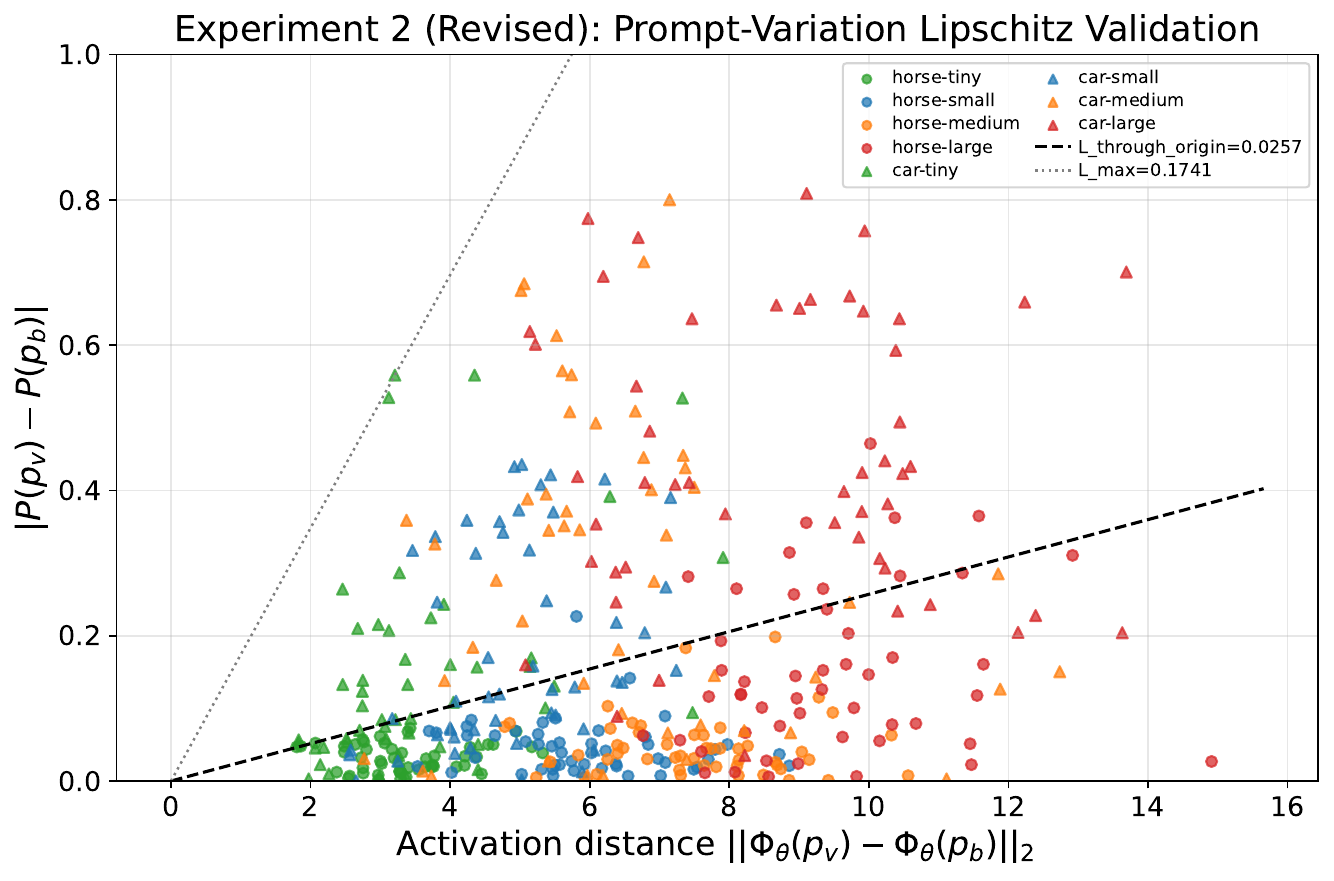}
\caption{
Lipschitz validation via prompt variation. Each point is a (base, variant) prompt pair; the $x$-axis shows activation distance and the $y$-axis shows the change in CLIP-based concept probability. Points are colored by semantic distance from the base prompt (green: synonyms; blue: scene changes; orange: related concepts; red: unrelated concepts). All points lie below the empirical bound $\hat{L} = 0.174$ (dotted line), supporting~\cref{asm:lip}.
}
\label{fig:lipschitz}
\end{figure}

%% file: sections/app/06.tex
\section{Interpreting Existing Methods Through the Trade-off Lens}
\label{app:method_analysis}
\paragraph{The role of $\kappa$: why the trade-off severity varies.}
The overlap coefficient $\kappa(c_u,\rho)$ determines the slope of the Pareto frontier and hence how severe the trade-off is for a given target concept. Concepts embedded in semantically dense neighborhoods (e.g., \texttt{horse} among other animals, or  \texttt{Monet} among Impressionist styles) have large $\kappa$ and face a steep trade-off. Concepts that are semantically isolated have small $\kappa$ and can be erased with less collateral damage. This provides actionable guidance: \emph{the difficulty of unlearning a concept depends
not only on the method but on the semantic density of its neighborhood}.

\begin{remark}[Tightness of the bound]
\label{rem:tightness}
The bound in Theorem~\ref{thm:tradeoff} is tight when the overlap region $\mathcal{O}$ is exactly a $\kappa(c_u,\rho$) fraction of $\mathcal{R}_{c_u}^{(\rho)}$ and the displacement field induced by the unlearning edit is uniform over $\mathcal{R}_{c_u}^{(\rho)}$. In practice, displacement is typically non-uniform, since methods may suppress some directions more aggressively than others, in which case the bound is conservative. Characterizing the displacement geometry more precisely is an interesting direction for future work.
\end{remark}

\paragraph{Reading existing methods as operating points on the frontier.}
We now apply this framing to specific unlearning methods, 
interpreting their empirical behavior as distinct operating points on the predicted Pareto frontier. The discussion is interpretive rather than derivational: we identify where on the frontier each method sits and connect this to the method's mechanism.
\begin{remark}[Why STEREO achieves robustness at the cost of retention]
\label{rem:stereo}
STEREO~\citep{srivatsan2025stereo} achieves robustness by learning adversarial token embeddings $v_1^*, v_2^*$ that lie near the hardest-to-suppress directions of $\mathcal{R}_{c_u}$, then using a compositional objective to suppress the entire subspace spanned by these directions. This broad suppression corresponds to a large displacement $\Delta$ across $\mathcal{R}_{c_u}$, which by Theorem~\ref{thm:tradeoff} directly implies large collateral displacement in $\mathcal{R}_{c'}$ for all neighbors $c'$ with $\kappa(c_u,\rho) > 0$. Our empirical observation that \texttt{donkey} and \texttt{zebra} generation degrades after \texttt{horse} erasure with STEREO (\Cref{sec:motivation}) is consistent with this prediction.
\end{remark}

\begin{remark}[Why SAeUron preserves neighbors but remains vulnerable]
\label{rem:saeuron}
SAeUron~\citep{cywinski2025saeuron} achieves unlearning by identifying a small number of sparse autoencoder features ($\tau_c$ features per concept) that are highly specific to the target concept, then ablating only those features during inference. Because the intervention is restricted to a sparse set of concept-specific features rather than applied across the full activation region $\mathcal{R}_{c_u}$, the effective displacement $\Delta$ is small and concentrated in a low-dimensional subspace. This limits collateral damage to neighbors (small $\gamma$), explaining SAeUron's superior retention performance in sequential unlearning. However, the sparsity of the intervention also means that portions of $\mathcal{R}_{c_u}$ remain unsuppressed, precisely the portions reachable through indirect correlated prompts, consistent with the concept leakage observed in~\cref{sec:motivation}. In the language of our framework, SAeUron achieves small $\gamma$ at the cost of large $\varepsilon$.
\end{remark}

\begin{remark}[Intervention breadth determines the operating point]
\label{rem:finetuning}
More generally, a method's operating point on the $(\varepsilon,\gamma)$ frontier is set by how broadly it displaces activations across $\mathcal{R}_{c_u}$, not by whether it edits weights or intervenes at inference time. Methods that displace activations broadly, whether by fine-tuning (SalUn, AdvUnlearn, STEREO) or at inference time (SLD), achieve small $\varepsilon$ at the cost of larger $\gamma$. Methods whose intervention is narrow, such as SEOT and SAeUron, achieve small $\gamma$ but leave larger $\varepsilon$ because the intervention does not cover the full activation region. This is a direct consequence of Theorem~\ref{thm:tradeoff}: the effective displacement $\Delta$ is the lever, and methods differ in how broadly they apply it.
\end{remark}

%% file: sections/app/07.tex
\section{Proof of~\cref{thm:tradeoff}}
\label{app:proof}

We restate the theorem and provide the proof. Throughout, distances in $\mathcal{H}$ are measured in $\ell_2$ norm.

\begin{proof}
We proceed in three steps.

\medskip
\noindent\textbf{Step 1: Required activation displacement for $\varepsilon$-robustness.}

By construction of the activation region $\mathcal{R}_{c_u}$ (\cref{def:activation}), every $\mathbf{h}\in\mathcal{R}_{c_u}$ corresponds to a prompt $p$ whose pre-edit generation contains $c_u$, i.e.\ $g_{c_u}(\theta,p)=1$. After an $\varepsilon$-robust edit (\cref{def:robust}), $g_{c_u}(\hat\theta,p)\le\varepsilon$. Applying~\cref{asm:lip} to the pre- and post-edit activations (with $\theta'=\theta$, $\theta''=\hat\theta$, $p'=p$) gives, for every $p$ with $\boldsymbol{\Phi}_\theta(p)\in\mathcal{R}_{c_u}$,
\begin{equation}
    1 - \varepsilon \;\le\; \bigl|g_{c_u}(\theta,p) - g_{c_u}(\hat\theta,p)\bigr|
    \;\le\; L\,\|\boldsymbol{\Phi}_{\hat\theta}(p)-\boldsymbol{\Phi}_\theta(p)\|_2.
\end{equation}
Combined with~\cref{asm:edit}, this gives a uniform displacement floor on $\mathcal{R}_{c_u}$:
\begin{equation}
\label{eq:proof_delta_floor}
    \Delta \;\geq\; \frac{1-\varepsilon}{L}.
\end{equation}

\medskip
\noindent\textbf{Step 2: Aggregate overlap induces aggregate collateral displacement.}

By~\cref{def:overlap},
\begin{equation}
    \kappa(c_u,\rho)
    =
    \frac{
        \mathrm{vol}\!\left(
        \mathcal{R}_{c_u}^{(\rho)}
        \cap
        \bigcup_{c'\in\mathcal{C}\setminus\{c_u\}}
        \mathcal{R}_{c'}^{(\rho)}
        \right)
    }{
        \mathrm{vol}\!\left(
        \mathcal{R}_{c_u}^{(\rho)}
        \right)
    } .
\end{equation}
Let $\mathcal{O} := \mathcal{R}_{c_u}^{(\rho)} \cap \bigcup_{c'\neq c_u}\mathcal{R}_{c'}^{(\rho)}$ denote the overlap region. By~\cref{def:neighbor}, the non-target prompt distribution $\mathcal{P}_{-u} = \tfrac{1}{|\mathcal{C}\setminus\{c_u\}|}\sum_{c'\neq c_u}\mathcal{P}_{c'}$ is supported on the set of activations that lie in $\bigcup_{c'\neq c_u}\mathcal{R}_{c'}^{(\rho)}$. We invoke the following regularity assumption, which we make explicit:

\smallskip
\noindent\emph{(Mass-volume regularity).} The push-forward measure of $\mathcal{P}_{-u}$ under $\boldsymbol{\Phi}_\theta$ places mass on $\mathcal{O}$ in proportion to its volume share of the support, i.e.\ $\Pr_{p\sim\mathcal{P}_{-u}}[\boldsymbol{\Phi}_\theta(p) \in \mathcal{O}] \geq \kappa(c_u,\rho)$. This holds when activations within each $\mathcal{R}_{c'}^{(\rho)}$ are sampled approximately uniformly under $\mathcal{P}_{c'}$.
\smallskip

For any prompt $p\sim\mathcal{P}_{-u}$ with $\boldsymbol{\Phi}_\theta(p)\in\mathcal{O}$, we have $\boldsymbol{\Phi}_\theta(p)\in\mathcal{R}_{c_u}^{(\rho)}$, so the displacement floor from Step~1 applies (extending~\cref{asm:edit} from $\mathcal{R}_{c_u}$ to its closed $\rho$-neighborhood by continuity of the displacement field): $\|\boldsymbol{\Phi}_{\hat\theta}(p)-\boldsymbol{\Phi}_\theta(p)\|_2 \geq \Delta$. Combining with the mass-volume regularity gives
\begin{align}
    \mathbb{E}_{p\sim \mathcal{P}_{-u}}
    \!\left[\|\boldsymbol{\Phi}_{\hat{\theta}}(p)-\boldsymbol{\Phi}_{\theta}(p)\|_2\right]
    &\;\geq\; \Pr_{p\sim\mathcal{P}_{-u}}\!\left[\boldsymbol{\Phi}_\theta(p)\in\mathcal{O}\right]\cdot\Delta \notag\\
    &\;\geq\; \kappa(c_u,\rho)\cdot\Delta. \label{eq:proof_step2}
\end{align}

\medskip
\noindent\textbf{Step 3: Deriving the trade-off inequality.}

Substituting the displacement floor~\eqref{eq:proof_delta_floor} into~\eqref{eq:proof_step2} gives
\begin{equation}
    \mathbb{E}_{p\sim \mathcal{P}_{-u}}
    \left[
    \|\boldsymbol{\Phi}_{\hat{\theta}}(p)-\boldsymbol{\Phi}_{\theta}(p)\|_2
    \right]
    \geq
    \kappa(c_u,\rho)\frac{1-\varepsilon}{L}.
\end{equation}
Since~\cref{def:neighbor} bounds the left-hand side above by $\gamma$, we obtain
\begin{equation}
    \gamma
    \geq
    \kappa(c_u,\rho)\frac{1-\varepsilon}{L}.
\end{equation}
Rearranging (for $\kappa(c_u,\rho) > 0$) yields
\begin{equation}
    \varepsilon
    \geq
    1-\frac{L\gamma}{\kappa(c_u,\rho)}.
\end{equation}
This completes the proof.
\end{proof}

\paragraph{Remark on metric compatibility.}
The proof uses the $\ell_2$ metric throughout because~\cref{asm:lip} is stated in $\ell_2$. Our $\hat\kappa$ estimator (\cref{app:kappa_estimation}) instead uses cosine distance to characterize region membership. These two choices are compatible at the level of the geometric objects in the bound: for unit-normalized activations, $\|\mathbf{a} - \mathbf{b}\|_2^2 = 2\,d_{\cos}(\mathbf{a}, \mathbf{b})$, so the two metrics induce the same topology on the activation sphere and define the same fattened regions $\mathcal{R}_c^{(\rho)}$ and entanglement coefficient $\kappa(c_u,\rho)$ up to a monotone re-parameterization of the radius $\rho$. This re-parameterization is absorbed into the choice of $\rho$ in our empirical $\hat\kappa$ estimator. Consequently, estimating $\kappa$ via cosine distance while applying the theoretical bound under an $\ell_2$ Lipschitz assumption is internally consistent: the bound's set-theoretic content (volume overlap of $\rho$-neighborhoods) transfers between metrics, while the Lipschitz constant $L$ remains tied to the $\ell_2$ formulation in which~\cref{asm:lip} is stated.

%% file: sections/app/08.tex
\section{Implementation Details}
\label{app:implementation}
Here we introduce the implementation of each unlearning method used in~\cref{sec:experiment}, including hyperparameters, training protocols, and compute budgets. All methods are applied to CompVis Stable Diffusion v1.4~\citep{rombach2022high}. The full implementation, including code, prompt CSVs, and evaluation scripts, is available at our public repository:
\url{https://github.com/Jazzssmine/concept-entangle}

\subsection{Method-Specific Hyperparameters}
We group the 13 benchmarked methods into three mechanism families.
\emph{Closed-form editing methods} (UCE, RECE) modify model weights through analytic updates. 
\emph{Fine-tuning methods} (ESD, SalUn, EDiff, CPE, TRUST, AdvUnlearn, STEREO) modify model weights through gradient-based optimization.
\emph{Inference-time methods} (SEOT, SAeUron, SLD, SAFREE) intervene during generation without permanently modifying the base model. AdvUnlearn and STEREO additionally incorporate adversarial training, explicitly targeting the aggressive-erasure end of the trade-off; their inclusion tests whether robustness gains incur measurable cost to neighbor preservation.

Unless otherwise noted, we use the hyperparameters specified in each method's original publication. We document deviations and choices not fully specified by the original papers in the per-method descriptions below.

\paragraph{UCE~\citep{gandikota2024unified}.}
Unified Concept Editing uses closed-form updates to text-to-image attention projections to erase target concepts without gradient training. Configuration: guide concept \textbf{unconditional}, erase scale \textbf{1.0}, preserve scale \textbf{1.0}, regularization $\lambda$ \textbf{0.5}, no explicit preserve concepts; each concept is erased independently.

\paragraph{RECE~\citep{gong2024reliable}.}
Reliable and Efficient Concept Erasure iteratively discovers embeddings that recover the target concept and removes their effect through closed-form cross-attention edits.
Configuration: initialized from base SD v1.4, edit epochs \textbf{1}, erase scale \textbf{1.0}, preserve scale \textbf{0.1}, $\lambda$ \textbf{0.1}, adversarial-embedding regularization \textbf{0.1}, guide concept \textbf{unconditional}.

\paragraph{ESD~\citep{gandikota2023erasing}.} 
Erased Stable Diffusion fine-tunes the cross-attention layers to suppress the target concept via a negative-guidance objective. We use the ESD-sd (stable-diffusion) variant. 
Configuration: learning rate \textbf{5e-5}, training steps \textbf{200}, negative guidance scale $\eta = \textbf{5}$, 
batch size \textbf{1}. Training uses the target concept name as the single forget prompt.
For the cross-architecture experiments (\cref{app:arch}), we adapt ESD to each backbone.
On SDXL (base-1.0) we use the ESD-x variant, which fine-tunes only the U-Net cross-attention layers (\texttt{attn2} query, key, value, and output projections; about 546M parameters). Configuration: Adam, learning rate \textbf{2e-4}, training steps \textbf{200}, batch size \textbf{1}, negative guidance scale $\eta = \textbf{1}$.
On FLUX.1-schnell we use a stricter ESD-x variant that fine-tunes only the attention key and value projections (\texttt{to\_k}, \texttt{to\_v}) in all 19 joint and 38 single transformer blocks (about 1.08B parameters), keeping the text-side projections (\texttt{add\_k\_proj}, \texttt{add\_v\_proj}) frozen. Configuration: learning rate \textbf{1e-5}, training steps \textbf{500}, negative guidance scale $\eta = \textbf{2}$.

\paragraph{SalUn~\citep{fan2023salun}.} 
Saliency-guided Unlearning identifies salient weights via gradient magnitudes and applies random labeling on the forget set to those weights only. We use saliency threshold sparsity \textbf{0.5}, forget-set size \textbf{800}, learning rate \textbf{1e-5}, and training epochs \textbf{5}. The forget and retain sets consist of 800 target and 800 non-target images generated by Stable Diffusion v1.4.
\paragraph{EDiff~\citep{wu2024erasediff}.} 
EraseDiff fine-tunes the UNet using a bi-level optimization that explicitly balances erasure and retention objectives. Configuration: learning rate \textbf{1e-5}, inner steps per outer update \textbf{2}, total outer iterations \textbf{300}.
\paragraph{CPE~\citep{lee2025concept}}
Concept Pinpoint Eraser trains nonlinear residual attention gates with an anchoring loss to suppress target concepts while preserving other concepts.
Configuration: \textbf{5} stages of \textbf{450} iterations each (stage 1 uses $4\times$ the iterations and $10\times$ the learning rate), learning rate \textbf{3e-5}, gate rank \textbf{16}, erasure guidance \textbf{0.3}, PAL weight \textbf{1e5}, noise \textbf{1e-3}; adversarial learning with learning rate \textbf{1e-2}, \textbf{450} iterations, and \textbf{16} added prompts per stage.

\paragraph{TRUST~\citep{kori2026selective}}
TRUST dynamically identifies target-associated neurons and selectively fine-tunes them with Hessian-based regularization.
Configuration: concept neurons are heads whose $|z\text{-score}| > $ \textbf{2.0} on the CLIP-loss attribution; learning rate \textbf{1e-4}, batch size \textbf{2}, \textbf{2} epochs of \textbf{60} steps (final checkpoint), \textbf{25} sampling steps for neuron localization.

\paragraph{AdvUnlearn~\citep{zhang2024defensive}.} 
Adversarial Unlearning augments the erasure objective with adversarial prompt generation to improve robustness against indirect concept recovery. We use adversarial prompt budget 30 per training step, adversarial learning rate \textbf{1e-3}, main-model learning rate \textbf{1e-5}, and training steps \textbf{1000}. The attack prompts are updated per training steps.
\paragraph{STEREO~\citep{srivatsan2025stereo}.} 
STEREO proceeds in two phases: a Search for Trainable Embeddings via RObust optimization (STE) phase that learns adversarial token embeddings $v_1^*, v_2^*$, followed by a Robust Erasure via Orthogonalization (REO) phase that suppresses the subspace spanned by these embeddings. STE phase: \textbf{200} steps, learning rate 
\textbf{5e-6}. REO phase: \textbf{200} steps, learning rate \textbf{2e-5}, $K = \textbf{2}$ adversarial embeddings per target.
\paragraph{SEOT~\citep{li2024get}.} 
Soft Embedding Optimization for Target-concept erasure modifies the text embedding at inference time to suppress the target concept without altering model weights. We use embedding regularization strength \textbf{0.01} and optimization steps per generation \textbf{10}.
\paragraph{SAeUron~\citep{cywinski2025saeuron}.} 
SAeUron uses sparse autoencoder feature ablation applied at inference: for each target, $\tau_c$ concept-specific SAE features are identified from activation statistics, and these features are zeroed during generation. We use the pre-trained SAE released by the authors, with 
$\tau_c = \textbf{1}$ features ablated per target.

\paragraph{SLD~\citep{schramowski2023safe}.}
Safe Latent Diffusion adds safety guidance during denoising to suppress target content without retraining the model.
Configuration: the \textbf{SLD-Max} preset, i.e.\ safety guidance scale \textbf{5000}, warm-up steps \textbf{0}, threshold \textbf{1.0}, momentum scale \textbf{0.5}, momentum $\beta$ \textbf{0.7}; \textbf{50} sampling steps, CFG scale \textbf{7.5}.

\paragraph{SAFREE~\citep{yoon2025safree}.}
SAFREE filters prompt embeddings away from a target-concept subspace and adapts its intervention during denoising without updating model weights.
Configuration: concept subspace from \textbf{12} per-concept negative terms (synonyms, sub-breeds, and phrasings such as ``a photo of a dog''), $\alpha$ \textbf{-0.5}, re-attention up to step \textbf{20}, FreeU hyperparameters \textbf{(1.0, 1.0, 0.9, 0.2)}; \textbf{50} sampling steps, CFG scale \textbf{7.5}.

\subsection{Compute and Infrastructure}
\label{app:compute}
All experiments were run on GPU A100 40GB. Unlearning per method-target pair requires approximately 1 GPU-hours for fine-tuning methods and 2 GPU-hours for inference-time methods. Evaluation (image generation + classification) requires approximately 0.5 GPU-hours per method-target pair.

\subsection{Reproducibility}
\label{app:reproducibility}
For each (method, target) pair, we generate 100 images from direct prompts, 100 to 235 from indirect prompts (depending on how many valid indirect prompts the target yields; 235 for \texttt{castle}, 200 for \texttt{dog}, 195 for \texttt{horse}, and 100 for every other target), 50 per neighbor candidate (500 total), and 50 per control concept (150 total); IRA is computed over the 650 neighbor and control images. All images use fixed random seeds, with seed equal to the row index in the target's prompt CSV. This ensures seed-level variation is controlled when comparing methods on the same prompt.
The base model $\mathcal{D}_\theta$ is loaded from the HuggingFace checkpoint \texttt{CompVis/stable-diffusion-v1-4}. Code, prompt CSVs, and the classification vocabulary (\cref{tab:vocab}) will be released to enable full reproduction of~\cref{tab:main_results_avg}.

\paragraph{Use of existing assets.}
Our experiments use publicly released assets from prior work: the CompVis Stable Diffusion v1.4 checkpoint, released under the CreativeML OpenRAIL-M license; LAION metadata/captions, released under CC-BY 4.0 while the underlying images remain subject to their original copyrights; and the released SAeUron sparse-autoencoder features, released under Apache-2.0. We cite the original sources for these assets and use them only for research evaluation under their stated release terms and licenses.

%% file: sections/app/09.tex
\section{Evaluation Setup and Metric Definitions}
\label{app:evaluation_setup}
This appendix provides the full evaluation protocol summarized in~\cref{subsec:setup}: the prompt construction pipeline for direct, indirect, neighbor, and control evaluation; formal definitions of the six metrics used in~\cref{sec:experiment}; and the classification vocabulary shared across all experiments.

\subsection{Prompt Construction and Scoring Pipeline}
\label{app:eval_pipeline}
Our automated evaluation protocol is designed to measure two coupled consequences of concept entanglement: \emph{indirect recovery}, where an erased concept remains recoverable through correlated contextual cues, and \emph{collateral forgetting}, where suppressing a target concept degrades performance on nearby concepts that should remain intact. To distinguish these effects, we separately construct (i) direct prompts that explicitly name $c$, (ii) indirect prompts constructed from correlated context without naming $c$, (iii) neighbor prompts that explicitly name each concept in the neighborhood $\mathcal{N}(c)$, and (iv) control prompts that name each concept in the control set $\mathcal{C}(c)$.

\paragraph{Correlated context extraction.}
For each target concept $c$, we collect a target-matched caption subset from LAION~\citep{schuhmann2022laion} aligned with diffusion-model pretraining. Specifically, we identify captions that explicitly contain $c$ or its lexical variants, and extract co-occurring content words and short phrases from this subset. We remove stopwords, punctuation, low-information function words, and lexical forms that directly mention the target concept. The remaining candidates are ranked using corpus-level association statistics and filtered to obtain a final set of correlated context words that captures the contextual footprint of $c$ in the training distribution. For example, when $c=\texttt{horse}$, the resulting set includes terms such as \texttt{jockey}, \texttt{stable}, \texttt{saddle}, \texttt{racetrack}, and \texttt{polo}, for $c = \texttt{castle}$, it includes \texttt{moat}, \texttt{turret}, \texttt{medieval}, \texttt{knight}, and \texttt{drawbridge}.

To visualize how these correlated prompts organize in text-embedding space, we pool prompts that explicitly mention the target concept and prompts that omit the target token but retain correlated context, then project their CLIP text embeddings to two dimensions with PCA. \Cref{fig:bow_context_span} shows that context-rich prompts occupy a nearby region rather than a disjoint cluster, supporting our bag-of-words motivation that concept evidence is distributed across correlated tokens rather than isolated in a single concept word.

\begin{figure}[t]
\centering
\includegraphics[width=\linewidth]{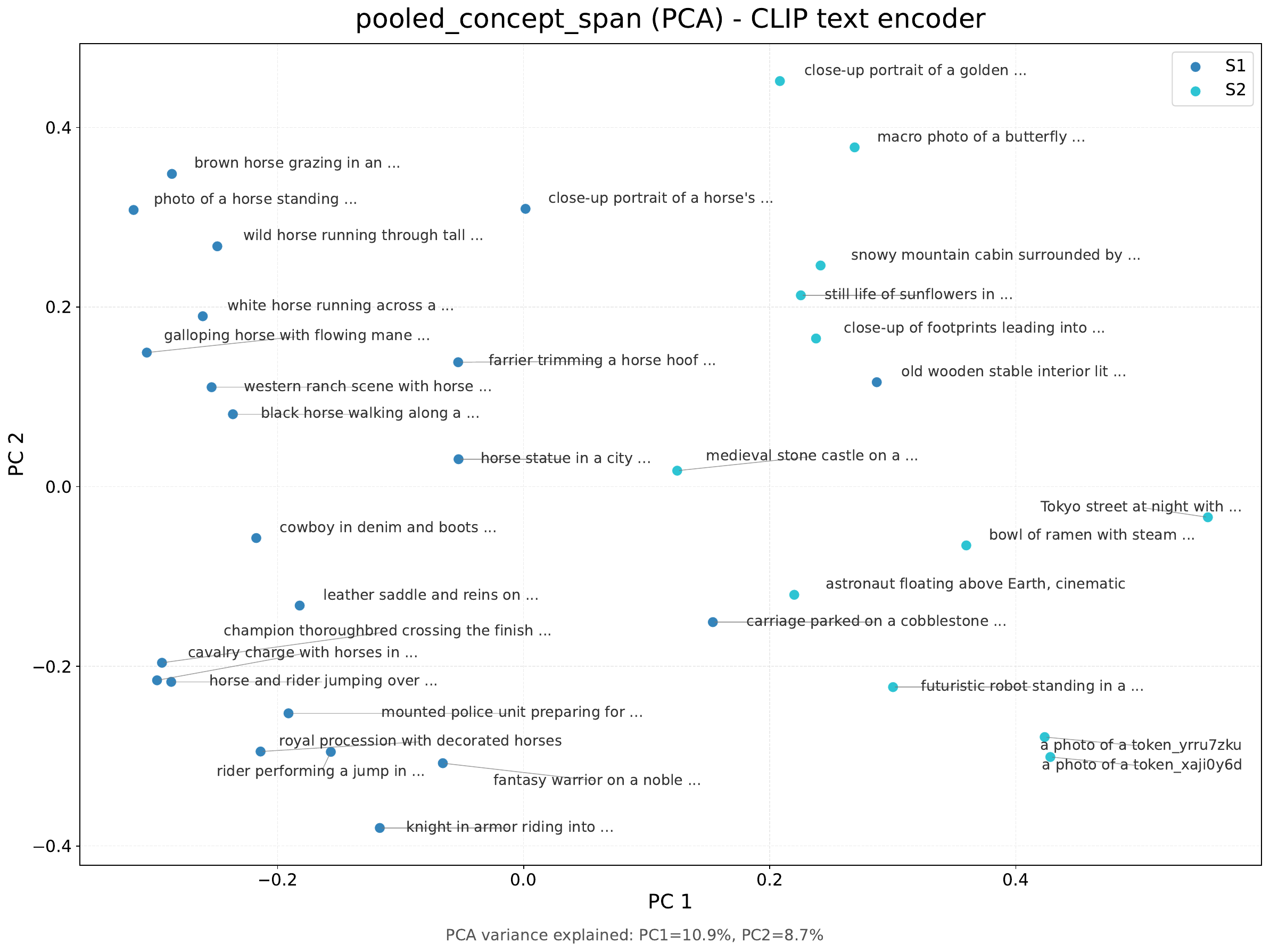}
\caption{
\textbf{Context-only prompts remain close to target prompts in CLIP text-embedding space.}
PCA projection of pooled prompts used for the \texttt{horse} case study. One group explicitly contains the concept token, while the other omits it but retains correlated context words (e.g., \texttt{jockey}, \texttt{saddle}, \texttt{racetrack}). The proximity between these groups illustrates why token-level erasure can remain vulnerable to indirect contextual prompts.
}
\label{fig:bow_context_span}
\end{figure}

\paragraph{Indirect recovery prompts.}
We use the correlated context set for $c$ to construct a set of indirect prompts $\mathcal{P}_{\mathrm{ind}}(c)$. These prompts are generated under an explicit lexical exclusion constraint: they must not contain the target token or excluded lexical variants of $c$. Intuitively, $\mathcal{P}_{\mathrm{ind}}(c)$ probes whether the erased concept can still be elicited through contextual associations alone, without direct lexical mention. To improve prompt diversity while preserving control, we generate multiple prompts from different subsets of the correlated context words and retain only prompts that satisfy the exclusion rule.

\paragraph{Semantic neighbors and control concepts.}
The semantic neighborhood $\mathcal{N}(c)$ is determined by the procedure of Definition~\ref{def:semantic_neighborhood}: $\mathcal{N}(c)$ collects the candidates in $c$'s own pool whose centroid cosine distance to $c$ falls below the 25th percentile $\delta_c$ of pairwise distances in that pool. This produces concept-specific neighborhoods of five to seven discovered neighbors per target (\cref{tab:kappa_per_target}).
The control set $\mathcal{C}(c)$ of each target consists of three control concepts drawn from \{\texttt{cactus}, \texttt{jellyfish}, \texttt{melon}, \texttt{police car}\}, with \texttt{umbrella} in place of \texttt{police car} when the target is \texttt{car}; these are concepts whose centroids lie well above $\delta_c$ for every target, verified to contribute zero to $\hat\kappa$ across all five targets. For each neighbor $c' \in \mathcal{N}(c)$ and control $\tilde c \in \mathcal{C}(c)$, we generate prompts that explicitly mention the concept while excluding lexical forms of the target $c$. Per-target prompt files are released in the supplementary material.

\paragraph{Classifier-based scoring.}
All images are classified using a CLIP-based classifier over a fixed vocabulary of 65 concept labels (\cref{app:vocab}). For each image, the classifier returns the top-1 label among the vocabulary. All metrics defined below are computed from these top-1 assignments. In several random instances, we also perform a manual evaluation to ensure consistency of the results.


\subsection{Evaluation Metrics}
\label{app:metric_defs}
Each metric is defined with respect to a base model 
$\mathcal{D}_\theta$ and an unlearned model 
$\mathcal{D}_{\hat\theta}$. We write $\mathrm{Acc}(\mathcal{D}, c, \mathcal{P})$ for the fraction of images generated by model $\mathcal{D}$ from prompts in set $\mathcal{P}$ that are classified as concept $c$ by the CLIP evaluator.

\paragraph{Erasure metrics (measuring $\varepsilon$).}
These metrics measure how effectively the target concept $c$ is suppressed in the unlearned model.
\begin{itemize}
    \item \emph{Unlearning Accuracy (UA)}: The fraction of images generated from direct prompts $\mathcal{P}_{\mathrm{dir}}(c)$ by $\mathcal{D}_{\hat\theta}$ that 
    are \emph{not} classified as $c$. 
    \begin{equation}
    \label{eq:ua}
    \mathrm{UA}(c) = 1 - \mathrm{Acc}(\mathcal{D}_{\hat\theta}, c, \mathcal{P}_{\mathrm{dir}}(c)).
    \end{equation}
    Higher UA indicates more complete erasure under explicit lexical prompts.
    \item \emph{Indirect Recovery Rate (IRR)}: The fraction of images generated from indirect recovery prompts $\mathcal{P}_{\mathrm{ind}}(c)$ by 
    $\mathcal{D}_{\hat\theta}$ that \emph{are} classified as $c$: 
    \begin{equation}
    \label{eq:irr}
    \mathrm{IRR}(c) = \mathrm{Acc}(\mathcal{D}_{\hat\theta}, c, \mathcal{P}_{\mathrm{ind}}(c)).
    \end{equation}
    Higher IRR indicates greater vulnerability to concept recovery through correlated contextual cues. In contrast to UA, IRR measures the $\varepsilon$-robust erasure property formalized in~\cref{def:robust}.
\end{itemize}

\paragraph{Retention metrics (measuring $\gamma$).}
We measure retention at three levels of granularity:
\begin{itemize}
    \item \emph{In-domain Retain Accuracy (IRA)}: The fraction of images whose intended concept is not the target, i.e., all images from neighbor and control prompts, that are classified as the concept they were meant to show. Direct and indirect recovery prompts are excluded because their intended label is the target itself. IRA therefore measures how well the unlearned model still renders every non-target concept it is asked for, without normalizing by the base model.
    \item \emph{Neighbor Preservation (NP)}: This is our primary trade-off metric. The overlap with remaining concepts is often dominated by the overlap with a few concepts which are mostly semantically similar with the target concept. Therefore, as a proxy metric for utility preservation, we evaluate utility preservation in the semantic neighborhood $\mathcal{N}(c)$ of the target concept (\cref{def:semantic_neighborhood}) and define
    \begin{equation}
    \label{eq:neighbor_pres}
    \mathrm{NP}(c) = \frac{1}{|\mathcal{N}(c)|} \sum_{c' \in \mathcal{N}(c)} \frac{\mathrm{Acc}(\mathcal{D}_{\hat\theta}, c', \mathcal{P}_{\mathrm{dir}}(c'))}
    {\mathrm{Acc}(\mathcal{D}_\theta, c', \mathcal{P}_{\mathrm{dir}}(c'))}.
    \end{equation}
    An NP score of $1.0$ indicates perfect preservation of neighbors after unlearning; lower values indicate proportional collateral damage. The ratio form normalizes out the base-model recognition accuracy, which differs across concepts (and which is generally below 1.0 under a 65-class vocabulary).
    \item \emph{Control preservation (CP)}: To distinguish neighbor-specific collateral damage from broader degradation, we additionally evaluate performance on a non-neighbor control set $\mathcal{C}(c)$. We define
    \begin{equation}
    \label{eq:cp}
    \mathrm{CP}(c) = \frac{1}{|\mathcal{C}(c)|} \sum_{\tilde c \in \mathcal{C}(c)} 
    \frac{\mathrm{Acc}(\mathcal{D}_{\hat\theta}, \tilde c, \mathcal{P}_{\mathrm{dir}}(\tilde c))}
     {\mathrm{Acc}(\mathcal{D}_\theta, \tilde c, \mathcal{P}_{\mathrm{dir}}(\tilde c))}.
    \end{equation}
    Comparing NP and CP isolates neighbor-specific collateral damage from broader degradation: a method with $\mathrm{NP} \ll \mathrm{CP}$ is damaging concepts selectively in the target's neighborhood, consistent with the propagation mechanism in~\cref{thm:tradeoff}. A method with $\mathrm{NP} \approx \mathrm{CP} \ll 1$ is causing uniform degradation across the concept space.
\end{itemize}


\paragraph{DamageGap.} A summary statistic we introduce to quantify neighbor-selective damage: $\mathrm{DamageGap}(c) = \mathrm{CP}(c) - \mathrm{NP}(c)$. A positive DamageGap indicates that neighbors are damaged more than controls, consistent with the trade-off predicted by~\cref{thm:tradeoff}. A DamageGap near zero indicates uniform degradation (or uniform preservation) across both sets.

\subsection{Classification Vocabulary}
\label{app:vocab}
All CLIP-based classification uses the following 65-concept vocabulary, which is the union of the five targets, their candidate neighbors, and ten out-of-domain control labels, which include the controls used for $\hat\kappa$ estimation (\cref{app:kappa_estimation}). The same vocabulary is used for every metric and every one of these five targets, ensuring cross-target comparability. For the five additional targets of~\cref{app:additional_concepts}, the vocabulary is extended with each new target and its candidate neighbors (\cref{tab:new_concept_neighbors}).

\begin{table}[h]
\centering
\small
\begin{tabular}{p{3.2cm} p{10cm}}
\toprule
Category & Concepts \\
\midrule
Targets & \texttt{horse}, \texttt{cat}, \texttt{dog}, \texttt{bear}, \texttt{castle} \\
Equine  & \texttt{pony}, \texttt{donkey}, \texttt{zebra}, \texttt{mustang}, \texttt{mare}, \texttt{foal}, \texttt{stallion}, \texttt{mule}, \texttt{bison}, \texttt{camel} \\
Feline  & \texttt{kitten}, \texttt{tiger}, \texttt{lion}, \texttt{leopard}, \texttt{panther}, \texttt{cheetah}, \texttt{lynx}, \texttt{jaguar}, \texttt{cougar}, \texttt{bobcat} \\
Canid  & \texttt{puppy}, \texttt{fox}, \texttt{coyote}, \texttt{jackal}, \texttt{dingo}, \texttt{hyena}, \texttt{husky}, \texttt{labrador}, \texttt{beagle}, \texttt{retriever} \\
Ursine & \texttt{grizzly}, \texttt{polar bear}, \texttt{panda}, \texttt{cub}, \texttt{koala}, \texttt{black bear}, \texttt{brown bear}, \texttt{teddy bear}, \texttt{sloth bear}, \texttt{sun bear} \\
Architectural  & \texttt{fortress}, \texttt{tower}, \texttt{cathedral}, \texttt{palace}, \texttt{keep}, \texttt{citadel}, \texttt{watchtower}, \texttt{stronghold}, \texttt{fortified wall}, \texttt{manor} \\
Controls  & \texttt{airplane}, \texttt{bicycle}, \texttt{cactus}, \texttt{jellyfish}, \texttt{lighthouse}, \texttt{melon}, \texttt{police car}, \texttt{teapot}, \texttt{violin}, \texttt{volcano} \\
\bottomrule
\end{tabular}
\caption{The $65$-concept classification vocabulary used for all CLIP-based metric evaluations in~\cref{sec:experiment}. The vocabulary is the union of the five targets, their respective fine-grained subtype candidates spanning four animal category and one architectural category, and ten out-of-domain controls. See~\cref{tab:kappa_per_target} for the per-target subset $\mathcal{N}(c)$ that survives the $\delta_{c}$-filter and forms the averaging set for NP.}
\label{tab:vocab}
\end{table}

\subsection{Classifier reliability}\label{app:classifier}
We evaluate the fixed CLIP classifier on 100 labeled images per target (500 in
total), drawn from the base model's generations, and report per-concept true-positive and false-positive rates
(Table~\ref{tab:clf}). Scoring is one-vs-rest among the five targets: for each
target, its 100 images are the positives and the 400 images of the other four
targets are the negatives, so TPR is out of 100 and FPR is out of 400. The classifier attains a micro-averaged TPR of $97.0\%$
and FPR of $0.55\%$, ranging from $90\%$ TPR for \texttt{dog} to $1.25\%$ FPR
for \texttt{horse}. All methods share the same classifier, vocabulary, prompts,
and decision rule, so classifier error does not differ across methods.

\begin{table}[h]\centering\small
\caption{CLIP classifier reliability on labeled images.}\label{tab:clf}
\begin{tabular}{lcc}\toprule
Target & TPR (\%) $\uparrow$ & FPR (\%) $\downarrow$\\\midrule
\texttt{dog} & 90 & 0.25\\
\texttt{cat} & 100 & 0.25\\
\texttt{bear} & 96 & 1.00\\
\texttt{horse} & 99 & 1.25\\
\texttt{castle} & 100 & 0.00\\\midrule
Average & 97.0 & 0.55\\\bottomrule
\end{tabular}
\end{table}

%% file: sections/app/10.tex
\section{Per-Concept Results}
\label{app:per_concept_tables}

This appendix provides the per-concept evaluation tables underlying the averaged results in \Cref{tab:main_results_avg}. For each of the five target concepts, we report all metrics across the seven original unlearning methods.
The concepts span estimated entanglement coefficients $\hat\kappa \in [0.27, 0.89]$ (\Cref{tab:kappa}), enabling the $\kappa$-scaling analysis in \Cref{subsec:kappa_scaling}.

We order tables by $\hat\kappa$ in descending order to make the predicted scaling visually apparent: collateral damage (low IRA, low NP, high DamageGap) is most pronounced for high-$\hat\kappa$ concepts and attenuates as $\hat\kappa$ decreases.

\paragraph{SAeUron coverage.} 
SAeUron's released sparse autoencoder features cover the four animal concepts (\texttt{dog}, \texttt{bear}, \texttt{horse}, \texttt{cat}) but not \texttt{castle}. We mark the corresponding row with dashes in \Cref{tab:per_concept_castle} and exclude SAeUron from the averaged castle entries. Attempts to substitute SAeUron's \texttt{architectures} feature for \texttt{castle} did not yield meaningful suppression of castle generation (verified by visual inspection at the default ablation multiplier); we therefore report SAeUron only on concepts where its released features apply.

\paragraph{Reading the tables.}
All tables follow the same column structure as \Cref{tab:main_results_avg}: erasure metrics (UA, IRR), retention metrics (IRA, NP, CP), and the composite DamageGap. The \emph{Original} row reports baseline rates for the unedited $\mathcal{D}_\theta$ on each concept's evaluation set; note that baseline IRA varies substantially across concepts (e.g., {0.85} for \texttt{cat} vs. {0.54} for \texttt{tower/castle}) due to differences in CLIP evaluator accuracy on each concept's in-domain object set. This per-concept baseline variation motivates our use of normalized retention ratios in~\cref{subsec:kappa_scaling}.

\input{sections/app/pareto_per_concept}


\subsection{Dog}
\label{app:per_concept_dog}

\begin{table}[ht]
\centering
\small
\setlength{\tabcolsep}{5pt}
\begin{tabular}{llcccccc}
\toprule
& & \multicolumn{2}{c}{\textbf{Erasure}} & \multicolumn{3}{c}{\textbf{Retention}} & \\
\cmidrule(lr){3-4} \cmidrule(lr){5-7}
\textbf{Method} & \textbf{Type} & UA$\uparrow$ & IRR$\downarrow$ & IRA$\uparrow$ & NP$\uparrow$ & CP$\uparrow$ & DamageGap$\downarrow$ \\
\midrule
Original    & --          & 0     & 27.0  & 78.92 & 100   & 100   & --    \\
\midrule
ESD         & FT          & 75.0  & 17.5  & 74.15 & 93.51 & 96.0  & 2.48  \\
SalUn       & FT          & 90.0  & 8.0   & 50.92 & 58.01 & 76.0  & 17.98 \\
EDiff       & FT          & 62.0  & 23.5  & 66.61 & 80.29 & 92.99 & 12.70 \\
\midrule
AdvUnlearn  & Robust FT   & 96.0  & 20.0  & 74.30 & 63.05 & 99.0  & 35.94 \\
STEREO      & Robust FT   & 100.0 & 1.0   & 4.76  & 1.81  & 64.0  & 62.18 \\
\midrule
SEOT        & Inf.        & 68.0  & 30.0  & 79.07 & 84.20 & 99.0  & 14.80 \\
SAeUron     & Inf.        & 78.0  & 18.5  & 8.76  & 4.27  & 73.0  & 68.72 \\
\bottomrule
\end{tabular}
\caption{\textbf{Per-concept evaluation for \texttt{dog} ($\hat\kappa = 0.89$).} Discovered neighbors include 
\texttt{puppy}, \texttt{labrador}, \texttt{beagle}, \texttt{husky}, and \texttt{retriever} (full list in~\Cref{app:eval_pipeline}). As the highest-$\hat\kappa$ concept in our evaluation, \texttt{dog} exhibits the most severe robustness-retention trade-off: robust fine-tuning methods (STEREO, AdvUnlearn) achieve high UA but cause substantial IRA and NP degradation.}
\label{tab:per_concept_dog}
\end{table}

\paragraph{Observations.}
\texttt{Dog} exhibits the most severe trade-off in our evaluation, consistent with its position as the highest-$\hat\kappa$ concept. STEREO achieves complete erasure (UA = $100$) but collapses NP to $1.81$ and IRA to $4.76$, effectively destroying the model's ability to generate dog-related neighbors. SAeUron shows analogous catastrophic damage (NP = $4.27$, IRA = $8.76$) despite operating at inference time, suggesting that even modest displacement can wreck retention when neighbor activation regions densely overlap the target. ESD stands out as the only method achieving meaningful erasure (UA = $75$) while preserving neighbors well (NP = $93.51$, DamageGap = $2.48$), at the cost of higher residual leakage (IRR = $17.5$). The Pareto separation between ESD ($75, 93.51$) and STEREO ($100, 1.81$) is a $\approx 92$-point NP gap for a $25$-point UA gain, the steepest trade-off observed across our concepts and a direct empirical instantiation of the $\kappa$-scaled lower bound predicted by~\Cref{thm:tradeoff}.


\subsection{Bear}
\label{app:per_concept_bear}

\begin{table}[ht]
\centering
\small
\setlength{\tabcolsep}{5pt}
\begin{tabular}{llcccccc}
\toprule
& & \multicolumn{2}{c}{\textbf{Erasure}} & \multicolumn{3}{c}{\textbf{Retention}} & \\
\cmidrule(lr){3-4} \cmidrule(lr){5-7}
\textbf{Method} & \textbf{Type} & UA$\uparrow$ & IRR$\downarrow$ & IRA$\uparrow$ & NP$\uparrow$ & CP$\uparrow$ & DamageGap$\downarrow$ \\
\midrule
Original    & --          & 0     & 85.5  & 62.75 & 100   & 100   & --    \\
\midrule
ESD         & FT          & 94.5  & 9.0   & 44.00 & 58.02 & 95.0  & 36.97 \\
SalUn       & FT          & 99.5  & 5.5   & 51.51 & 53.89 & 89.0  & 35.10 \\
EDiff       & FT          & 96.0  & 6.0   & 38.61 & 56.75 & 94.0  & 37.24 \\
\midrule
AdvUnlearn  & Robust FT   & 99.0  & 1.0   & 42.46 & 67.37 & 100   & 32.62 \\
STEREO      & Robust FT   & 99.5  & 2.0   & 15.88 & 13.75 & 77.0  & 63.25 \\
\midrule
SEOT        & Inf.        & 94.0  & 13.5  & 50.84 & 64.14 & 96.48 & 32.34 \\
SAeUron     & Inf.        & 85.5  & 15.0  & 22.37 & 63.50 & 93.8  & 30.30 \\
\bottomrule
\end{tabular}
\caption{\textbf{Per-concept evaluation for \texttt{bear} ($\hat\kappa = 0.82$).} Discovered neighbors include \texttt{grizzly}, \texttt{panda}, \texttt{polar bear}, \texttt{koala}, and \texttt{cub} (full list in~\Cref{app:eval_pipeline}). Image counts differ for two runs: SalUn uses 272 neighbor images and SEOT uses 462 neighbor images.}
\label{tab:per_concept_bear}
\end{table}

\paragraph{Observations.}
\texttt{Bear} sits in the high-$\hat\kappa$ regime alongside horse and dog, and exhibits the predicted Pareto structure. STEREO again occupies the high-erasure / low-retention extreme (UA = $99.5$, NP = $13.75$, DamageGap = $63.25$), while AdvUnlearn achieves comparable erasure (UA = $99$) with substantially better neighbor preservation (NP = $67.37$). The clustering pattern is unusually compressed on bear: standard fine-tuning methods (ESD, SalUn, EDiff) and the inference-time methods cluster tightly between NP = $53.89$ to $67.37$, with SEOT and AdvUnlearn essentially tied at NP $\approx 64$ to $67$. Notably, no method on bear matches the (high UA, high NP) operating point achieved on lower-$\hat\kappa$ concepts: even ESD, the best performer on \texttt{dog} and \texttt{castle}, reaches only NP = $58.02$ here. This compression of the achievable retention region is a direct manifestation of the $\hat\kappa$-scaled Pareto frontier predicted by~\Cref{thm:tradeoff}.
\subsection{Horse}
\label{app:per_concept_horse}

\begin{table}[ht]
\centering
\small
\setlength{\tabcolsep}{5pt}
\begin{tabular}{llcccccc}
\toprule
& & \multicolumn{2}{c}{\textbf{Erasure}} & \multicolumn{3}{c}{\textbf{Retention}} & \\
\cmidrule(lr){3-4} \cmidrule(lr){5-7}
\textbf{Method} & \textbf{Type} & UA$\uparrow$ & IRR$\downarrow$ & IRA$\uparrow$ & NP$\uparrow$ & CP$\uparrow$ & DamageGap$\downarrow$ \\
\midrule
Original    & --          & 0     & 75.5  & 71.23 & 100   & 100   & --    \\
\midrule
ESD         & FT          & 80.0  & 33.5  & 62.92 & 81.12 & 97.0  & 15.87 \\
SalUn       & FT          & 95.0  & 28.5  & 40.92 & 39.42 & 91.99 & 52.57 \\
EDiff       & FT          & 81.0  & 31.0  & 63.84 & 76.48 & 94.0  & 17.51 \\
\midrule
AdvUnlearn  & Robust FT   & 97.5  & 23.0  & 65.38 & 71.42 & 100   & 28.57 \\
STEREO      & Robust FT   & 100.0 & 7.5   & 20.74 & 14.71 & 27.0  & 12.28 \\
\midrule
SEOT        & Inf.        & 81.5  & 38.5  & 69.69 & 97.06 & 99.0  & 1.93  \\
SAeUron     & Inf.        & 78.0  & 39.5  & 76.92 & 82.48 & 94.0  & 11.51 \\
\bottomrule
\end{tabular}
\caption{\textbf{Per-concept evaluation for \texttt{horse} ($\hat\kappa = 0.81$).} Discovered neighbors include 
\texttt{pony}, \texttt{mare}, \texttt{donkey}, \texttt{zebra}, \texttt{stallion}, \texttt{foal} and \texttt{mule} (full list in~\Cref{app:eval_pipeline}). Image counts differ for two runs: \emph{Original} uses approximately 480 neighbor and control images, and SEOT uses 58 direct images.}
\label{tab:per_concept_horse}
\end{table}

\paragraph{Observations.}
\texttt{Horse} is the running example used throughout \Cref{sec:experiment} and exhibits the canonical Pareto structure across method families. STEREO occupies the high-erasure / catastrophic-retention extreme (UA = $100$, NP = $14.71$, IRA = $20.74$); SEOT and SAeUron occupy the high-retention / weaker-erasure region (NP $\geq 82$ but IRR $\geq 38.5$); standard fine-tuning methods sit in between, with ESD reaching a clean $(80, 81.12)$ operating point at DamageGap = $15.87$. SalUn is an outlier among standard FT methods, achieving high UA = $95$ with NP = $39.42$, far below ESD and EDiff at comparable erasure, consistent with SalUn's gradient-saliency masking displacing more activation than necessary for erasure. A second notable observation is STEREO's collapsed CP = $27$, far below STEREO's CP on the other four concepts (range $64$ to $96$): in this high-$\hat\kappa$ regime, STEREO's adversarial subspace appears to spill beyond the intended object subspace into stylistically-correlated features, manifesting the broad subspace suppression discussed in~\Cref{rem:stereo}.
\subsection{Cat}
\label{app:per_concept_cat}

\begin{table}[ht]
\centering
\small
\setlength{\tabcolsep}{5pt}
\begin{tabular}{llcccccc}
\toprule
& & \multicolumn{2}{c}{\textbf{Erasure}} & \multicolumn{3}{c}{\textbf{Retention}} & \\
\cmidrule(lr){3-4} \cmidrule(lr){5-7}
\textbf{Method} & \textbf{Type} & UA$\uparrow$ & IRR$\downarrow$ & IRA$\uparrow$ & NP$\uparrow$ & CP$\uparrow$ & DamageGap$\downarrow$ \\
\midrule
Original    & --          & 0     & 96.0  & 85.50 & 100   & 100   & --    \\
\midrule
ESD         & FT          & 66.0  & 38.0  & 66.92 & 87.35 & 98.0  & 10.65  \\
SalUn       & FT          & 87.0  & 21.0  & 34.53 & 73.04 & 100   & 26.95 \\
EDiff       & FT          & 76.0  & 24.0  & 62.30 & 65.99 & 100   & 34.00 \\
\midrule
AdvUnlearn  & Robust FT   & 100.0 & 0.0   & 50.76 & 72.40 & 98.0  & 25.60 \\
STEREO      & Robust FT   & 93.0  & 7.0   & 27.68 & 27.60 & 90.0  & 62.40 \\
\midrule
SEOT        & Inf.        & 71.66 & 34.0  & 42.15 & 81.99 & 100   & 18.01 \\
SAeUron     & Inf.        & 72.0  & 37.5  & 60.25 & 85.63 & 96.0  & 10.37 \\
\bottomrule
\end{tabular}
\caption{\textbf{Per-concept evaluation for \texttt{cat} ($\hat\kappa = 0.77$).} Discovered neighbors include \texttt{kitten}, \texttt{lynx}, \texttt{leopard}, \texttt{tiger}, and \texttt{cheetah} (full list in~\Cref{app:eval_pipeline}). The original baseline IRA = $85.50$ reflects high CLIP-evaluator accuracy on cat's in-domain object set after correcting for kitten/cat label confusion (see~\Cref{app:evaluation_setup}). Image counts differ for two runs: \emph{Original} uses approximately 480 neighbor and control images, and SEOT uses 72 direct images.}
\label{tab:per_concept_cat}
\end{table}

\paragraph{Observations.}
\texttt{Cat} exhibits a trade-off structure consistent with the high-$\hat\kappa$ regime, though its base IRA = $85.50$ is the highest of the five concepts, providing more retention "headroom" than other targets. STEREO again occupies the high-erasure / catastrophic-retention extreme (UA = $93$, NP = $27.60$, IRA = $27.68$), with DamageGap = $62.40$ matching its damage on horse and bear. AdvUnlearn achieves complete erasure (UA = $100$, IRR = $0$) at NP = $72.40$, demonstrating that adversarial training can preserve neighbors more effectively than STEREO at comparable erasure strength. ESD preserves neighbors well (NP $=87.35$, DamageGap $=10.65$) but achieves the weakest erasure on this concept (UA $=66$, IRR $=38$). The STEREO IRA-retention ratio for cat ($27.68 / 85.50 = 0.32$) is consistent with the corresponding ratio on horse ($0.29$) and bear ($0.25$), reinforcing the $\hat\kappa$-scaling result in~\Cref{subsec:kappa_scaling}.
\subsection{Castle}
\label{app:per_concept_castle}

\begin{table}[ht]
\centering
\small
\setlength{\tabcolsep}{5pt}
\begin{tabular}{llcccccc}
\toprule
& & \multicolumn{2}{c}{\textbf{Erasure}} & \multicolumn{3}{c}{\textbf{Retention}} & \\
\cmidrule(lr){3-4} \cmidrule(lr){5-7}
\textbf{Method} & \textbf{Type} & UA$\uparrow$ & IRR$\downarrow$ & IRA$\uparrow$ & NP$\uparrow$ & CP$\uparrow$ & DamageGap$\downarrow$ \\
\midrule
Original    & --          & 0     & 9.5   & 53.84 & 100   & 100   & --    \\
\midrule
ESD         & FT          & 88.0  & 6.5   & 48.15 & 83.03 & 99.0  & 15.98 \\
SalUn       & FT          & 97.0  & 1.0   & 54.76 & 67.58 & 94.0  & 26.41 \\
EDiff       & FT          & 82.0  & 2.5   & 46.46 & 87.35 & 99.0  & 11.65 \\
\midrule
AdvUnlearn  & Robust FT   & 100.0 & 1.0   & 44.00 & 83.10 & 100   & 16.89 \\
STEREO      & Robust FT   & 100.0 & 0.0   & 29.38 & 31.96 & 96.0  & 64.03 \\
\midrule
SEOT        & Inf.        & 98.0  & 6.5   & 53.84 & 98.21 & 99.0  & 0.79  \\
SAeUron     & Inf.        & --    & --    & --    & --    & --    & --    \\
\bottomrule
\end{tabular}
\caption{\textbf{Per-concept evaluation for \texttt{castle} ($\hat\kappa = 0.27$).} Discovered neighbors include 
\texttt{fortress}, \texttt{palace}, \texttt{citadel}, \texttt{keep}, \texttt{cathedral} and \texttt{manor} (full list in~\Cref{app:eval_pipeline}). SAeUron's released sparse autoencoder features do not cover \texttt{castle}; we attempted to substitute its \texttt{architectures} feature, but the resulting ablation did not yield meaningful suppression of castle generation. Image counts differ for two runs: \emph{Original} uses approximately 480 neighbor and control images, and SEOT uses 60 direct and 100 indirect images.}
\label{tab:per_concept_castle}
\end{table}

\paragraph{Observations.}
\texttt{Castle} is the lowest-$\hat\kappa$ concept in our evaluation and demonstrates qualitatively different trade-off behavior from the high-$\hat\kappa$ animal concepts. Several methods achieve strong erasure with minimal retention damage: ESD reaches UA = $88$ at NP = $83.03$ (DamageGap = $15.98$), AdvUnlearn achieves complete erasure (UA = $100$) at NP = $83.10$ (DamageGap = $16.89$), and SEOT reaches UA = $98$ at NP = $98.21$ (DamageGap = $0.79$, the lowest of any (method, concept) pair in our evaluation). This is qualitatively distinct from the high-$\hat\kappa$ concepts, where no method achieves both UA $\geq 90$ and NP $\geq 80$. STEREO is the notable exception: despite the low predicted $\kappa$-floor, STEREO incurs NP = $31.96$ and DamageGap = $64.03$, comparable to its damage on the high-$\hat\kappa$ animal concepts. This is the empirical illustration of \Cref{thm:tradeoff} as a \emph{lower bound} rather than a ceiling: STEREO's REO objective broadly suppresses the spanned adversarial subspace regardless of whether neighbors actually occupy the displaced region (\Cref{rem:stereo}), producing damage well above the $\kappa$-predicted floor that other methods successfully hug. The separation between STEREO and the alternatives on \texttt{castle} suggests room for $\kappa$-aware unlearning methods that achieve robust erasure without STEREO's $\kappa$-blind over-suppression.

\subsection{Additional Target Concepts}
\label{app:additional_concepts}
To extend the five-target analysis in the main text, we evaluate \texttt{snake}, \texttt{fish}, \texttt{tree}, \texttt{butterfly}, and \texttt{car}. Their estimated entanglement coefficients span $\hat\kappa \in [0.37, 0.83]$, adding concepts between the low-overlap \texttt{castle} and the densely overlapping animal subtypes. \Cref{tab:new_concept_results} reports the unedited reference and results for EDiff, ESD, SalUn, and STEREO using the metrics defined in~\cref{sec:experiment}.
Averaged over the four methods, NP is $65.0$, $62.0$, $61.7$, $61.3$, and $58.8$ for \texttt{snake}, \texttt{fish}, \texttt{tree}, \texttt{butterfly}, and \texttt{car}, decreasing monotonically with $\hat\kappa$ ($\rs=-1.0$). Within-method correlations are $-0.4$ (EDiff), $-0.7$ (ESD), $-0.2$ (SalUn), and $-0.7$ (STEREO). STEREO again sits at the aggressive end, reaching UA $\geq 99.1$ on every new target with NP between $16.0$ and $26.6$.
\Cref{tab:new_concept_neighbors} lists the neighbors discovered for each new target with the procedure of~\cref{app:kappa_estimation}. The classification vocabulary for these targets is extended accordingly (\cref{app:vocab}).

\begin{table}[htbp]
\centering
\small
\setlength{\tabcolsep}{6pt}
\begin{tabular}{lccp{0.58\linewidth}}
\toprule
\textbf{Target} & $\hat\kappa$ & $|\mathcal{N}(c_u)|$ & \textbf{Discovered neighbors} \\
\midrule
\texttt{snake}     & 0.37 & 7 & anaconda, boa, cobra, garter snake, king snake, python, rattlesnake \\
\texttt{fish}      & 0.52 & 8 & bass, cod, eel, herring, mackerel, salmon, trout, tuna \\
\texttt{tree}      & 0.63 & 6 & cedar, cypress, maple, oak, pine, spruce \\
\texttt{butterfly} & 0.66 & 7 & birdwing, dragonfly, fritillary, monarch butterfly, moth, red admiral, swallowtail \\
\texttt{car}       & 0.83 & 7 & convertible, coupe, hatchback, minivan, sedan, suv, van \\
\bottomrule
\end{tabular}
\caption{\textbf{Discovered neighbors for the five additional targets.} As for the original targets (\cref{tab:kappa_per_target}), neighborhood size does not determine $\hat\kappa$: \texttt{fish} has the most neighbors but an intermediate $\hat\kappa$.}
\label{tab:new_concept_neighbors}
\end{table}
\begin{table*}[htbp]
\centering
\caption{Results for the five additional target concepts. Original denotes the pretrained model before unlearning. All evaluation metrics are percentages. Higher UA, IRA, NP, and CP are better; lower IRR and DamageGap are better. DamageGap is the difference between control preservation (CP) and neighbor preservation (NP).}
\label{tab:new_concept_results}
\resizebox{\textwidth}{!}{
\begin{tabular}{llccccccc}
\toprule
Concept & Method & $\hat{\kappa}$ & UA (\%) $\uparrow$ & IRR (\%) $\downarrow$ & IRA (\%) $\uparrow$ & NP (\%) $\uparrow$ & CP (\%) $\uparrow$ & DamageGap (\%) $\downarrow$ \\
\midrule
Snake
& Original  & 0.37 & 9.2 & 6.0 & 31.1 & 100.0 & 100.0 & 0.0 \\
& EraseDiff & 0.37 & 88.1 & 8.6 & 19.1 & 81.2 & 93.6 & 12.4 \\
& ESD       & 0.37 & 97.9 & 6.8 & 20.6 & 81.3 & 98.1 & 16.8 \\
& SalUn     & 0.37 & 100.0 & 0.0 & 5.5 & 72.4 & 86.2 & 13.8 \\
& STEREO    & 0.37 & 100.0 & 0.0 & 14.8 & 25.1 & 96.0 & 71.1 \\
\midrule
Fish
& Original  & 0.52 & 6.2 & 11.0 & 40.5 & 100.0 & 100.0 & 0.0 \\
& EraseDiff & 0.52 & 92.3 & 5.3 & 23.2 & 78.0 & 92.1 & 14.1 \\
& ESD       & 0.52 & 94.5 & 2.7 & 34.5 & 79.2 & 98.0 & 18.8 \\
& SalUn     & 0.52 & 100.0 & 0.0 & 14.9 & 67.5 & 97.8 & 30.3 \\
& STEREO    & 0.52 & 100.0 & 0.0 & 5.8 & 23.2 & 38.4 & 15.2 \\
\midrule
Tree
& Original  & 0.63 & 4.4 & 61.8 & 51.8 & 100.0 & 100.0 & 0.0 \\
& EraseDiff & 0.63 & 32.0 & 55.2 & 53.8 & 75.1 & 98.0 & 22.9 \\
& ESD       & 0.63 & 39.7 & 48.0 & 50.3 & 82.4 & 98.9 & 16.5 \\
& SalUn     & 0.63 & 100.0 & 0.0 & 5.8 & 62.7 & 88.0 & 25.3 \\
& STEREO    & 0.63 & 100.0 & 0.0 & 4.9 & 26.6 & 32.0 & 5.4 \\
\midrule
Butterfly
& Original  & 0.66 & 7.1 & 12.2 & 36.0 & 100.0 & 100.0 & 0.0 \\
& EraseDiff & 0.66 & 46.2 & 37.3 & 34.6 & 74.7 & 92.5 & 17.8 \\
& ESD       & 0.66 & 72.6 & 28.9 & 31.2 & 78.1 & 83.3 & 5.2 \\
& SalUn     & 0.66 & 100.0 & 0.0 & 0.9 & 73.1 & 93.9 & 20.8 \\
& STEREO    & 0.66 & 100.0 & 0.0 & 10.5 & 19.2 & 68.0 & 48.8 \\
\midrule
Car
& Original  & 0.83 & 2.6 & 60.0 & 87.5 & 100.0 & 100.0 & 0.0 \\
& EraseDiff & 0.83 & 47.4 & 49.3 & 69.2 & 78.7 & 91.2 & 12.4 \\
& ESD       & 0.83 & 60.5 & 28.6 & 75.8 & 73.8 & 98.6 & 24.8 \\
& SalUn     & 0.83 & 79.3 & 16.1 & 17.8 & 66.8 & 85.6 & 18.8 \\
& STEREO    & 0.83 & 99.1 & 4.3 & 21.1 & 16.0 & 75.2 & 59.1 \\
\bottomrule
\end{tabular}
}
\end{table*}

\subsection{Which targets admit clean erasure}\label{app:clean-erasure}
Across the per-concept results we report (the seven original methods on the
five main targets, \cref{tab:per_concept_dog,tab:per_concept_bear,tab:per_concept_horse,tab:per_concept_cat,tab:per_concept_castle}, and four methods on the five additional
targets, \cref{tab:new_concept_results}), the only (method, concept) pairs that reach both
UA $>90$ and NP $>80$ are on the two lowest-$\hat\kappa$ targets:
\texttt{castle} ($\hat\kappa=0.27$; SEOT at $98.0/98.21$, AdvUnlearn at
$100.0/83.10$) and \texttt{snake} ($\hat\kappa=0.37$; ESD at $97.9/81.3$).
No method reaches this region on any target with $\hat\kappa\geq0.52$
(\cref{fig:pareto_per_concept,tab:new_concept_results}). Because $\hat\kappa$ is computed from base-model activations alone,
it identifies these targets before any unlearning is run, which is the sense in
which we propose it as a pre-deployment diagnostic (\cref{sec:conclusion}). The damage is
also neighbor-selective rather than global: control preservation stays above
$85$ for twelve of the thirteen methods in \cref{tab:main_results_avg} (all except
STEREO), while NP drops.

%% file: sections/app/pareto_per_concept.tex
\begin{figure}[t]
\centering
\begin{tikzpicture}
\begin{axis}[
  width=0.78\linewidth, height=0.52\linewidth,
  xmin=58, xmax=102, ymin=-3, ymax=103,
  xlabel={UA (\%) $\uparrow$}, ylabel={NP (\%) $\uparrow$},
  grid=major, grid style={line width=0.2pt, draw=black!10},
  axis line style={draw=black!40}, tick style={draw=black!40},
  legend style={at={(1.03,0.5)}, anchor=west, font=\footnotesize, draw=none, fill=none, row sep=1pt},
  legend cell align=left,
  mark size=2.6pt,
]
\fill[black!6] (axis cs:90,80) rectangle (axis cs:102,103);
\draw[black!45, dashed, line width=0.6pt] (axis cs:90,-3) -- (axis cs:90,103);
\draw[black!45, dashed, line width=0.6pt] (axis cs:58,80) -- (axis cs:102,80);
\addlegendimage{only marks, mark=*, famFT}\addlegendentry{Fine-tuning}
\addlegendimage{only marks, mark=*, famRob}\addlegendentry{Robust FT}
\addlegendimage{only marks, mark=*, famInf}\addlegendentry{Inference-time}
\addlegendimage{empty legend}\addlegendentry{}
\addlegendimage{only marks, mark=*, black!55}\addlegendentry{\texttt{dog}}
\addlegendimage{only marks, mark=triangle*, mark size=3.2pt, black!55}\addlegendentry{\texttt{bear}}
\addlegendimage{only marks, mark=diamond*, mark size=3.2pt, black!55}\addlegendentry{\texttt{horse}}
\addlegendimage{only marks, mark=square*, black!55}\addlegendentry{\texttt{cat}}
\addlegendimage{only marks, mark=star, mark size=3.4pt, line width=1pt, black!55}\addlegendentry{\texttt{castle}}
\addplot[only marks, mark=*, mark size=2.8pt, draw=white, line width=0.4pt, fill=famFT, forget plot] coordinates {(75,93.51) (90,58.01) (62,80.29)};
\addplot[only marks, mark=triangle*, mark size=3.4pt, draw=white, line width=0.4pt, fill=famFT, forget plot] coordinates {(94.5,58.02) (99.5,53.89) (96,56.75)};
\addplot[only marks, mark=diamond*, mark size=3.4pt, draw=white, line width=0.4pt, fill=famFT, forget plot] coordinates {(80,81.12) (95,39.42) (81,76.48)};
\addplot[only marks, mark=square*, mark size=2.6pt, draw=white, line width=0.4pt, fill=famFT, forget plot] coordinates {(66,87.35) (87,73.04) (76,65.99)};
\addplot[only marks, mark=star, mark size=4.2pt, draw=famFT, line width=1.1pt, forget plot] coordinates {(88,83.03) (97,67.58) (82,87.35)};
\addplot[only marks, mark=*, mark size=2.8pt, draw=white, line width=0.4pt, fill=famRob, forget plot] coordinates {(96,63.05) (100,1.81)};
\addplot[only marks, mark=triangle*, mark size=3.4pt, draw=white, line width=0.4pt, fill=famRob, forget plot] coordinates {(99,67.37) (99.5,13.75)};
\addplot[only marks, mark=diamond*, mark size=3.4pt, draw=white, line width=0.4pt, fill=famRob, forget plot] coordinates {(97.5,71.42) (100,14.71)};
\addplot[only marks, mark=square*, mark size=2.6pt, draw=white, line width=0.4pt, fill=famRob, forget plot] coordinates {(100,72.40) (93,27.60)};
\addplot[only marks, mark=star, mark size=4.2pt, draw=famRob, line width=1.1pt, forget plot] coordinates {(100,83.10) (100,31.96)};
\addplot[only marks, mark=*, mark size=2.8pt, draw=white, line width=0.4pt, fill=famInf, forget plot] coordinates {(68,84.20) (78,4.27)};
\addplot[only marks, mark=triangle*, mark size=3.4pt, draw=white, line width=0.4pt, fill=famInf, forget plot] coordinates {(94,64.14) (85.5,63.50)};
\addplot[only marks, mark=diamond*, mark size=3.4pt, draw=white, line width=0.4pt, fill=famInf, forget plot] coordinates {(81.5,97.06) (78,82.48)};
\addplot[only marks, mark=square*, mark size=2.6pt, draw=white, line width=0.4pt, fill=famInf, forget plot] coordinates {(71.66,81.99) (72,85.63)};
\addplot[only marks, mark=star, mark size=4.2pt, draw=famInf, line width=1.1pt, forget plot] coordinates {(98,98.21)};
\end{axis}
\end{tikzpicture}
\caption{\textbf{Per-concept operating points of the seven original methods.} Each point is one (method, concept) pair from~\cref{app:per_concept_tables}; color encodes method family and marker encodes the target concept. The shaded region marks strong erasure together with strong neighbor preservation (UA $>90$, NP $>80$). No point for the four high-$\hat\kappa$ animal targets enters it; the only points that do are \texttt{castle} ($\hat\kappa=0.27$) under SEOT and AdvUnlearn, consistent with the $\kappa$-scaled bound of~\cref{thm:tradeoff}. SAeUron has no \texttt{castle} point (\cref{app:per_concept_castle}).}
\label{fig:pareto_per_concept}
\end{figure}
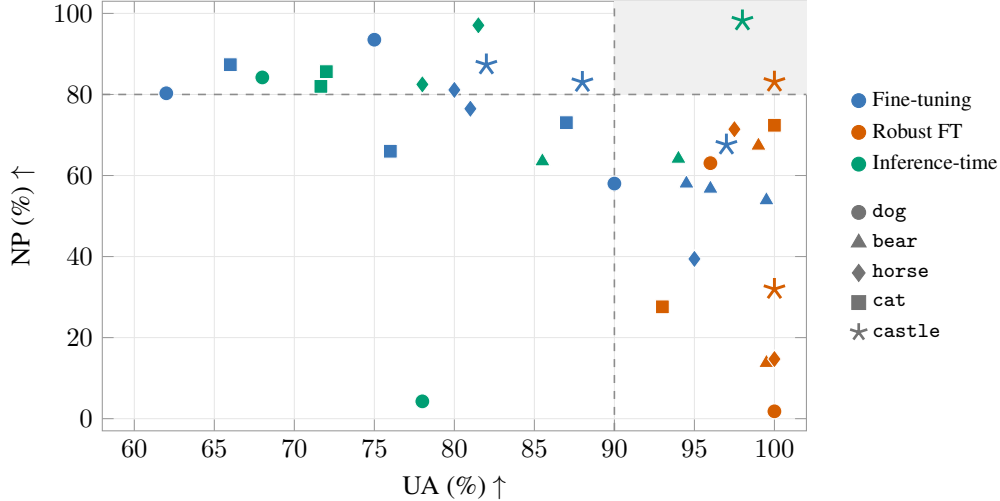

%% file: sections/app/11.tex
\section{Direct Validation: Activation Displacement Scales with $\hat\kappa$}
\label{app:displacement_kappa}
\Cref{thm:tradeoff} predicts that activation displacement at neighbor prompts scales linearly with $\kappa$ for a method achieving $\varepsilon$-robust erasure. The Neighbor Preservation (NP) measure used in~\Cref{sec:experiment} is a probability-units proxy for this displacement, related through~\cref{asm:lip}. To test~\Cref{thm:tradeoff}'s prediction in the units the theorem is stated in, we measure activation displacement directly:
\[
\Delta_{\text{neighbor}}(c_u) \;=\; \mathbb{E}_{p_{c'} \sim \mathcal{N}(c_u)}\!\left[\, \big\| \Phi_{\hat\theta}(p_{c'}) - \Phi_\theta(p_{c'}) \big\|_2 \,\right]
\]
We measure $\Delta_{\text{neighbor}}$ for each fine-tuning method on each of the five target concepts using $200$ neighbor prompts per concept, with activations captured from \texttt{up\_blocks.1.attentions.1} at DDIM step $25$ of $50$. Activations are flattened across spatial and token dimensions before computing $L_2$ distances, so absolute magnitudes here are larger than the mean-pooled token-level magnitudes reported in~\cref{app:displacement} (\cref{fig:displacement}); both views are consistent in ordering and in the prediction they probe (within-target ordering in~\cref{app:displacement}, cross-target $\hat\kappa$-scaling here). Inference-time methods are omitted: SEOT's text-embedding intervention does not propagate measurably to this layer in $L_2$ distance, and SAeUron's sparse-feature ablation produces displacement values an order of magnitude larger than fine-tuning methods, reflecting a real mechanism difference rather than direct comparability with weight-update methods.

\paragraph{Displacement scales monotonically with $\hat\kappa$.}
\Cref{tab:displacement} reports mean neighbor displacement across the five concepts and five fine-tuning methods. All five methods exhibit a positive Spearman rank correlation between $\hat\kappa$ and displacement, four of them at $\rs\geq0.70$, with ESD showing strict monotonicity ($\rs = +1.00$) and STEREO closely tracking ($\rs = +0.90$ with a single inversion between \texttt{cat} and \texttt{horse}, whose $\hat\kappa$ values differ by only $0.04$). Across all five methods, displacement on \texttt{castle} ($\hat\kappa = 0.27$) is substantially lower than the average displacement across the four high-$\hat\kappa$ animal concepts: ratios range from $1.68\times$ (ESD) to $3.18\times$ (STEREO). The qualitative form predicted by \Cref{thm:tradeoff}, displacement increasing with $\hat\kappa$, is empirically supported.

\begin{table}[h]
\centering
\small
\setlength{\tabcolsep}{6pt}
\begin{tabular}{lccccc|c}
\toprule
\textbf{Method} & \texttt{castle} & \texttt{cat} & \texttt{horse} & \texttt{bear} & \texttt{dog} & $\rs$ \\
& ($\hat\kappa{=}0.27$) & ($0.77$) & ($0.81$) & ($0.82$) & ($0.89$) & vs.\ $\hat\kappa$ \\
\midrule
ESD        & 128.4 & 194.8 & 213.6 & 221.0 & 235.4 & $+1.00$ \\
STEREO     & \phantom{0}75.9 & 234.2 & 233.4 & 240.9 & 257.0 & $+0.90$ \\
AdvUnlearn & \phantom{0}71.6 & 110.5 & 150.9 & 123.2 & 123.7 & $+0.70$ \\
EDiff      & \phantom{0}67.8 & 163.0 & 178.1 & 172.1 & 175.6 & $+0.70$ \\
SalUn      & \phantom{0}88.1 & 220.7 & 215.6 & 220.4 & 215.3 & $+0.10$ \\
\bottomrule
\end{tabular}
\caption{\textbf{Mean activation displacement at neighbor prompts.} Each cell reports $\Delta_{\text{neighbor}}$ in $L_2$ activation distance, averaged over $200$ neighbor prompts. $\rs$ is computed across the five concepts. Note that $n = 5$ concepts limits the statistical power of significance testing on individual methods; we report $\rs$ as a descriptive monotonicity measure. The consistency of positive correlations across all five methods (range $+0.10$ to $+1.00$; four of five at $\geq 0.70$) provides aggregate evidence for the $\kappa$-scaling predicted by~\Cref{thm:tradeoff}.}
\label{tab:displacement}
\end{table}

\paragraph{STEREO's displacement is consistent with a single local Lipschitz constant.}
\Cref{thm:tradeoff} predicts that for a method operating exactly at the bound, the ratio $\kappa(1-\varepsilon)/(\Delta_{\text{neighbor}} + c_0)$ should equal a single Lipschitz constant $L$ across all concepts. We compute this ratio for each method using STEREO's saturated $\varepsilon \approx 0.035$ (mean IRR across the five concepts; see~\cref{tab:main_results_avg})  and an additive constant $c_0 = 0.05$.
 The resulting $L_{\text{tight}}$ values are reported in~\Cref{tab:l_tight}. STEREO's $L_{\text{tight}}$ varies by only $7.8\%$ across the five concepts ($L \in [0.0032, 0.0035]$); other methods exhibit spreads of $24\%$ to $57\%$. STEREO's single-$L$ consistency is precisely the quantitative prediction of \Cref{thm:tradeoff}: a method operating at the bound should exhibit a $\kappa$-and-concept-independent local Lipschitz constant. Among the five methods evaluated, only STEREO satisfies this prediction, consistent with STEREO occupying the high-$\Delta$, low-$\varepsilon$ regime where the bound is tightest.

\begin{table}[ht]
\centering
\small
\setlength{\tabcolsep}{6pt}
\begin{tabular}{lcccccc}
\toprule
\textbf{Method} & \texttt{castle} & \texttt{cat} & \texttt{horse} & \texttt{bear} & \texttt{dog} & Spread \\
\midrule
\textbf{STEREO} & 0.0035 & 0.0032 & 0.0034 & 0.0033 & 0.0034 & \textbf{7.8\%} \\
EDiff           & 0.0039 & 0.0046 & 0.0045 & 0.0047 & 0.0050 & 23.6\% \\
SalUn           & 0.0030 & 0.0034 & 0.0037 & 0.0036 & 0.0040 & 29.5\% \\
ESD             & 0.0021 & 0.0039 & 0.0037 & 0.0036 & 0.0037 & 53.4\% \\
AdvUnlearn      & 0.0037 & 0.0068 & 0.0053 & 0.0065 & 0.0070 & 57.1\% \\
\bottomrule
\end{tabular}
\caption{\textbf{Per-concept $L_{\text{tight}}$ values.} $L_{\text{tight}} = \kappa(1-\varepsilon)/(\Delta_{\text{neighbor}} + c_0)$ is the Lipschitz constant value that places \Cref{thm:tradeoff}'s bound exactly at the observed displacement. If a method operates at the bound, $L_{\text{tight}}$ should be approximately constant across concepts. STEREO is the only method exhibiting this consistency ($7.8\%$ spread), supporting its operation at the bound's saturated-erasure regime.}
\label{tab:l_tight}
\end{table}

\paragraph{Local versus global Lipschitz.}
The local $L \approx 0.0034$ inferred from STEREO's displacement is much smaller than the global $\hat L = 0.174$ reported in~\Cref{fig:lipschitz}. This reflects the well-known phenomenon that probability changes saturate near classification decision boundaries: the global Lipschitz constant measured across diverse prompt-pair perturbations averages over both saturated and unsaturated regimes, while the local constant near $\mathcal{R}_{c_u}$ during erasure reflects unsaturated behavior specifically. The two values are compatible measurements of different but related quantities; principled estimation of $L$ that distinguishes these regimes is a direction for future work.

%% file: sections/app/13.tex
\section{Text-Space versus Denoiser Entanglement}\label{app:textspace}
Our analysis measures entanglement inside the denoising network. To examine
whether overlap is already present in the text encoder, we apply the same
estimator, candidate pools, and prompts to text-encoder representations.

\paragraph{T5 (FLUX.1-schnell).} We extract padding-masked, mean-pooled token
representations from T5 blocks 0, 11, and 23 and from the final encoder output
for the five targets. At all four layers and for all five targets,
$\hat\kappa=0$. This is not caused by unrelated concepts entering the
neighborhoods: no control is selected at any layer, 19 of 20 target-layer
conditions have positive silhouette scores, and the nearest final-layer
centroids are semantically appropriate (\texttt{dog}/\texttt{puppy},
\texttt{cat}/\texttt{kitten}, \texttt{horse}/\texttt{pony},
\texttt{castle}/\texttt{palace}), yet lie $1.61$ to $2.97\times$ beyond the
discovery threshold. The same pools and estimator give
$\hat\kappa=0.27$, $0.77$, $0.81$, $0.82$, $0.89$ in the SD v1.4 denoiser.

\paragraph{CLIP text space versus denoiser within SDXL.} To compare the two
spaces within one pipeline, we estimate $\hat\kappa$ for a ten-concept set
spanning several semantic groups in SDXL's CLIP text-embedding space and in its
denoiser cross-attention space (\cref{tab:sdxl-text-vs-denoiser}). In CLIP space, nine of ten concepts obtain
$\hat\kappa=1.0$ with semantically related neighborhoods; only \texttt{castle}
obtains $0$. The denoiser produces heterogeneous values, including $0.72$ for
\texttt{horse}, $0.68$ for \texttt{car}, and $0.60$ for \texttt{dog}, with the
remaining concepts at $0$. Neighbors here are discovered only within this
ten-concept set, the visualization pool of~\cref{app:activation_regions},
rather than from each target's full candidate pool, so these values are not
comparable to $\hat\kappa_{\text{SDXL}}$ in~\cref{tab:sdxl}; they serve only to
contrast the two representation spaces under identical conditions.

\begin{table}[h]\centering\small
\caption{$\hat\kappa$ for the ten-concept set in SDXL's CLIP text-embedding
space and in its denoiser cross-attention space.}\label{tab:sdxl-text-vs-denoiser}
\begin{tabular}{lcc}\toprule
Concept & $\hat\kappa$ (CLIP text) & $\hat\kappa$ (denoiser)\\\midrule
\texttt{horse}  & 1.00 & 0.72\\
\texttt{pony}   & 1.00 & 0.00\\
\texttt{donkey} & 1.00 & 0.00\\
\texttt{deer}   & 1.00 & 0.00\\
\texttt{dog}    & 1.00 & 0.60\\
\texttt{cat}    & 1.00 & 0.00\\
\texttt{bear}   & 1.00 & 0.00\\
\texttt{car}    & 1.00 & 0.68\\
\texttt{truck}  & 1.00 & 0.00\\
\texttt{castle} & 0.00 & 0.00\\\bottomrule
\end{tabular}
\end{table}

\paragraph{Interpretation.} Text-space geometry varies across encoders: CLIP
shows substantial overlap under our estimator, whereas T5 shows none. Neither
comparison is a controlled variance decomposition (pooling differs between text
and denoiser features, and the T5 comparison spans model families), so we read
them as evidence that the source and severity of entanglement depend on both
the text representation and the denoiser computation, not as proof that one
encoder is less entangled in general.

%% file: sections/app/14.tex
\section{Cross-Architecture Experiments}\label{app:arch}

\paragraph{SDXL.} SDXL shares the cross-attention structure of SD v1.4, so
$\hat\kappa$ is directly computable. Table~\ref{tab:sdxl} reports ESD under the
same erasure and evaluation protocol for three concepts. \Cref{fig:umap_sdxl} shows that SDXL's cross-attention activations have qualitatively similar geometry to SD v1.4 (\cref{fig:umap_kappa}): related animal concepts occupy overlapping regions, and \texttt{castle} lies apart from them. Castle's own architectural neighbors are not in this visualization, and its estimated entanglement with them is higher on SDXL ($\hat\kappa_{\text{SDXL}}=0.72$) than on SD v1.4 ($0.27$). The absolute level of entanglement therefore shifts across architectures, but $\hat\kappa$ and observed collateral damage agree within each: on SDXL, NP orders \texttt{cat} $>$ \texttt{castle} $>$ \texttt{dog} exactly as $\hat\kappa_{\text{SDXL}}$ predicts (\cref{tab:sdxl}).

\begin{figure}[h]
\centering
\includegraphics[width=0.7\linewidth]{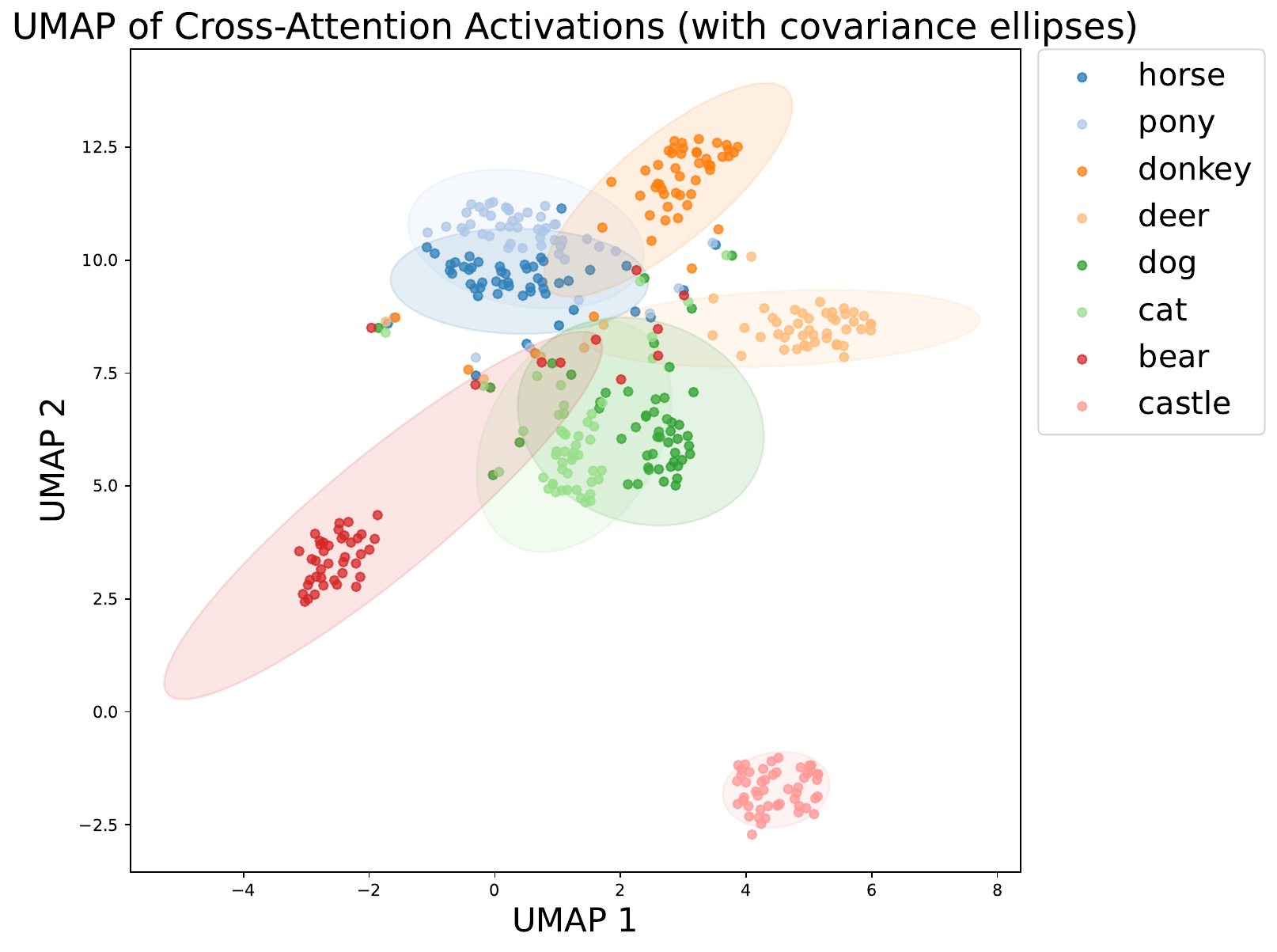}
\caption{UMAP projection of SDXL cross-attention activations for eight
concepts, with covariance ellipses. As in SD v1.4 (\cref{fig:umap_kappa}),
related animal concepts (\texttt{horse}/\texttt{pony}/\texttt{donkey},
\texttt{dog}/\texttt{cat}) overlap, while \texttt{castle} lies apart from the animal clusters. Castle's architectural neighbors are not shown; its entanglement with them is $\hat\kappa_{\text{SDXL}}=0.72$ (\cref{tab:sdxl}).}
\label{fig:umap_sdxl}
\end{figure}

\begin{table}[h]\centering\small
\caption{ESD on SDXL. NP decreases monotonically with $\hat\kappa_{\text{SDXL}}$.}
\label{tab:sdxl}
\begin{tabular}{lcccc}\toprule
Concept & $\hat\kappa_{\text{SDXL}}$ & UA $\uparrow$ & NP $\uparrow$ & DamageGap $\downarrow$\\\midrule
\texttt{cat} & 0.65 & 99 & 90.9 & 9.1\\
\texttt{castle} & 0.72 & 91 & 85.0 & 15.0\\
\texttt{dog} & 0.92 & 77 & 77.1 & 22.9\\\bottomrule
\end{tabular}
\end{table}

\paragraph{FLUX (MMDiT with T5).} FLUX uses joint text-image attention and has
no cross-attention block analogous to \texttt{mid\_block.attentions.0}. A proxy
$\kappa$ estimated from the joint-attention denoiser saturates near $1$ for all
targets (\texttt{dog} $1.00$, \texttt{cat} $0.97$, \texttt{castle} $0.97$) and
also for two unrelated controls (\texttt{airplane}$\to$\texttt{teapot} $0.99$,
bootstrap 95\% CI $[0.94,1.00]$; \texttt{teapot}$\to$\texttt{airplane} $1.00$,
CI $[0.95,1.00]$; centroid distance $0.011$). The proxy therefore does not
resolve concept structure at the layers we can access, and we treat
$\kappa$-scaling on MMDiT as untested. Table~\ref{tab:flux} nonetheless shows
that erasure produces neighbor-selective damage: NP collapses while unrelated
controls are preserved.

\begin{table}[h]\centering\small
\caption{ESD on FLUX (percent).}\label{tab:flux}
\begin{tabular}{lccccc}\toprule
Concept & UA $\uparrow$ & IRR $\downarrow$ & NP $\uparrow$ & CP $\uparrow$ & DamageGap $\downarrow$\\\midrule
Base (\texttt{dog}) & 8.2 & 26.5 & 100.0 & 100.0 & 0.0\\
\texttt{dog} & 78.0 & 8.5 & 35.6 & 96.9 & 61.3\\
\texttt{cat} & 83.0 & 9.5& 37.8 & 95.3 & 57.6\\
\texttt{castle} & 85.0 &20.0 & 40.5 & 98.9 & 58.3\\\bottomrule
\end{tabular}
\end{table}

\paragraph{Comparison across backbones.} For ESD on \texttt{dog}, erasure is
nearly matched across backbones (UA $75$, $77$, and $78$ on SD v1.4, SDXL, and
FLUX), while NP is $93.5$, $77.1$, and $35.6$. The robustness-retention
trade-off therefore persists on newer backbones rather than being specific to
SD v1.4. Absolute NP is not strictly comparable across architectures because
their image distributions differ.

%% file: sections/app/15.tex
\section{Broader Impacts}\label{app:impacts}
Our analysis is instantiated in diffusion models, but the mechanism only requires activation displacement, local smoothness of the activation map, and target-neighbor overlap in representation space. Similar $\kappa$-scaled trade-offs may arise in autoregressive language models, classifiers, and multimodal models, but verifying this requires task-specific definitions of concepts, activations, and preservation.